\documentclass{article}

\PassOptionsToPackage{numbers, compress}{natbib}

\usepackage[preprint]{neurips_2026}

\usepackage[utf8]{inputenc} 
\usepackage[T1]{fontenc}    
\usepackage[hidelinks]{hyperref}       
\usepackage[table,usenames,dvipsnames]{xcolor}
\hypersetup{
    colorlinks=true,
    linkcolor=black,
    citecolor=black,
    urlcolor=MidnightBlue
}
\usepackage{url}            
\usepackage{booktabs}       
\usepackage{amsfonts}       
\usepackage{nicefrac}       
\usepackage{microtype}      
\usepackage{xcolor}         
\usepackage{enumitem}
\usepackage{makecell}

\usepackage{amsmath, amssymb, amsthm}
\usepackage{mathtools}

\newtheorem{theorem}{Theorem}
\newtheorem{definition}{Definition}
\newtheorem{lemma}[theorem]{Lemma}
\newtheorem{proposition}[theorem]{Proposition}

\usepackage{mdframed}

\usepackage[toc,page,header]{appendix}
\usepackage{minitoc}

\definecolor{calloutbg}{HTML}{F3F9FF}
\definecolor{calloutborder}{HTML}{FFFFFF}
\definecolor{titlecolor}{HTML}{254D8C}  

\newcommand{\eqtitle}[1]{\textbf{\textcolor{titlecolor}{#1}}}

\newenvironment{calloutbox}{%
  \begin{mdframed}[
    backgroundcolor=calloutbg, linecolor=calloutborder,
    linewidth=0pt, leftline=true, rightline=false,
    topline=false, bottomline=false,
    innerleftmargin=12pt, innerrightmargin=8pt,
    innertopmargin=4pt, innerbottommargin=4pt,
    skipabove=4pt, skipbelow=0pt,
    nobreak=true,
  ]%
}{\end{mdframed}}

\newcommand{\titledeq}[3]{%
  \begin{calloutbox}
  \refstepcounter{equation}\label{#3}%
  \noindent
  \makebox[0pt][l]{\eqtitle{#1}}%
  \makebox[\linewidth][c]{$\displaystyle #2$}%
  \makebox[0pt][r]{(\theequation)}%
  \end{calloutbox}
}

\title{Language models suffer from a curse of ambiguity}

\author{%
  Nicolas Zucchet \quad Hyun Dong Lee \quad Scott Linderman \\
  Stanford University \\
  \texttt{\{nzucchet, hdlee, scott.linderman\}@stanford.edu} \\
}

\begin{document}

\doparttoc 
\faketableofcontents 

\maketitle

\begin{abstract}
    Large language models increasingly rely on sampling as a driver of their own improvement, making the fidelity of their learned distributions more critical than ever. 
    Yet, not all distributions are equally easy to learn.
    In this work, we identify a curse of ambiguity: in large language models, and more broadly in all neural networks that produce discrete probability distributions, the more ambiguous a next-token distribution is, the harder it is to learn accurately.
    Through an extensive theoretical analysis, we trace this curse to architectural and learning roots. More ambiguous distributions require more capacity to be stored, larger embeddings to be represented, more steps to be fitted, and amplify token-sampling noise.
    We validate these findings on synthetic tasks with controlled ground truth and observe the same signatures in language models trained on real data.
    Our results provide a new perspective on the statistical capabilities of large language models and a practical framework for when to trust their output distribution.
\end{abstract}

\vspace{-.5em}
\section{Introduction} 
\vspace{-.5em}
Autoregressive language models generate text one token at a time, sampling a conditional distribution given the tokens that came before~\citep{shannon_prediction_1951, bengio_neural_2003, mikolov_recurrent_2010, sutskever_generating_2011, vaswani_attention_2017, radford_improving_2018, radford_language_2019, brown_language_2020}.
We often use greedy decoding strategies to bias samples toward the most likely tokens~\citep{holtzman_curious_2020, hendrycks_measuring_2021, wei_chain_2022, chowdhery_palm_2023}, but such strategies have deleterious effects.
Inference-time scaling methods~\citep{silver_mastering_2017, jones_scaling_2021, wang_self-consistency_2023, openai_learning_2024} only succeed when the model spreads probability over diverse continuations, not just the top ones~\citep{brown_monkeys_2024, snell_scaling_2024}.
Similarly, on-policy reinforcement learning post-training~\citep{ouyang_training_2022,guo_deepseek_2025} can only explore and reinforce trajectories the model actually samples~\citep{yue_does_2025}.
As model-generated synthetic data~\citep{gunasekar_textbooks_2023, dubey_llama_2024} and AI-generated content online~\citep{thompson_shocking_2024, liang_mapping_2024} become an increasingly large share of training corpora, a lack of diversity in the sampled distribution risks compounding across generations until the model collapses~\citep{shumailov_ai_2024}.
For all these reasons, the learned distribution needs to be well calibrated, not only capturing the most likely token but the full next-token distribution.
This necessity is especially acute in ambiguous settings where several continuations are plausible. 

In this work, we take a step toward understanding how well neural networks, and thus large language models, can model such uncertainty.
We uncover a new phenomenon, the \emph{curse of ambiguity}:
\vspace{-0.35em}
\begin{center}
    \textit{The more ambiguous the next-token distribution, the harder it is to learn.}
\end{center}
\vspace{-0.15em}
The rest of the paper is dedicated to making this statement precise, to uncovering its roots, and to discussing its consequences for large language models.
Our contributions are the following:
\begin{itemize}[leftmargin=20pt]
    \item[--] We identify and formalize the curse of ambiguity in large language models.
    \item[--] We trace the curse to architectural roots (Section~\ref{sec:architecture})---more ambiguous distributions require more capacity to store and larger embeddings to represent---and to learning roots (Section~\ref{sec:learning})---they take more optimization steps to fit and amplify the noise from token sampling.
    \item[--] We validate the curse empirically (Section~\ref{sec:experiments}), confirming our theoretical predictions in controlled synthetic settings, showing that the same signatures appear in language models trained on real data, and arguing that these results extend to larger models.
\end{itemize}

\section{Autoregressive language modeling as context classification}
\label{sec:setup}

Most large language models are decoder-only Transformers \cite{vaswani_attention_2017, radford_improving_2018, radford_language_2019, brown_language_2020} that learn to autoregressively predict the next token \cite{shannon_prediction_1951, bengio_neural_2003, mikolov_recurrent_2010, sutskever_generating_2011}. In our analysis, we view this process as a context-to-token classification problem, and we distinguish the many contexts that language models must classify by the ambiguity of the next token under the training distribution.
This section introduces this formalism and motivates the simplified toy model we analyze theoretically in Sections~\ref{sec:architecture} and \ref{sec:learning}.

\paragraph{From sequences to classification.}
An autoregressive language model factorizes the probability of a sequence $x_{1:t}$ as a product of conditional distributions, each predicting $x_t$ given its preceding context $x_{1:t-1}$, that is
\begin{equation}
    p_\theta(x_{1:t}) = \prod_{t'=1}^{t} p_\theta(x_{t'}|x_{1:t'-1})
\end{equation}
with $\theta$ the parameters of the model and $p_\theta$ the distribution it defines. Mechanistically, $p_\theta(x_t | x_{1:t-1})$ is computed by first letting a Transformer-based backbone map the context $x_{1:t-1}$ to an embedding $e_{t} = f_\theta(x_{1:t-1})$ of dimension $h$, and then applying a linear head $W$ followed by a softmax to produce a distribution over the size-$d$ vocabulary:
\begin{equation}
    \label{eqn:lm-head}
    p_\theta(x_{t} = j \, | \, x_{1:t-1}) = \mathrm{softmax}(W e_{t})_{j}, \, \text{where } e_{t} = f_\theta(x_{1:t-1}).
\end{equation}
From this perspective, next token prediction is a classification problem: the backbone produces a representation of the context, and then the linear head predicts the next token with a multi-class logistic regression. Training jointly fits the backbone and the head by minimizing, on average over training contexts, the cross-entropy between the ground-truth and predicted distributions.

\paragraph{Focusing on the linear head for analytical tractability.}
Sections~\ref{sec:architecture} and~\ref{sec:learning} study this last linear head in isolation, treating the context representations produced by the backbone as given. The linear head plays a central role in the network: it is the only component that maps embeddings to the output vocabulary, and it modulates all gradients flowing back into the backbone~\citep{godey_lost_2026}. Studying it in isolation therefore probes the part of the model most directly responsible for fitting the next-token distribution. Representation learning is, in this respect, an important but secondary concern. Section~\ref{sec:experiments} will then empirically show that similar results hold for the language model as a whole.

Since we view the last layer as a logistic regression, we can simplify our notation: the head sees $n$ context embeddings $e_i$ paired with target distributions $p_i$, and must learn to map one to the other. We study how it copes when these targets vary in ambiguity.
By default, we treat the embeddings as sampled i.i.d.~and uniformly from the unit sphere. In Section~\ref{subsec:embedding}, we will instead ask what optimal embeddings would look like and what they can achieve at a fixed dimension $h$; in that regime, we jointly optimize the weights $W$ and the embeddings $e_i$.

\paragraph{Modeling ambiguity through support size.}
In natural language, contexts often admit multiple plausible continuations: after ``the cat sat on the'', reasonable next tokens include ``mat'', ``floor'', ``couch'', etc. We model this ambiguity by assigning each context $i$ a ground truth distribution $p_i$ that is uniform on a random support $S_i$ of size $k$. The probability of context $i$ being followed by token $j$ is
\begin{equation}
    \label{eqn:target-dist}
    p_{i,j} = \begin{cases} \frac{1}{k} & \text{if } j \in S_i \\ 0 & \text{otherwise.} \end{cases}
\end{equation}
The hyperparameter $k$ controls the degree of ambiguity:\footnote{We use ``ambiguity'' beyond its linguistic sense~\citep{piantadosi_communicative_2012}, as a term for distributions spreading mass over many tokens.} $k = 1$ corresponds to a context with a single possible continuation according to the data, while $k > 1$ means several tokens are equally likely.
Although next-token distributions are obviously not uniform in practice, this model captures the key intuition: spreading probability mass over more tokens makes the prediction problem intrinsically harder.
On the non-uniform distributions we encounter in our experiments in Section~\ref{sec:experiments}, we measure ambiguity through the effective support size defined as the exponential of the entropy of $p_i$.
It is more robust than the raw support size in the sense that vanishing probabilities contribute vanishingly little, and it reduces to $k$ on the uniform-on-support distributions we study in theory.
The same quantity appears under different names in other fields, including perplexity in information theory \citep{jelinek_perplexitymeasure_1977} and multiplicity in statistical mechanics \citep{schroeder_introduction_2000}.

\section{Architectural roots of the curse of ambiguity}
\label{sec:architecture}

The curse of ambiguity is partially rooted in the architectural properties of the model. More precisely, we show that, as they become more ambiguous, next token distributions require more capacity to be stored at a desired level of accuracy and larger embedding dimensions to be properly represented. We focus on gaining theoretical intuition for why this holds true using linear models.

\begin{figure}
    \centering
    \vspace{-0.5cm}
    \includegraphics{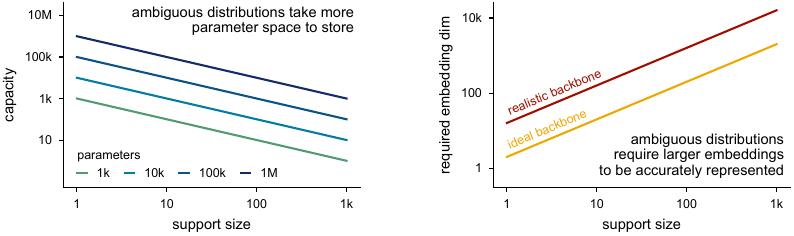}
    \vspace{-0.15cm}
    \caption{\textbf{Architectural roots of the curse of ambiguity: ambiguous distributions require more parameters to be stored and larger embeddings to be represented.} Schematic description of the theoretical analysis of Section~\ref{sec:architecture}, depicting how the network capacity (left) and the required embedding dimension (right) depend on our measure of ambiguity, which is the number of possible next tokens~$k$ (the support size) under the ground-truth distribution. The capacity is proportional to the number of parameters and inversely proportional to~$k$, while the required embedding dimension grows proportionally with~$k$, with a larger constant when the backbone cannot produce arbitrary embeddings (realistic vs.\ ideal backbone). Capacity predictions are empirically validated in Appendix~\ref{app:capacity-validation}.}
    \label{fig:architecture-theory}
\end{figure}

\vspace{-0.6em}
\subsection{Model capacity: ambiguous distributions require more parameters to be stored}
\vspace{-0.35em}
\label{subsec:capacity}

Intuitively, modeling $k$ possible continuations is like solving $k$ different context-to-token associations at once, which should take $k$ times more capacity. We now show that this is indeed the case.
We consider the linear head model of Equation~\ref{eqn:lm-head} with random embeddings $e_i$ and ask how the number of context-dependent distributions $p_i$ (as defined in Equation~\ref{eqn:target-dist}) it can store depends on~$k$. By ``store'' we informally mean the maximum number of contexts that the model can accurately recall; Appendix~\ref{subsec:capacity-def} gives a precise definition. This can be viewed as an associative memory problem, connecting to a long line of work on neural associative memories, historically on Hopfield networks~\citep{hopfield_neural_1982, krotov_dense_2016, ramsauer_hopfield_2021} and more recently on language models~\citep{katharopoulos_transformers_2020, geva_transformer_2021, schlag_linear_2021, bietti_birth_2023, allen-zhu_physics_2025, zucchet_how_2025, morris_how_2025}.

\vspace{-0.6em}
\paragraph{A Hebbian model.} Following previous work \cite{cabannes_scaling_2024, nichani_understanding_2025}, we study a Hebbian-like construction where the learned weight matrix is approximated by,
\begin{equation}
    W = \beta \sum_{i=1}^n z_i \, e_i^\top ~~~~ \mathrm{with} ~z_{i,j} = \begin{cases} 1-k/d & \text{if } j \in S_i \\ -k/d & \text{otherwise}\end{cases}.
\end{equation}
Here, $z_i$ are the target logits, and they are designed such that the difference between on- and off-support tokens is one and their corresponding probability distributions are close to $p_i$. 
We can think of $W$ as approximating the result of online gradient descent to minimize $\|z_i - W e_i\|^2$ with learning rate $\beta$, assuming the embeddings are approximately orthogonal.
Minimizing the actual cross-entropy loss by gradient descent would yield different weights with greater capacity, so this analysis only provides a lower bound on capacity.
However, the bound is close to tight: the Hebbian model captures the same scaling with $k$ as a linear head trained with Adam on random embeddings (Appendix~\ref{app:capacity-validation}).

\vspace{-0.6em}
\paragraph{Ambiguity requires more capacity.} To understand how well the Hebbian model predicts $p_i$ given embedding $e_i$ of size $h$, we decompose its output logits into a signal and a noise term:
\begin{equation}
    \ell_i := W e_i = \beta \big( \underbrace{z_i e_i^\top e_i}_{\text{signal}} + \underbrace{\textstyle\sum_{j \neq i} z_j \, (e_j^\top e_i)}_{\text{noise}} \big).
\end{equation}
Since the embeddings are unit vectors, the signal term has fixed magnitude and recovers $z_i$ exactly. The noise term comes from interference between embeddings. It has zero mean and per-component variance $\sigma^2 \approx \frac{nk}{hd}$ in the regime $k \ll d$ we focus on. The variance grows with the number of patterns~$n$ and the support size~$k$: Each target $z_i$ affects $k$ output coordinates, so any given prediction direction is interfered with by more patterns as ambiguity increases. The Hebbian model should be able to store all $n$ patterns as long as this noise stays smaller than the signal. Solving for $n$ yields the capacity scaling limit, our first root cause of the curse of ambiguity:
\titledeq{Root 1}{\mathrm{capacity} \propto \frac{hd}{k}.}{eqn:capacity}
The left panel of Figure~\ref{fig:architecture-theory} illustrates this scaling, and Appendix~\ref{app:capacity} formally proves it up to logarithmic factors under two different notions of capacity. A size-$k$ distribution requires $k$ times more capacity than a deterministic one, so more ambiguity directly requires more network capacity.
Appendix~\ref{app:capacity-validation} verifies the $hd$ and $1/k$ scalings empirically.
The $k$ scaling is tight up to logarithmic factors: Appendix~\ref{subsec:capacity-upper-bound} establishes a matching upper bound on capacity that applies to any linear-softmax model, via a sign-pattern counting argument on the model's degrees of freedom.
The linear model also recovers the proportionality to the total number of parameters $hd$, a now well-established universal measure of language model capacity, both
empirically~\citep{allen-zhu_physics_2025, morris_how_2025} and theoretically across different architectures~\citep{cabannes_scaling_2024, nichani_understanding_2025}.
The fact that it recovers this universal scaling suggests that its dependence on ambiguity extends to more general models as well.

\vspace{-0.6em}
\subsection{Embedding dimension: ambiguous distributions require larger embeddings}
\vspace{-0.3em}
\label{subsec:embedding}

We now move to another property of the architecture that matters just as much: the embedding dimension~$h$. We show that, to represent all possible target distributions, $h$ must grow at least linearly with the ambiguity level $k$. The intuition is that an $h$-dimensional embedding can only point in a few vocabulary directions at once, so once $h$ falls below the ambiguity level there is no room left to single out the top-$k$ tokens of the target distribution.

Compared to the total parameter count, the role of the embedding dimension in language model performance is less understood~\citep{yin_dimensionality_2018, gao_representation_2019, tao_scaling_2024, petty_impact_2024}. Yet, we are not the first to point out the softmax bottleneck~\citep{yang_breaking_2018, godey_small_2024, godey_lost_2026}: the embedding dimension constrains the family of distributions a language model can represent and learn. \citet{chang_softmax_2022} observe that ambiguous distributions are a particularly hard case. We sharpen these observations into a precise linear scaling with the ambiguity level.

\vspace{-0.55em}
\paragraph{The ideal backbone assumption.} As highlighted in Section~\ref{sec:setup}, a large language model can be split into a backbone that produces context-dependent embeddings and a linear head that must reconstruct the output distribution from these embeddings. For the embedding dimension analysis, we assume the backbone can produce close to optimal embeddings for the task. We ask how large the embedding dimension $h$ needs to be for the linear head to exactly recover the target distribution. This idealization says nothing about how the backbone achieves such embeddings; it isolates a purely representational constraint that no learning can circumvent, and therefore provides a lower bound on the embedding dimension that a successful model must exceed.

\vspace{-0.55em}
\paragraph{A compressed sensing analysis.} This representation question falls under the umbrella of compressed sensing, which broadly studies when sparse signals in high-dimensional spaces can be exactly or approximately recovered from few linear measurements~\citep{candes_robust_2006,donoho_compressed_2006,foucart_mathematical_2013}. The main takeaway of our analysis is that the embedding dimension is a representational root of the curse of ambiguity as ambiguous distributions require larger embeddings to be properly represented. We prove this through two theorems, formally stated and proven in Appendix~\ref{app:embedding}. The first shows that, in the idealized setup above, the linear-softmax model can exactly represent every family of size-$k$ distributions as long as the embedding dimension satisfies
\titledeq{Root 2}{h \geq 2k.}{eqn:embedding_root}
Its proof relates the representation problem to a classical question on convex polytopes, namely $k$-neighborliness, and invokes classical results from that literature (namely the Radon-type upper bound on neighborliness~\citep{grunbaum_convex_2003} and the cyclic-polytope construction~\citep{gale_neighborly_1963}) to precisely quantify the threshold.
The second theorem relaxes this idealization by allowing context-dependent perturbations of the embedding, reflecting the fact that real backbones can only produce embeddings that are close to, rather than exactly, optimal. In this robust regime, random-weight constructions from the compressed-sensing literature still yield an accurate representation up to a small increase of the KL-divergence, at the cost of an extra $\log(d/k)$ factor on~$h$. The embedding dimension thus remains proportional to~$k$, only with a larger overall constant, as illustrated in the right panel of Figure~\ref{fig:architecture-theory}.
\section{Learning-based roots of the curse of ambiguity}
\label{sec:learning}

Section~\ref{sec:architecture} characterized which distributions models efficiently represent and how many they can store, but said nothing about how fast they learn them. We show that there is an additional learning-based curse of ambiguity: high ambiguity distributions take longer to optimize and are inherently noisier. Appendix~\ref{app:classical} complements our learning-centric perspective with classical sampling results. 

\vspace{-0.2em}
\subsection{Optimization: ambiguous outputs take longer to learn}
\vspace{-0.15em}
\label{subsec:optimization}

We study gradient flow on the expected cross-entropy loss of the linear model of Equation~\ref{eqn:lm-head}, that is, on the average KL divergence between the ground-truth conditional distributions and the model's outputs, up to a constant entropy term. Gradient flow has a long tradition of being used as a proxy for actual training dynamics~\citep{jacot_neural_2018, soudry_implicit_2017}, outside of a few documented edge cases~\citep{cohen_gradient_2021, cohen_adaptive_2024}.
Our analysis shows that ambiguity slows down initial learning, since each output token is seen less frequently and thus provides a weaker signal, and yields higher losses late in training.

\textbf{Reduction to a scalar dynamics.}
The linear head must solve a logistic regression problem to map contexts to corresponding output probabilities. Far below the capacity threshold of Section~\ref{subsec:capacity}, these contexts decouple and we can study each in isolation as gradient flow on its logits. For a single context with target distribution~$p$ uniform on a support of size $k$, all on-support logits remain equal at all times (and likewise for off-support logits), so the dynamics collapse onto a single scalar, $\Delta \ell$, the gap between any on-support and any off-support logit. As shown in Appendix~\ref{app:optimization}, $\Delta \ell$ obeys the implicit equation
\begin{equation}
    \label{eqn:implicit-main}
    \frac{k^2}{d}\left(\exp(\Delta \ell) - 1\right) + \frac{k(d-k)}{d}\Delta \ell = t,
\end{equation}
with initial condition $\Delta \ell(0) = 0$, which relates the logit gap, training time $t$ and ambiguity $k$.

\begin{figure}
    \centering
    \vspace{-0.25cm}
    \includegraphics{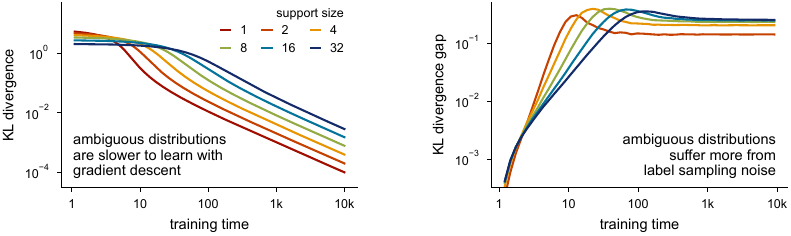}
    \vspace{-0.2cm}
    \caption{\textbf{Learning roots of the curse of ambiguity: ambiguous distributions require more optimization steps to be learned and amplify the noise coming from next-token sampling.} Simulation of the dynamics theoretically analyzed in Section~\ref{sec:learning}. Increasing the support size $k$ slows down initial gradient descent on the expected loss and yields larger losses in the long run (left), and widens the gap between sampled-batch trajectories and the expected-loss dynamics (right). The insights from this simplified one-context toy model extend to a linear head trained on many contexts of mixed ambiguity, as shown in Appendices~\ref{app:optimization-validation} and~\ref{app:sampling-validation}.}
    \label{fig:learning}
\end{figure}

\vspace{-0.25em}
\paragraph{Two phases of learning.} Equation~\ref{eqn:implicit-main} reveals two distinct phases of learning. Early on, $\Delta \ell$ is close to $0$ and grows linearly at rate $1/k$, so the KL divergence decreases as
\titledeq{Root 3.a}{
    D_\mathrm{KL}(t) \approx \log\!\left(\frac{d}{k}\right) - \frac{t}{k}  \quad \text{for small }t,
  }{eqn:early_training}
\vspace{-0.65em}
that is, ambiguity slows down initial learning. Late in training, the exponential term in Equation~\ref{eqn:implicit-main} dominates as $\Delta \ell$ grows to infinity and
\titledeq{Root 3.b}{
    D_\mathrm{KL}(t) \approx \frac{k(d-k)}{td} \quad \text{for large $t$}.
  }{eqn:asymptotics_late_training}
\vspace{-0.65em}
The convergence rate $1/t$ no longer depends on $k$, but the prefactor scales (approximately linearly in $k$ for $k \ll d$) with the support size: at any fixed training time, more ambiguous distributions sit at a higher loss.
Figure~\ref{fig:learning} confirms that these two phases together provide a comprehensive description of the true gradient flow trajectory.
Despite its simplicity, this single-context toy model accurately captures the learning dynamics of a linear head trained on many contexts of mixed ambiguity below the capacity threshold, with only a slow-down from cross-context competition as the number of contexts grows. Above the threshold, the capacity perspective becomes the relevant one: contexts that require less capacity are better modeled (Appendix~\ref{app:optimization-validation}).

\vspace{-0.2em}
\subsection{Stochastic optimization: ambiguous outputs have noisier gradients}
\vspace{-0.15em}
\label{subsec:sampling}

Our analysis so far has assumed access to the ground-truth conditional distributions $p_i$. In practice, large language models only observe a single next-token sample $y_i \sim p_i$ for each context $i$, and the cross-entropy loss is computed against this one-hot target rather than against $p_i$ itself. We focus on this output sampling noise, conditional on the context. This noise incurs an additional curse of ambiguity: a deterministic target produces no noise at all, whereas increasing the support size $k$ increases output sampling noise and thus impairs learning. The analysis below quantifies this effect.

\vspace{-0.25em}
\paragraph{Integrating output sampling noise in the gradient flow dynamics.} In our analysis in Section~\ref{subsec:optimization}, we optimized the expected cross-entropy loss, averaged over all possible outputs from the ground-truth distributions, which is unobserved in practice. In this section, we lift this limitation by introducing label sampling noise into the dynamics. To this end, we leverage a well-established line of work~\citep{mandt_stochastic_2017, li_stochastic_2017} which shows that, for a small enough learning rate, the stochastic gradient descent iterates are closely tracked by a stochastic differential equation that characterizes how the trajectory deviates from gradient flow. This framework lets us cleanly isolate how label sampling noise perturbs the deterministic dynamics of Section~\ref{subsec:optimization}. Writing the iterates as $\ell_t = \ell_t^\star + \delta_t$ where $\ell_t^\star$ is the gradient-flow trajectory, the deviation $\delta_t$ obeys an Ornstein--Uhlenbeck process driven by the per-step gradient noise covariance and damped by the local Hessian of the loss.

\vspace{-0.25em}
\paragraph{Noise introduces within-support variance.} Applied to the dynamics of Section~\ref{subsec:optimization}, the stochastic differential equation collapses onto a single scalar degree of freedom: the within-on-support variance of the logits,
\begin{equation}
    V := \frac{1}{k}\sum_{j \in S} (\ell_j - \bar \ell_S)^2, \, \text{with }\bar \ell_S := \frac{1}{k}\sum_{j \in S} \ell_j.
\end{equation}
This variance is zero under gradient flow and is the new degree of freedom unlocked by sampling. At each step, only one on-support coordinate is incremented, scattering the on-support logits around their mean. As derived in Appendix~\ref{app:sampling}, the expected variance obeys the differential equation
\begin{equation}
    \label{eqn:V-dynamics-main}
    \frac{\mathrm{d}\,\mathbb{E}V}{\mathrm{d}t} = -\frac{2\,\hat p^\star_S(t)}{k}\,\mathbb{E}V + \frac{k-1}{k^2},
\end{equation}
where $\hat p^\star_S(t)$ is the on-support mass along the gradient-flow trajectory. Noise is injected at rate $(k-1)/k^2$ and damped at rate $2 \hat p^\star_S(t)/k$ by the curvature of the loss.

\vspace{-0.25em}
\paragraph{Convergence rate vs. noise floor.} Equation~\ref{eqn:V-dynamics-main} cleanly separates the two effects of sampling. The mean trajectory of the logits is untouched at leading order in the learning rate: label sampling noise does not slow down convergence, it only scatters the logits around the gradient-flow trajectory.\footnote{This is a statement of the leading-order stochastic differential equation limit. Discrete stochastic gradient descent at finite step size does exhibit rate slowdowns, a standard stochastic optimization result that we review in Appendix~\ref{app:sampling}.} However, sampling changes the asymptotic value of the loss. 
A second-order KL expansion around the deterministic trajectory, evaluated at the stationary value of Equation~\ref{eqn:V-dynamics-main}, yields the noise floor
\titledeq{Root 4}{
    \lim_{t\to\infty}\left(\mathbb{E}[D_\mathrm{KL}^{\,\mathrm{SGD}}(t)] - D_\mathrm{KL}^{\,\mathrm{GF}}(t)\right) = \frac{k-1}{4k}
  }{eqn:sgd-gap-main}
which is zero for deterministic targets and asymptotes to $1/4$ as $k$ grows to infinity. This is another manifestation of the curse of ambiguity, compounding with the slowdown of Section~\ref{subsec:optimization}: early in training the $1/k$ rate dominates; as training progresses, the deterministic loss decays as $k(d-k)/(td)$ until it falls below the noise floor, after which stochastic gradient descent plateaus while gradient flow keeps decreasing. Figure~\ref{fig:learning} confirms this behavior in simulation. Both the onset of the plateau and its value worsen with~$k$. Appendix~\ref{app:sampling-validation} confirms that this $k$-dependent gap also appears when training the linear head with sampled labels on mixed-ambiguity data.
\section{The curse of ambiguity in large language models}
\label{sec:experiments}

Having dissected the roots of the curse of ambiguity in toy settings, we now establish that large language models suffer from it. We first study a synthetic $n$-gram task in which each component of our theory can be ablated in isolation, then turn to pretrained transformers on natural data.

\vspace{-0.4em}
\subsection{Synthetic $n$-gram data}
\vspace{-0.2em}

\begin{figure}
    \centering
    \includegraphics{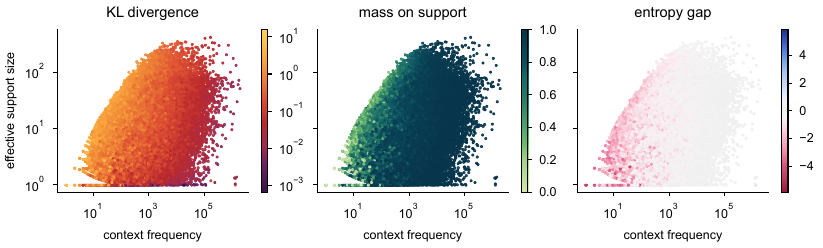}
    \vspace{-0.6cm}
    \caption{\textbf{Transformers trained autoregressively on a $3$-gram task suffer from the curse of ambiguity.} The data follows the empirical $3$-gram distributions from the TinyStories dataset, at the word level. Each dot is a different context (a pair of words), characterized by how many times it appears in the dataset (context frequency) and the effective support size of the true next-token distribution. For each context, we report the average KL divergence between the model's prediction and the true distribution (left), the average probability mass assigned to the support (middle), and the average entropy gap between the true and the predicted distribution (right). The first two panels confirm that the model suffers from the curse of ambiguity: performance degrades as effective support size increases at fixed context frequency. The model's failure mode systematically consists of leaking probability mass outside the support.}
    \label{fig:ngram-confirmation-curse}
\end{figure}

We start with a synthetic language modeling task whose ground-truth conditional distributions are known exactly, letting us measure precisely how predictions deviate from them. To that end, we construct a word-level $3$-gram dataset derived from the TinyStories corpus \citep{eldan_tinystories_2023}. Training sequences are sampled autoregressively from the empirical $3$-gram distribution, and by default, we train a $4$-layer transformer with 4M parameters on 100M tokens; see Appendix~\ref{app:exp_details_ngram} for details.

\vspace{-0.6em}
\paragraph{Small Transformer models suffer from the curse of ambiguity.}
In this synthetic setting, we know exactly which part of the input is relevant for next-token prediction, namely the previous two tokens, and can therefore characterize the complexity of each context precisely. Ignoring correlations between contexts in the $n$-gram table, we summarize this complexity by the number of times the context appears in the training data, which we refer to as its context frequency.
We then analyze how well the model predicts the ground-truth distribution as contexts become more ambiguous, that is, as their effective support size (exponential of the entropy, cf.\ Section~\ref{sec:setup}) grows, at fixed context frequency. Figure~\ref{fig:ngram-confirmation-curse} summarizes this per-context evaluation. These results confirm that transformer models suffer from the curse of ambiguity, as more ambiguous contexts are harder to predict (left panel). The middle and right panels further show that the model fails by systematically leaking probability mass outside the support, leading the sequence-level distribution to have more entropy than the ground truth.

\vspace{-0.6em}
\paragraph{Ablations qualitatively match the toy-model predictions.} To probe whether the predictions of our toy-model theory hold for actual transformers, we ablate the key components identified in Sections~\ref{sec:architecture} and~\ref{sec:learning}: model capacity, embedding dimension, training duration, and whether we minimize the sampled cross-entropy or the full KL divergence to the true conditional distribution. To make this comparison tractable, we summarize each configuration by a single iso-loss contour in the effective support size, context frequency plane; see Appendix~\ref{app:exp_details_ngram} for how this contour is computed. Each intervention then either shifts the contour, indicating a uniform improvement across all effective support sizes, or rotates it, indicating a different effect on small versus large effective support sizes.

Figure~\ref{fig:ablation-roots} shows the result. Increasing the number of training steps produces a uniform shift, consistent with the late-phase regime of Section~\ref{subsec:optimization} where the convergence rate $1/t$ no longer depends on $k$. We see no trace of the early-dynamics regime in which large $k$ should be learned much more slowly, presumably because even our shortest run already lies past that early phase. Capacity also shifts the line uniformly, with saturation at the largest models since training duration is held fixed. This is consistent with theory, as more parameters change the total capacity but not its optimal allocation. In contrast, increasing the embedding dimension and removing sampling noise by directly optimizing the KL divergence rather than a sampled cross-entropy both rotate this line, improving large effective support sizes while leaving deterministic contexts essentially unchanged. This is consistent with the embedding dimension analysis of Section~\ref{subsec:embedding} and the sampling analysis of Section~\ref{subsec:sampling}.

\begin{figure}
    \centering
    \includegraphics{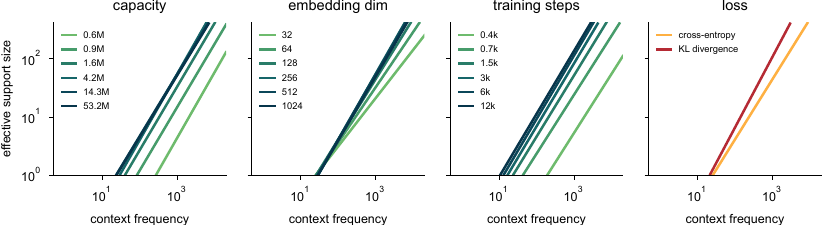}
    \vspace{-0.45cm}
    \caption{\textbf{Transformers suffer from the curse of ambiguity in the same way as linear models.} Each line is an iso-loss contour at $D_\mathrm{KL} = 0.5$, obtained from a log-log linear fit of the per-context loss of Figure~\ref{fig:ngram-confirmation-curse} (left) in effective support size and context frequency, which describes the measurements well (Appendix~\ref{app:exp_details_ngram}). Increasing capacity or the number of training steps improves performance uniformly across all effective support sizes. Increasing the embedding dimension or removing sampling noise by directly optimizing the expected loss over outputs (KL divergence) significantly improves modeling abilities at large effective support sizes, but barely affects deterministic next-token distributions, consistent with the theory developed for the linear case. Setup is the same as in Figure~\ref{fig:ngram-confirmation-curse}; see Appendix~\ref{app:exp_details_ngram} for experimental details.}
    \label{fig:ablation-roots}
\end{figure}

\subsection{Real-world models}
\label{subsec:realworld}

We now show that large language models also suffer from the curse of ambiguity.

\paragraph{Proxies for the ground truth and the context frequency.}
The main obstacle to this analysis is that, on natural data, neither the ground-truth next-token distribution nor the frequency of each training context is observable: past a few tokens, every context appears at most a handful of times.
We therefore replace each by a proxy.
We use a large language model, OLMo-2 13B \cite{olmo_2_2025}, as a stand-in for the ground-truth next-token distribution and evaluate the smaller OLMo-2 1B against it; both checkpoints are taken before any post-training.
From the predictions of the larger model we read off the effective support size of each context by taking the exponential of the entropy of the predicted probability distribution.
To estimate context frequency, we exploit the fact that similar contexts produce similar embeddings. We extract the last-layer hidden state of OLMo-2 13B at every position of $100{,}000$ sequences of length $512$ sampled from the Dolma v1.7 corpus \cite{soldaini_dolma_2024}, and count how many other embeddings are close by; see Appendix~\ref{app:exp_details_realworld} for details.
This yields an estimate of the number of semantically similar contexts seen during training, which we call the \emph{semantic context frequency}.
Both proxies are biased, since OLMo-2 13B is itself far from the true distribution of natural text and embedding distance is a coarse measure of context similarity.
We nevertheless show in Appendix~\ref{app:exp_details_realworld} that, on the $3$-gram task of the previous section where the ground truth is known, the semantic context frequency correlates with the true context frequency.

\paragraph{Large language models suffer from the curse of ambiguity.}
Figure~\ref{fig:realworld-curse} reports the per-context KL divergence between OLMo-2 13B and OLMo-2 1B in the estimated effective support size, semantic context frequency plane.
At fixed semantic context frequency, more ambiguous contexts are harder to model, mirroring Figure~\ref{fig:ngram-confirmation-curse}; Figure~\ref{fig:realworld-binned-correlation-perplexity} in the Appendix quantifies this trend.
One qualitative difference with the synthetic setup is worth flagging: the low-frequency region contains a long tail of outliers, contexts seen rarely yet on which the two models agree closely. Inspecting them shows that they rely on in-context copying mechanisms that generalize to unseen sequences; see Appendix~\ref{app:realworld_in_context_band}.

\begin{figure}
    \centering
    \begin{minipage}[c]{0.44\linewidth}
        \includegraphics{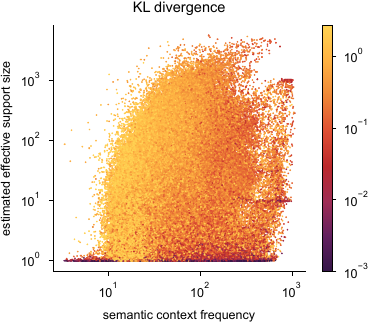}
    \end{minipage}\hfill
    \begin{minipage}[c]{0.53\linewidth}
        \caption{\textbf{Large language models suffer from the curse of ambiguity.} Same per-context evaluation as Figure~\ref{fig:ngram-confirmation-curse} (left), except that we now evaluate OLMo-2 1B \citep{olmo_2_2025} on natural data sampled from Dolma~v1.7 \citep{soldaini_dolma_2024}. The effective support size and semantic context frequency are estimated from the larger OLMo-2 13B model, a proxy for the ground-truth distribution. Overall, we recover the same behavior as in Figure~\ref{fig:ngram-confirmation-curse}: at fixed frequency, more ambiguous contexts are harder to predict. Appendix~\ref{app:quantification_correlation} quantitatively shows that the log KL divergence and the log effective support size at fixed semantic context frequency are correlated. See Section~\ref{subsec:realworld} and Appendix~\ref{app:realworld_in_context_band} for a discussion of what the outliers in the bottom left of the plot correspond to, and Appendix~\ref{app:exp_details_realworld} for experimental details.}
        \label{fig:realworld-curse}
    \end{minipage}
\end{figure}

\section{Discussion}

\paragraph{The power of toy models.} Our paper illustrates how studying simplified models and data can uncover phenomena that are hard to isolate at scale, by exposing controlled axes of complexity that point to the right setups, observables, and ablations to use in real models. Such studies are one strand of a broader movement, sometimes referred to as the physics or mechanics of learning \cite{bahri_statistical_2020, saxe_if_2021, ganguli_physics_nodate, simon_there_2026}, which has already yielded insights into learning dynamics~\cite{saxe_exact_2014, saxe_mathematical_2019, saxe_neural_2022, zucchet_emergence_2025}, scaling laws~\citep{allen-zhu_physics_2025, cabannes_scaling_2024, kaplan_scaling_2020, michaud_quantization_2023, cagnetta_deriving_2026, liu_universal_2026}, in-context learning~\cite{olsson_-context_2022, chan_data_2022, garg_what_2022, oswald_transformers_2023, oswald_uncovering_2024, reddy_mechanistic_2024}, or representation geometry~\cite{elhage_toy_2022, liu_superposition_2025}. Our work adds the discovery of the curse of ambiguity to this list\footnote{In fact, we first formulated the curse of ambiguity hypothesis while playing with one-hidden-layer multilayer perceptrons trained on synthetic ambiguous distributions.}.

\paragraph{How relevant is the curse of ambiguity in practice?} While the curse of ambiguity is clearly visible in the large-scale models we study (Section~\ref{subsec:realworld}), the picture there is less clean than in synthetic settings, in part because of a property of natural data that our analysis sets aside: structure. Indeed, our analysis implicitly assumes that a context is learned at a rate proportional to how often it appears; structural similarities between contexts, for example synonym outputs, short-circuit this assumption, allowing rare contexts to be accurately predicted via shared, compressed mechanisms. Under ambiguity, where the memorization route is slow, such structured solutions have more time to emerge, whereas for deterministic transitions memorization is fast enough that such pressure is minimal, as per the neural race argument \cite{saxe_neural_2022}. A corollary is that neural networks are more likely to memorize sequences with deterministic, unique transitions than stochastic ones; we leave this interesting hypothesis for future work. Our analysis reveals that scaling mitigates the curse of ambiguity, but only asymptotically. Even at the largest scale a residual fraction of the distribution is left unmodeled \cite{kaplan_scaling_2020, hoffmann_training_2022}, and the curse of ambiguity will be particularly relevant in this regime. Fine-tuning distributions that significantly differ from the pretraining distribution should also remain particularly exposed. 

\paragraph{Not all neural networks are language models.} While we have framed this work around large language models, our results extend well beyond them. The theoretical arguments of Sections~\ref{sec:architecture} and~\ref{sec:learning} apply to any neural network that produces discrete probability distributions, and the curse of ambiguity may be an even greater concern in settings where such networks operate at much smaller scales. Beyond artificial networks, the cognitive psychology literature has long documented that humans struggle when the environment admits many plausible options. In the classical paradox of choice~\citep{iyengar_when_2000, schwartz_paradox_2015}, increasing the number of viable alternatives degrades the speed and the quality of decision-making, and lowers satisfaction with the eventual choice. Our results offer a complementary mechanistic perspective: representing a distribution that spreads mass over many continuations is fundamentally more demanding, in capacity and in samples, than representing one that concentrates on a single answer.

\section*{Acknowledgments}
We thank Yizhou Liu, Noah Cowan and members of the Linderman lab for helpful feedback.
N.Z. is supported by Postdoc.Mobility grant P500-2\_235376 from the Swiss National Science Foundation.
S.W.L. is supported by grants from the NIH (U01NS136507, R01NS131987, R01NS113119, RF1MH133778, R01AG097491, \& R01NS130789), the NSF (2440859), and the Simons and McKnight Foundations.

The experiments were performed on the Stanford Marlowe computing cluster \cite{kapfer_marlowe_2025}. We thank Stanford University and the Stanford Research Computing Center for providing computational resources and support that contributed to these research results.

\section*{Code availability}

The code to reproduce all experiments is available at \url{https://github.com/NicolasZucchet/Curse-of-ambiguity}.

{
\small

\bibliography{references}
\bibliographystyle{unsrtnat}

}

\newpage
\appendix

\renewcommand{\ptctitle}{Table of contents}

\setcounter{tocdepth}{3}  
\setcounter{parttocdepth}{3}  
\renewcommand \thepart{}
\renewcommand \partname{}
\renewcommand \thepart{Appendix}
\addcontentsline{toc}{section}{Appendix} 
\part{} 
\parttoc 

\newpage

\newcounter{robustinformalnum}
\newcounter{tempthmsave}

\section*{Notation}
\addcontentsline{toc}{section}{Notation}

\begin{center}
\renewcommand{\arraystretch}{1.15}
\setlength{\tabcolsep}{8pt}
\begin{tabular}{rl}
    \specialrule{0.12em}{0pt}{4pt}
    \textbf{Symbol} & \textbf{Meaning} \\
    \specialrule{0.12em}{3pt}{2pt}
    \multicolumn{2}{l}{\emph{Problem setup (Section~\ref{sec:setup})}} \\
    \midrule
    $d$ & vocabulary size \\
    $k$ & support size of the target distribution (ambiguity level) \\
    $h$ & embedding dimension \\
    $n$ & number of contexts \\
    $S_i \subseteq [d]$ & support of the target distribution for context $i$ \\
    $p_i \in \Delta^{d-1}$ & ground-truth distribution for context $i$ (uniform on $S_i$) \\
    $\hat p_i = \mathrm{softmax}(W e_i) \in \Delta^{d-1}$ & model's predicted distribution for context $i$ \\
    $e_i \in \mathbb{R}^h$ & embedding of context $i$ \\
    $f_\theta : [d]^{t-1} \to \mathbb{R}^h$ & backbone mapping a context to its embedding \\
    $W \in \mathbb{R}^{d \times h}$ & weight matrix of the linear head \\
    $\ell_i = W e_i \in \mathbb{R}^d$ & logits for context $i$ \\
    \midrule
    \multicolumn{2}{l}{\emph{Capacity (Appendix~\ref{app:capacity})}} \\
    \midrule
    $\beta$ & scalar factor on the Hebbian weights \\
    $\rho = k/d$ & support density \\
    $\xi_j = e_j^\top e_i$ & pairwise overlap between random embeddings \\
    $\Xi_i = \sum_{j \neq i} z_j\, e_j^\top e_i$ & cross-context noise in the logits \\
    $\sigma^2 = (n-1)\rho(1-\rho)/h$ & per-component variance of $\Xi_i$ \\
    $n_{\mathrm{top\text{-}k}}$ & top-$k$ capacity \\
    $n_\varepsilon$ & $\varepsilon$ capacity \\
    \midrule
    \multicolumn{2}{l}{\emph{Embedding dimension (Appendix~\ref{app:embedding})}} \\
    \midrule
    $\tau$ & magnitude of an adversarial embedding perturbation \\
    $\varepsilon$ & KL slack in the robust representation result \\
    \midrule
    \multicolumn{2}{l}{\emph{Learning dynamics (Appendices~\ref{app:optimization}--\ref{app:sampling})}} \\
    \midrule
    $\Delta \ell \in \mathbb{R}$ & gap between any on- and off-support logit (gradient-flow scalar) \\
    $\ell_t^\star \in \mathbb{R}^d$ & gradient-flow trajectory of the logits at time $t$ \\
    $\hat p^\star_S(t) = \sum_{j \in S} \hat p_j^\star(t) \in [0, 1]$ & on-support mass along $\ell_t^\star$ \\
    $\eta$ & learning rate of stochastic gradient descent \\
    $\delta_t = \ell_t - \ell_t^\star \in \mathbb{R}^d$ & deviation of the stochastic iterates from gradient flow \\
    $V = \tfrac{1}{k}\sum_{j \in S}(\ell_j - \bar \ell_S)^2 \in \mathbb{R}$ & within-on-support variance of the logits \\
    $H(\ell) = \mathrm{diag}(\hat p) - \hat p \hat p^\top \in \mathbb{R}^{d \times d}$ & Hessian of the softmax cross-entropy \\
    \midrule
    \multicolumn{2}{l}{\emph{Empirical analysis (Section~\ref{sec:experiments})}} \\
    \midrule
    $D_{\mathrm{KL}}(p \,\|\, \hat p)$ & Kullback--Leibler divergence \\
    $H(p)$ & Shannon entropy of $p$ \\
    $\exp(H(p))$ & effective support size of $p$ \\
    $n_{\mathrm{est}}$ & semantic context frequency from kNN density on embeddings \\
    $\lambda$ & bandwidth of the exponential kernel used to compute $n_{\mathrm{est}}$ \\
    \bottomrule
\end{tabular}
\end{center}

\newpage

\section{Theoretical results regarding capacity}
\label{app:capacity}

\subsection{Definitions of top-$k$ and $\varepsilon$ capacity}
\label{subsec:capacity-def}

When $k=1$, the output is deterministic and the notion of correct storage is clear: the model must predict the correct output. In the stochastic case, multiple criteria are possible, and we consider the following two.

\begin{definition}[top-$k$ capacity]
    The top-$k$ capacity $n_{\mathrm{top\text{-}k}}$ is the largest $n$ such that the top-$k$ predictions of the model over all contexts $i$ match the support of the true conditional distributions.
\end{definition}

\begin{definition}[$\varepsilon$ capacity]
    \label{def:eps-capacity}
    The $\varepsilon$-capacity $n_\varepsilon$ is the largest $n$ such that the KL divergence $\mathrm{KL}(p_i \| \hat{p}_i)$ stays below $\varepsilon$ for every context $i$.
\end{definition}

These two capacity notions are related in the sense that the top-$k$ capacity is always at least as large as the $\varepsilon$-capacity for small enough $\varepsilon$, since accurately matching the full distribution requires in particular identifying the correct support. The $\varepsilon$-capacity is thus the more demanding notion: it asks not only that the support be identified, but that the $k$ probabilities within it be nearly equal.

\subsection{Setup, Hebbian model and signal-noise decomposition of the model's logits}

Let us quickly recall the setup we are considering, the Hebbian model we study as well as the notations we have introduced in the main text.

We have $n$ input contexts, each context $i$ having a random embedding $e_i \in \mathbb{R}^h$ sampled iid and uniformly from the unit sphere $\mathbb{S}^{h-1}$. Context $i$ is associated with a ground truth output distribution $p_i$ that is uniform on a random support $S_i \subseteq [d]$ of size $k$:
\begin{equation}
    p_{i,j} = \begin{cases} 1/k & \text{if } j \in S_i \\ 0 & \text{otherwise.} \end{cases}
\end{equation}

We consider a Hebbian model whose predicted output distribution for context $i$ is $\hat{p}_i = \mathrm{softmax}(W e_i)$ with weight matrix
\begin{equation}
    W = \beta \sum_{j=1}^{n} z_{j}\, e_{j}^\top \quad\text{with}\quad z_{j,\ell} = \begin{cases} 1 - \rho & \text{if } \ell \in S_{j} \\ -\rho & \text{otherwise,}\end{cases}
\end{equation}
where $\rho := k/d$.

The logits of the model $\ell_i$ are given by
\begin{equation}
    \ell_i = W e_i = \beta z_i (e_i^\top e_i) + \beta \sum_{j \neq i} z_{j} (e_{j}^\top e_i)
\end{equation}
Since the embeddings are normalized, $e_i^\top e_i = 1$ exactly and the signal term is simply $\beta z_i$. The cross-terms $\xi_{j} := e_{j}^\top e_i$ are mean-zero and sub-Gaussian\footnote{This is a standard result for dot products of random embeddings from the unit sphere; see~\citet{vershynin_high-dimensional_2018}.} with variance $1/h$. We have additionally introduced the cross-context noise
\begin{equation}
    \Xi_i = \sum_{j \neq i} z_{j} e_{j}^\top e_i
\end{equation}
and so the logits are
\begin{equation}
    \ell_i = \beta\bigl(z_i + \Xi_i\bigr).
\end{equation}

Since context embeddings and target distributions are independent, the cross-context noise is a sum of $n-1$ independent random variables, whose first and second order moments are
\begin{align}
    \mathbb{E} \left [ \Xi_{i} \right ] & = \sum_{j \neq i} \mathbb{E} \left [z_{j} \right ] \mathbb{E} \left [ \xi_{j} \right ] = 0\\
    \mathrm{Cov} \left [ \Xi_{i} \right ] & = \frac{(n - 1)\,\rho (1 - \rho)\,d}{h\,(d-1)} \left (I - \frac{1}{d}\mathbf{1} \mathbf{1}^\top \right).
\end{align}
The covariance calculation uses the fact that membership in $S_i$ of two outputs $j$ and $j'$ are not independent (they are drawn without replacement), inducing off-diagonal terms; we omit the details. It follows that the per-component variance of the noise is
\begin{equation}
    \sigma^2 := \mathrm{Var}[\Xi_{i,j}] = \frac{(n-1)\,\rho(1-\rho)}{h}
\end{equation}

\subsection{Top-$k$ capacity of the Hebbian model}

In this section, we demonstrate that the top-$k$ capacity of the Hebbian construction scales as $\frac{h}{\rho(1-\rho)}$ up to some logarithmic factor, thereby providing theoretical grounding for the scaling that we have discussed in the main text.

\begin{theorem}[Top-$k$ capacity lower bound]
    \label{thm:topk-capacity}
    There exists a universal constant $C > 0$ such that the following holds. For any $\delta \in (0,1)$, if $h \geq C \log^2(hd/\delta)$ and
    \begin{equation}
        \label{eqn:condition-topk-capacity}
        n \leq \frac{h}{C\,\rho(1-\rho)\,\log(hd/\delta)}
    \end{equation}
    then with probability at least $1 - \delta$ over the draw of embeddings and supports, the top-$k$ predictions of the Hebbian model are correct for all $n$ contexts:
    \begin{equation}
        \forall\, i \in [n]\quad \operatorname{top}_k\bigl(W e_i\bigr) = S_i.
    \end{equation}
\end{theorem}

To prove this statement, we will make use of the following concentration bound that relates the maximal magnitude of the noise $\Xi$ (over $i$ and $j$) to its variance. We will prove it after Theorem~\ref{thm:topk-capacity}.

\begin{lemma}[Uniform noise bound]
    \label{lem:uniform-noise-bound}
    There exists a universal constant $C > 0$ such that for any $\delta \in (0,1)$, with probability at least $1 - \delta$,
    \begin{equation}
        \max_{i \in [n], j \in [d]} |\Xi_{i,j}| \leq C \sqrt{\left(\sigma^2 + \frac{\log(nd/\delta)}{h}\right) \log(nd/\delta)}.
    \end{equation}
\end{lemma}

\begin{proof}[Proof of Theorem~\ref{thm:topk-capacity}]
The model correctly identifies the support $S_i$ if and only if every on-support logit exceeds every off-support logit. Substituting $z_{i,j} = 1 - \rho$ for $j \in S_i$ and $z_{i,j'} = -\rho$ otherwise, this simplifies to
\begin{equation}
    \forall\, j \in S_i,\; j' \notin S_i\quad \Xi_{i,j'} - \Xi_{i,j} < 1.
\end{equation}
A sufficient condition is $\max_{i,j} |\Xi_{i,j}| < 1/2$, and by Lemma~\ref{lem:uniform-noise-bound} this holds with probability at least $1-\delta$ provided
\begin{equation}
    \left(\sigma^2 + \frac{\log(nd/\delta)}{h}\right) \log(nd/\delta) \leq \frac{1}{4C^2}.
\end{equation}
Splitting the two summands evenly, it holds whenever
\begin{equation}
    \label{eqn:two-conditions}
    \sigma^2 \log(nd/\delta) \leq \frac{1}{8C^2} \; \text{ and } \;\frac{\log^2(nd/\delta)}{h} \leq \frac{1}{8C^2}.
\end{equation}
The hypotheses of the theorem, after possibly enlarging the universal constant~$C$, imply both inequalities. First, under Eq.~\ref{eqn:condition-topk-capacity} we have 
\begin{equation}
    n - 1 \leq n \leq h/(C\rho(1-\rho)\log(hd/\delta)).
\end{equation}
Combined with
\begin{equation}
    \rho(1-\rho) \geq \frac{d-1}{d^2} \geq \frac{1}{2d}
\end{equation}
for $1 \leq k \leq d-1$ and $d \geq 2$, this gives
\begin{equation}
    n \leq 2hd/(C\log(hd/\delta)) \leq hd
\end{equation}
for $C$ large enough, and therefore $\log(nd/\delta) \leq 2\log(hd/\delta)$. The first inequality of Eq.~\ref{eqn:two-conditions} thus reduces, via $\sigma^2 = (n-1)\rho(1-\rho)/h \leq n\rho(1-\rho)/h$, to Eq.~\ref{eqn:condition-topk-capacity}; the second reduces to the hypothesis $h \geq C\log^2(hd/\delta)$. Both reductions only modify the universal constant by a fixed factor.
\end{proof}

\begin{proof}[Proof of Lemma~\ref{lem:uniform-noise-bound}]
The variance of $\Xi_{i,j}$ depends on the random supports (through the coefficients $z_{\ell,j}$), so we proceed in two steps: first control the variances uniformly via Bernstein's inequality and a union bound over the supports, then apply sub-Gaussian tail bounds conditionally on the supports.

\medskip
\noindent\textbf{Step 1: Uniform variance control.}
Without loss of generality we assume $\rho \leq 1/2$ (i.e., $k \leq d/2$); the complementary case follows by the symmetry of $\rho(1-\rho)$. Conditional on the supports, each $\xi_{\ell} = e_{\ell}^\top e_i$ is sub-Gaussian with variance proxy $1/h$ (a standard property of dot products of independent uniform unit-sphere vectors) and the coefficients $z_{\ell,j}$ are deterministic, so $\Xi_{i,j} = \sum_{\ell \neq i} z_{\ell,j} \xi_{\ell}$ is sub-Gaussian with variance proxy
\begin{equation}
    \sigma_{i,j}^2 = \frac{1}{h}\sum_{\ell \neq i} z_{\ell,j}^2.
\end{equation}
The values $z_{\ell,j} \in \{1-\rho, -\rho\}$ give
\begin{align}
    \sum_{\ell \neq i} z_{\ell,j}^2
    &= M_{i,j}(1-\rho)^2 + (n - 1 - M_{i,j})\rho^2 \nonumber\\
    &= M_{i,j}(1 - 2\rho) + (n-1)\rho^2,
\end{align}
where $M_{i,j} := |\{\ell \neq i \mid j \in S_{\ell}\}| \sim \operatorname{Binomial}(n-1, \rho)$, with mean $(n-1)\rho$ and variance at most $(n-1)\rho$. Letting $L := \log(2nd/\delta)$, Bernstein's inequality gives, with probability at least $1 - \delta/(2nd)$,
\begin{equation}
    M_{i,j} \leq (n-1)\rho + \sqrt{2(n-1)\rho L} + \tfrac{2}{3} L \leq 2(n-1)\rho + 2L,
\end{equation}
where the last inequality uses $\sqrt{2(n-1)\rho L} \leq (n-1)\rho + L/2$ (AM-GM). On this event, since $1 - 2\rho \geq 0$,
\begin{align}
    \sigma_{i,j}^2
    &\leq \frac{[2(n-1)\rho + 2L](1-2\rho) + (n-1)\rho^2}{h} \nonumber\\
    &\leq \frac{2(n-1)\rho(1-\rho) + 2L}{h}\\
    &= 2\sigma^2 + \frac{2L}{h}.
\end{align}
A union bound over the $nd$ pairs ensures this holds simultaneously for all $(i,j)$ with failure probability at most $\delta/2$.

\medskip
\noindent\textbf{Step 2: Sub-Gaussian tail bound.} Fix a pair $(i,j)$ and work on the event from Step~1, which is measurable with respect to the supports. Condition on the supports $\{S_\ell\}$ and on the single embedding~$e_i$; note that we cannot condition on all the embeddings, as that would leave no randomness in $\Xi_{i,j}$. Given these, $\Xi_{i,j} = \sum_{\ell\neq i} z_{\ell,j} \xi_{\ell}$ is a sum of independent sub-Gaussian terms with deterministic coefficients and variance proxy at most $2\sigma^2 + 2L/h$, since the $e_\ell$ with $\ell \neq i$ are independent of $e_i$. The sub-Gaussian tail bound thus gives
\begin{equation}
    \mathbb{P}\left(|\Xi_{i,j}| > t \;\Big|\; \{S_\ell\}, e_i\right) \leq 2 \exp\left(-\frac{c t^2}{2\sigma^2 + 2L/h}\right).
\end{equation}
The right-hand side does not depend on~$e_i$, so averaging over~$e_i$ removes that conditioning and the same bound holds given the supports alone. A union bound over the $nd$ pairs then gives
\begin{equation}
    \mathbb{P}\left(\max_{i,j}|\Xi_{i,j}| > t \;\Big|\; \{S_\ell\}\right) \leq 2nd \exp\left(-\frac{c t^2}{2\sigma^2 + 2L/h}\right).
\end{equation}
Since this right-hand side does not depend on $\{S_\ell\}$ either, the law of total probability removes the remaining conditioning, and combining with the $\delta/2$ failure probability from Step~1 gives
\begin{equation}
    \mathbb{P}\left(\max_{i,j}|\Xi_{i,j}| > t\right) \leq \frac{\delta}{2} + 2nd \exp\left(-\frac{c t^2}{2\sigma^2 + 2L/h}\right).
\end{equation}
Choosing $t = C \sqrt{\left(\sigma^2 + L/h\right) \log(nd/\delta)}$ with $C$ sufficiently large makes both terms at most $\delta/2$.
\end{proof}

\paragraph{Symmetry.} Since $\rho(1-\rho)$ is symmetric in $k$ and $d-k$, predicting which $k$ outputs are in the support is equally hard as predicting which $d-k$ are out of it.

\paragraph{Scaling.} Since $\rho(1-\rho) = k(d-k)/d^2$, Eq.~\ref{eqn:condition-topk-capacity} gives the following scalings. Throughout, $\lesssim$ hides the universal constant $C$, and $\log(hd/\delta)$ is the single logarithmic factor of Eq.~\ref{eqn:condition-topk-capacity}, which at fixed confidence level $\delta$ is a logarithm of the number of parameters.
\begin{itemize}
\item \textbf{$k = 1$ (deterministic):} $\rho(1-\rho) = (d-1)/d^2 \approx 1/d$, so $n \lesssim hd/\log(hd/\delta)$ and we recover, up to the logarithmic factor, the $hd$ scaling (or $d^2$ when $h = d$) of \citet{nichani_understanding_2025}.
\item \textbf{$k \ll d$ (sparse support):} $\rho(1-\rho) \approx k/d$, so $n \lesssim hd/(k \log(hd/\delta))$. Storing a distribution with support size $k$ costs $k$ times more capacity than a deterministic association; \emph{this is the capacity-based curse of ambiguity}.
\item \textbf{$k = d/2$ (maximal ambiguity):} $\rho(1-\rho) = 1/4$, so capacity reaches its minimum $n \lesssim h/\log(hd/\delta)$: the logarithmic factor is the only remaining dependence on the vocabulary size.
\item \textbf{$k$ close to $d$:} by the symmetry above, $k = d - 1$ brings the bound back to $n \lesssim hd/\log(hd/\delta)$, and more generally $n \lesssim hd/((d-k)\log(hd/\delta))$ for $d - k \ll d$. The limiting case $k = d$ is excluded from the theorem and is degenerate: there is a single possible support, so all distributions coincide and nothing has to be stored.
\end{itemize}

\subsection{Formal statement and proof for \texorpdfstring{$\varepsilon$}{epsilon}-capacity}

We now turn to the $\varepsilon$-capacity, which requires not just correct support identification but accurate approximation of the uniform distribution on the support.

When $k=1$ the target is deterministic and correct support identification is already enough to drive the KL to zero: once the top-$1$ prediction is correct, sending $\beta \to \infty$ pushes the predicted probability of the supported token to $1$. The interesting regime is therefore $k \ge 2$, where the challenge is no longer merely to separate the support from its complement: one must also make the logits nearly equal within the support.

\begin{theorem}[$\varepsilon$-capacity lower bound]
    \label{thm:eps-capacity}
    There exists a universal constant $C > 0$ such that the following holds. For any $\varepsilon \in (0,1)$ and $\delta \in (0,1)$, if $2 \le k \le d/2$, $h \geq \frac{C \log^2(d/(k\varepsilon)) \log^2(hd/\delta)}{\varepsilon}$, and
    \begin{equation}
        \label{eqn:condition-eps-capacity}
        n \leq \frac{h\,d\,\varepsilon}{C\,k\,\log^2\!\bigl(\frac{d}{k\varepsilon}\bigr)\,\log\!\bigl(\frac{hd}{\delta}\bigr)}
    \end{equation}
    then with probability at least $1 - \delta$ over the draw of embeddings and supports, every context is accurately modeled:
    \begin{equation}
        \max_{i \in [n]} \mathrm{KL}(p_i \| \hat{p}_i) \le \varepsilon,
    \end{equation}
    that is, $n_\varepsilon \geq n$.
\end{theorem}

The proof of Theorem~\ref{thm:eps-capacity} relies on the following deterministic bound on the per-context KL in terms of two elementary geometric quantities of the logit vector: its on-support variation and its margin. We state it once here, in a form that will also be invoked in the embedding-dimension appendix (Appendix~\ref{subsec:embedding-robust}). We prove it after the theorem.

\begin{lemma}[KL control via margin and variation]
    \label{lem:KL-control}
    Let $S \subseteq [d]$ have size~$k$ and $\ell \in \mathbb{R}^d$ be a logit vector. Denote the \emph{on-support variation} and \emph{margin} of $\ell$ at~$S$ by
    \begin{align}
        \label{eqn:nu-def}
        \nu(S, \ell) &:= \max_{j, j' \in S} \bigl|\ell_j - \ell_{j'}\bigr|,\\
        \label{eqn:mu-def}
        \mu(S, \ell) &:= \tfrac{1}{k}\textstyle\sum_{j \in S}\ell_j \;-\; \max_{j' \notin S}\ell_{j'}.
    \end{align}
    If $\nu(S, \ell) \leq 1$, then
    \begin{equation}
        \label{eqn:kl-quadratic-bound}
        D_{\mathrm{KL}}\!\left(p_S \,\middle\|\, \mathrm{softmax}(\ell)\right)
        \;\leq\; \mathrm{e}\,\nu(S, \ell)^2 \;+\; \frac{d-k}{k}\,e^{-\mu(S, \ell)}.
    \end{equation}
\end{lemma}

\begin{proof}[Proof of Theorem~\ref{thm:eps-capacity}]
We proceed in four steps:
\begin{enumerate}
    \item Fix $\beta$ and compute the per-context variation and margin.
    \item Apply the noise bound.
    \item Apply the KL-control lemma.
    \item Group everything together and conclude.
\end{enumerate}

\medskip
\noindent\textbf{Step 1: Fix $\beta$ and compute variation and margin.} Set
\begin{equation}
    \beta = \log\frac{8(d-k)}{k\varepsilon}.
\end{equation}
This is a legitimate inverse temperature: $k \leq d/2$ gives $d - k \geq k$ and $\varepsilon < 1$, so $\beta \geq \log(8/\varepsilon) \geq \log 8 > 0$.
Fix a context $i$ and drop the $i$ subscript for readability. With $z_j = 1 - \rho$ for $j \in S = S_i$ and $z_j = -\rho$ otherwise, the logits $\ell_j = \beta(z_j + \Xi_j)$ have on-support variation and margin
\begin{align}
    \nu(S, \ell) &= \beta\max_{j, j' \in S}|\Xi_j - \Xi_{j'}| \;\leq\; 2\beta M,\\
    \mu(S, \ell) &= \beta + \beta\bigl(\tfrac{1}{k}\textstyle\sum_{j \in S}\Xi_j - \max_{j' \notin S}\Xi_{j'}\bigr) \;\geq\; \beta(1 - 2M),
\end{align}
where $M := \max_{i\in[n], j\in[d]}|\Xi_{i,j}|$ is the uniform noise magnitude.

\medskip
\noindent\textbf{Step 2: Apply the noise bound.} By Lemma~\ref{lem:uniform-noise-bound}, with probability at least $1 - \delta$,
\begin{equation}
    M \leq C \sqrt{\left(\sigma^2 + \frac{\log(nd/\delta)}{h}\right) \log(nd/\delta)}.
\end{equation}
The assumption $\nu(S, \ell) \leq 1$ required by Lemma~\ref{lem:KL-control} is implied by $2\beta M \leq 1$, that is, by $M \leq 1/(2\beta)$. Squaring this requirement and inserting the bound on $M$ above, it is enough that
\begin{equation}
    \label{eqn:sigma-beta-condition}
    \left(\sigma^2 + \frac{\log(nd/\delta)}{h}\right) \log(nd/\delta) \leq \frac{1}{C\beta^2},
\end{equation}
the $\beta^{-2}$ coming from squaring $M \leq 1/(2\beta)$ and the universal constant $C$ absorbing the resulting numerical factor.
Since $\sigma^2 \asymp n\rho(1-\rho)/h$, the $\sigma^2$ part of Eq.~\ref{eqn:sigma-beta-condition} is implied by Eq.~\ref{eqn:condition-eps-capacity} and the $\log(nd/\delta)/h$ part by the assumption on $h$, both after adjusting the universal constant. Under Eq.~\ref{eqn:condition-eps-capacity} we also have $n \leq hd$, so $\log(nd/\delta) \leq C\log(hd/\delta)$.

\medskip
\noindent\textbf{Step 3: Apply the KL-control lemma.} On the event from Step~2, Lemma~\ref{lem:KL-control} combined with Step~1 gives, for every context~$i$,
\begin{equation}
    \label{eqn:bound-KL}
    \mathrm{KL}(p_i\|\hat p_i)
    \;\leq\; \mathrm{e}\,(2\beta M)^2 \;+\; \frac{d-k}{k}\,e^{-\beta(1 - 2M)}
    \;=\; 4\mathrm{e}\,\beta^2 M^2 \;+\; \frac{d-k}{k}\,e^{-\beta}\,e^{2\beta M}.
\end{equation}

\medskip
\noindent\textbf{Step 4: Conclude.} By the choice of~$\beta$ we have $\frac{d-k}{k}e^{-\beta} = \frac{\varepsilon}{8}$, and Eq.~\ref{eqn:sigma-beta-condition} gives $2\beta M \leq 1$, hence $e^{2\beta M} \leq \mathrm{e}$. The second term of Eq.~\ref{eqn:bound-KL} is therefore at most $\frac{\mathrm{e}}{8}\varepsilon \leq \frac{\varepsilon}{2}$, and since Eq.~\ref{eqn:bound-KL} holds simultaneously for all contexts on the event of Step~2,
\begin{equation}
    \max_{i \in [n]} \mathrm{KL}(p_i\|\hat p_i)
    \;\leq\; 4\mathrm{e}\,\beta^2 M^2 \;+\; \frac{\varepsilon}{2}.
\end{equation}
It remains to make the first term at most $\varepsilon/2$ as well. By the noise bound from Step~2, $M^2 \leq C \left(\sigma^2 + \frac{\log(hd/\delta)}{h}\right) \log(hd/\delta)$, so it suffices that
\begin{equation}
    \label{eqn:noise-variance-condition}
    \beta^2 \sigma^2 \log\!\left(\frac{hd}{\delta}\right) \leq \frac{\varepsilon}{C} \quad \text{and} \quad \beta^2 \frac{\log^2(hd/\delta)}{h} \leq \frac{\varepsilon}{C}
\end{equation}
for a large enough universal constant~$C$. Using $\sigma^2 \asymp n\rho(1-\rho)/h$ and $\beta^2 = O(\log^2(d/(k\varepsilon)))$, the first part of Eq.~\ref{eqn:noise-variance-condition} becomes exactly Eq.~\ref{eqn:condition-eps-capacity}, and the second part is implied by the assumption on $h$. The two halves then add up to $\varepsilon$.
\end{proof}

\begin{proof}[Proof of Lemma~\ref{lem:KL-control}]
Write $\bar\ell := \tfrac{1}{k}\sum_{j \in S} \ell_j$ and $\phi_j := \ell_j - \bar\ell$ for $j \in S$. By definition, $\sum_{j \in S} \phi_j = 0$, $|\phi_j| \leq \nu(S, \ell)$, and $\bar\ell - \max_{j' \notin S}\ell_{j'} = \mu(S, \ell) =: \Delta$. Since $p_S$ is uniform on~$S$,
\begin{align}
    D_{\mathrm{KL}}\!\left(p_S \,\middle\|\, \mathrm{softmax}(\ell)\right)
    &= -\bar\ell + \log\!\left(\sum_{j' \in [d]} e^{\ell_{j'}}\right) - \log k \\
    &= \log\!\left(\tfrac{1}{k}\sum_{j \in S} e^{\phi_j} \;+\; \tfrac{1}{k}\sum_{j' \notin S} e^{\ell_{j'} - \bar\ell}\right).
\end{align}
For the on-support sum, use the inequality $e^u \leq 1 + u + \mathrm{e}\,u^2$ valid for $|u| \leq 1$ together with $\sum_{j \in S}\phi_j = 0$:
\begin{equation}
    \tfrac{1}{k}\sum_{j \in S} e^{\phi_j} \;\leq\; 1 + \tfrac{\mathrm{e}}{k}\sum_{j \in S} \phi_j^2 \;\leq\; 1 + \mathrm{e}\,\nu(S, \ell)^2.
\end{equation}
For the off-support sum, bound each term by $e^{-\Delta}$:
\begin{equation}
    \tfrac{1}{k}\sum_{j' \notin S} e^{\ell_{j'} - \bar\ell} \;\leq\; \frac{d - k}{k}\,e^{-\Delta}.
\end{equation}
Combining and using $\log(1 + x) \leq x$ yields Eq.~\ref{eqn:kl-quadratic-bound}.
\end{proof}

\subsection{Matching upper bound via sign-pattern counting}
\label{subsec:capacity-upper-bound}

Theorem~\ref{thm:topk-capacity} shows that the Hebbian construction stores at least roughly $h/\rho(1-\rho)$ patterns, with $\rho = k/d$.
We now establish a matching upper bound that applies to any linear-softmax model: with high probability over the draw of embeddings and supports, no weight matrix $W \in \mathbb{R}^{d \times h}$ can store substantially more patterns than this.
Together, the two bounds pin down the top-$k$ capacity up to logarithmic factors:
\begin{equation}
    \label{eqn:topk-capacity-tight}
    \mathrm{capacity} \;\asymp\; \frac{h}{\rho(1-\rho)}.
\end{equation}
More formally, the theorem we prove in this section is the following upper bound.
\begin{theorem}[Top-$k$ capacity upper bound]
\label{thm:topk-capacity-upper}
There exists a universal constant $C > 0$ such that, for any $\delta \in (0,1)$ and any $k < d$, if
\begin{equation}
    \label{eqn:condition-topk-capacity-upper}
    n \geq \frac{C \left( hd \log(hd) + \log(1/\delta) \right)}{\log\binom{d}{k}},
\end{equation}
then with probability at least $1 - \delta$ over the embeddings $\{e_i\}$ and supports $\{S_i\}$, no weight matrix $W \in \mathbb{R}^{d \times h}$ achieves $\mathrm{top}_k(W e_i) = S_i$ for all $i \in [n]$.
\end{theorem}

\paragraph{High-level idea.}
The proof of the theorem relies on a counting argument. We compare two quantities:
\begin{itemize}
    \item[--] The number of possible target labelings $(S_1, \dots, S_n)$, which is $\binom{d}{k}^n$ since each $S_i$ is drawn uniformly from the size-$k$ subsets of~$[d]$.
    \item[--] The number of realizable labelings, that is, the number of distinct tuples $(\mathrm{top}_k(W e_1), \dots, \mathrm{top}_k(W e_n))$ that can be produced as $W$ ranges over $\mathbb{R}^{d \times h}$.
\end{itemize}
The map $W \mapsto \mathrm{top}_k(W e_i)$ depends on $W$ only through the signs of a few linear functions of $W$, so a classical sign-pattern bound, due to \citet{warren_lower_1968}, controls the number of realizable labelings by the number of free parameters in~$W$.
Once $hd$ is too small for the realizable labelings to cover all the possible target labelings, with high probability the random targets fall outside the realizable set and no $W$ can fit them.

\paragraph{Hyperplane arrangements and regions.}
Given $m$ affine functions $f_1, \dots, f_m$ on $\mathbb{R}^P$, the corresponding hyperplane arrangement $\mathcal A$ is the union of their zero sets, and the connected components of the complement $\mathbb{R}^P \setminus \mathcal A$ are called the open regions of $\mathcal A$.
In $\mathbb{R}^2$, for instance, three lines in general position partition the plane into seven such regions.
Each region corresponds to one of the sign patterns $(\mathrm{sign}\,f_1(x), \dots, \mathrm{sign}\,f_m(x)) \in \{+,-\}^m$ that can be realized as $x$ varies in $\mathbb{R}^P$: two points lie in the same region if and only if they yield the same sign pattern.
The total number of sign patterns is at most $2^m$, but for $m$ much larger than~$P$ the number of realizable patterns is dramatically smaller because the ambient space $\mathbb{R}^P$ is too low-dimensional to interleave sign changes in too many ways.
The intuition is already visible in $\mathbb{R}^1$: $m$ points cut the line into only $m + 1$ intervals, that is, $m + 1$ realizable sign patterns out of the $2^m$ a priori options, because once a real number lies to the left of all $m$ points its sign pattern is fully determined and moving rightward can only flip one sign at a time.
Warren's theorem makes this precise in arbitrary dimensions.

\begin{lemma}[Region count, linear case of \citet{warren_lower_1968}]
\label{lem:warren}
Let $f_1, \dots, f_m$ be affine functions on $\mathbb{R}^P$, with $m \geq P$.
The number of open regions of the corresponding hyperplane arrangement is at most $(2\mathrm{e}m/P)^P$, where $\mathrm{e}$ is Euler's number.
In particular, the number of distinct sign patterns realizable by $(f_1(x), \dots, f_m(x))$ as $x$ varies in $\mathbb{R}^P$ is also at most $(2\mathrm{e}m/P)^P$.
\end{lemma}

This is the linear specialization of Warren's bound on sign patterns of polynomials; it is also a consequence of \citet{cover_geometrical_1965}'s counting formula. 
The bound holds uniformly in the choice of the $f_j$: it does not require the hyperplanes to be in general position.
We can now turn to the proof of our upper bound on capacity.

\begin{proof}[Proof of Theorem~\ref{thm:topk-capacity-upper}]
The proof proceeds in three steps.
\begin{enumerate}
    \item We first express the labeling $(\mathrm{top}_k(We_1), \dots, \mathrm{top}_k(We_n))$ as a sign pattern in $W$.
    \item We then apply Warren's lemma to bound the number of distinct realizable labelings.
    \item Finally, we compare the number of realizable labelings to the number of possible random target labelings.
\end{enumerate}

\noindent\textbf{Step 1: Realized labelings are determined by sign patterns in $W$.}
Writing $w_j \in \mathbb{R}^h$ for the $j$-th row of $W$, the $j$-th output logit for context $i$ is $\ell_{ij} = w_j^\top e_i$.
The set $\mathrm{top}_k(W e_i) \subseteq [d]$ is the set of $j$ for which $\ell_{ij}$ is among the $k$ largest entries of $(\ell_{i1}, \dots, \ell_{id})$, and is in particular determined by the comparisons $\ell_{ij} > \ell_{ij'}$ across the $\binom{d}{2}$ unordered pairs $\{j, j'\}$.
Note that this fixes the full sort of the logits, of which the top-$k$ set is a coarsening: distinct sign patterns of these comparisons may yield the same top-$k$ set, but this overcount can only inflate the count of realizable labelings, so an upper bound on sign patterns gives an upper bound on labelings.
Each such comparison is the sign of the affine function
\begin{equation}
    f_{i,j,j'}(W) := \ell_{ij} - \ell_{ij'} = (w_j - w_{j'})^\top e_i, \qquad i \in [n], \ 1 \leq j < j' \leq d,
\end{equation}
which is linear in $W$, viewed as a vector in $\mathbb{R}^{hd}$, with the embedding $e_i$ playing the role of fixed coefficient. The realized labeling
\begin{equation}
    \mathcal{L}(W) := (\mathrm{top}_k(W e_i))_{i \in [n]}
\end{equation}
is therefore a function of $W$ only through the sign pattern of the $m := \binom{d}{2} n$ affine functions $\{f_{i,j,j'}\}_{i, j<j'}$ on $\mathbb{R}^{hd}$, in the sense that two weight matrices $W, W'$ that yield the same sign pattern also yield the same labeling.
Ties, where some $f_{i,j,j'}(W) = 0$ and the sign pattern is undefined, are harmless: storing a given labeling requires the strict inequalities $\min_{j \in S_i} \ell_{ij} > \max_{j' \notin S_i} \ell_{ij'}$, so the set of $W$ that store it is open, and if nonempty it has positive Lebesgue measure and therefore contains a point off the measure-zero union of the hyperplanes $\{f_{i,j,j'} = 0\}$; it thus suffices to count the labelings realized by tie-free~$W$, which is what Warren's lemma does.

\medskip
\noindent\textbf{Step 2: There are not too many realizable labelings.} By Step~1, the number of distinct labelings realizable by some $W \in \mathbb{R}^{d \times h}$ is at most the number of sign patterns of the $m = \binom{d}{2} n$ affine functions $\{f_{i,j,j'}\}_{i, j<j'}$ on $\mathbb{R}^{hd}$. Applying Lemma~\ref{lem:warren} with $P = hd$ and $m = \binom{d}{2} n$ gives, for some universal $C_0 > 0$,
\begin{equation}
    \label{eqn:counting-realizable}
    |\{\mathcal{L}(W) : W \in \mathbb{R}^{d \times h}\}| \leq \left(\frac{C_0 n d^2}{hd}\right)^{hd} = \left(\frac{C_0 n d}{h}\right)^{hd},
\end{equation}
provided $\binom{d}{2} n \geq hd$, which holds in the regime of interest as $n \geq 2h/(d-1)$. Crucially, the right-hand side does not depend on the specific embeddings $\{e_i\}$: Lemma~\ref{lem:warren} counts regions of any arrangement of $m$ hyperplanes in $\mathbb{R}^{hd}$, regardless of the orientation of those hyperplanes, hence regardless of how the embeddings happen to be drawn.

\medskip
\noindent\textbf{Step 3: Random targets typically miss the realizable set.} Condition on the embeddings $\{e_i\}$ and view the targets $S_1, \dots, S_n$ as drawn independently and uniformly from $\binom{[d]}{k}$. Some weight matrix $W$ stores all $n$ patterns if and only if the random target tuple $(S_1, \dots, S_n)$ lies in the realizable set $\{\mathcal{L}(W) : W \in \mathbb{R}^{d \times h}\} \subseteq \binom{[d]}{k}^n$. Since the targets are uniform on the $\binom{d}{k}^n$ possible labelings,
\begin{align}
    \mathbb{P}\!\left[\exists W : \mathrm{top}_k(W e_i) = S_i \ \forall i \,\Big|\, \{e_i\}\right]
    &\leq \frac{|\{\mathcal{L}(W) : W \in \mathbb{R}^{d \times h}\}|}{\binom{d}{k}^n} \\
    &\leq \binom{d}{k}^{-n} \left(\frac{C_0 n d}{h}\right)^{hd} \quad \text{by Eq.~\ref{eqn:counting-realizable}.}
\end{align}
The right-hand side does not depend on $\{e_i\}$, so the bound also holds unconditionally on the embeddings. Taking logarithms, that right-hand side is at most $\delta$ if and only if $n$ satisfies the implicit condition
\begin{equation}
    \label{eqn:topk-capacity-upper-implicit}
    n \log\binom{d}{k} \geq hd \log\!\left(\frac{C_0 n d}{h}\right) + \log(1/\delta).
\end{equation}
We now show that the explicit hypothesis Eq.~\ref{eqn:condition-topk-capacity-upper} implies the implicit condition Eq.~\ref{eqn:topk-capacity-upper-implicit} for a large enough universal $C$.
The only $n$-dependent term on the RHS of Eq.~\ref{eqn:topk-capacity-upper-implicit} is $hd \log n$, and since $\log n$ grows much more slowly than $n$ it cannot keep up with the linear LHS $n \log\binom{d}{k}$.
We make this quantitative via the tangent-line bound $\log n \leq \alpha n + \log \frac{1}{\alpha}$, valid for any $n \geq 1$ and $\alpha > 0$, which we apply with $\alpha = \frac{\log\binom{d}{k}}{2 hd}$ so that the $\alpha n$ term consumes exactly half of the LHS:
\begin{equation}
    \label{eqn:capacity-upper-tangent}
    hd \log n \leq \frac{n \log\binom{d}{k}}{2} + hd \log \frac{2 hd}{\log\binom{d}{k}} \leq \frac{n \log\binom{d}{k}}{2} + 3 hd \log(hd),
\end{equation}
where the last step uses $\log\binom{d}{k} \geq \log 2$ for $1 \leq k < d$, so that $\log\frac{2hd}{\log\binom{d}{k}} \leq \log(3hd) \leq 3\log(hd)$ whenever $hd \geq 2$.
Plugging into the RHS of Eq.~\ref{eqn:topk-capacity-upper-implicit} and using $\log\!\left(\frac{C_0 d}{h}\right) \leq \log C_0 + \log(hd)$,
\begin{align}
    \label{eqn:capacity-upper-expansion}
    hd \log\!\left(\frac{C_0 n d}{h}\right) + \log(1/\delta)
    &= hd \log\!\left(\frac{C_0 d}{h}\right) + hd \log n + \log(1/\delta) \\
    &\leq \frac{n \log\binom{d}{k}}{2} + C_1 hd \log(hd) + \log(1/\delta)
\end{align}
for a universal $C_1 > 0$ absorbing $\log C_0$ and the constant from Eq.~\ref{eqn:capacity-upper-tangent}.
Rearranging, Eq.~\ref{eqn:topk-capacity-upper-implicit} holds whenever
\begin{equation}
    n \log\binom{d}{k} \geq 2C_1 hd \log(hd) + 2\log(1/\delta),
\end{equation}
which is implied by the explicit hypothesis Eq.~\ref{eqn:condition-topk-capacity-upper} for $C \geq 2 \max(C_1, 1)$.
\end{proof}

\paragraph{Behavior across $k$.}
The Stirling approximation gives $\log\binom{d}{k} \approx d H(\rho)$ with $\rho = k/d$, where $H(\rho) = -\rho \log\rho - (1-\rho)\log(1-\rho)$ is the binary entropy. Dropping the explicit constants and absorbing the $\log(hd)$ term into the $\lesssim$ notation, Eq.~\ref{eqn:condition-topk-capacity-upper} reads $n \gtrsim h/H(\rho)$ up to logarithmic factors.
This matches the lower bound $n \lesssim h/(\rho(1-\rho))$ of Theorem~\ref{thm:topk-capacity} up to logarithmic factors, since $H(\rho)$ and $\rho(1-\rho)$ agree up to a factor of $\log(1/\rho)$ in the sparse regime $\rho \ll 1$ and up to constants near $\rho = 1/2$.

\subsection{Validation of theory}
\label{app:capacity-validation}

Theorem~\ref{thm:topk-capacity} only provides a lower bound on the capacity of the linear-softmax model: the Hebbian construction is one specific weight configuration, and a fully optimized model could in principle achieve a higher capacity and a qualitatively different scaling with~$k$. We empirically verify that, while the trained model achieves a higher capacity than the Hebbian construction, its capacity differs from the one of the construction by a roughly constant factor (between $7$ and $15$ over the range of $k$ we test), and thus scales similarly with $k$; see Figure~\ref{fig:validation-capacity}.
Figure~\ref{fig:validation-capacity-params} additionally verifies that the capacity of the trained head grows proportionally to the number of parameters $hd$, with the $1/k$ dependence predicted by the theory.

\begin{figure}[h]
    \centering
    \includegraphics{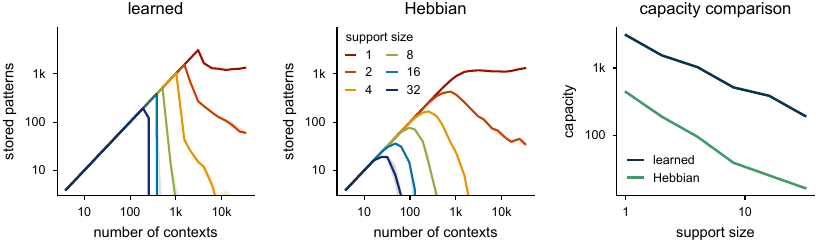}
    \caption{\textbf{The Hebbian model captures the same dependency on ambiguity as learned linear heads.} The left and middle panels show the number of contexts whose top-$k$ predictions exactly match the ground-truth support, as a function of the dataset size $n$, for a linear head trained with Adam (left) and the closed-form Hebbian construction of Section~\ref{subsec:capacity} (middle). The right panel shows the corresponding capacity $n^\ast$ as a function of the support size $k$, where $n^\ast$ is the largest $n$ for which the median top-$k$ accuracy stays above $0.99$. Both models exhibit similar dependencies on ambiguity and only approximately differ by a constant factor.}
    \label{fig:validation-capacity}
\end{figure}

\begin{figure}[h]
    \centering
    \includegraphics{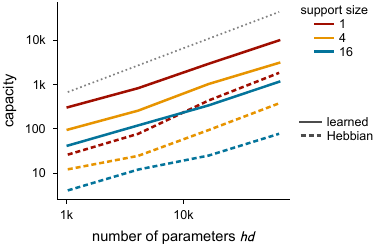}
    \caption{\textbf{The capacity of a trained linear head grows proportionally to the number of parameters $hd$, with the $1/k$ offset predicted by Theorem~\ref{thm:topk-capacity}.} Same capacity definition and learning-rate selection as Figure~\ref{fig:validation-capacity}, up to the reductions described in the experimental setup below, with the capacity $n^\ast$ reported as a function of the number of parameters $hd$ for $k \in \{1, 4, 16\}$, sweeping $h = d \in \{32, 64, 128, 256\}$. The gray dotted line indicates a slope of $1$, that is, a capacity directly proportional to $hd$.}
    \label{fig:validation-capacity-params}
\end{figure}

\paragraph{Experimental setup.} We train a linear head with $h = d = 128$ using Adam on the exact KL divergence to uniform size-$k$ target distributions; the swept configurations are summarized in Table~\ref{tab:capacity-validation}. Embeddings are sampled iid from the unit sphere in $h$ dimensions. The reported metric is the top-$k$ accuracy, that is, the fraction of patterns whose top-$k$ predictions match the support of the target data distribution. For each $(k, n)$ we first select the learning rate that maximizes the mean top-$k$ accuracy across seeds, breaking ties by the lowest mean KL, which picks the best-converged run among the top-$k$-optimal ones; we then define the capacity $n^\ast$ as the largest $n$ for which the median top-$k$ accuracy across seeds stays above $0.99$, interpolating linearly in $\log n$ between the two swept values that bracket the crossing. Reading $n^\ast$ off the mean instead of the median across seeds changes it by less than $0.5\%$ at every~$k$, and leaves the log-log slope in $k$ unchanged. The Hebbian construction does not require training; we evaluate it on the same metric. As $k$ ranges from $1$ to $32$, the learned capacity exceeds the Hebbian one by a factor between $7$ and $15$, and the log-log slopes of $n^\ast$ against $k$ are $-0.77$ (learned) and $-0.97$ (Hebbian).

Figure~\ref{fig:validation-capacity-params} uses the same capacity definition, the same learning-rate selection rule and the same target distributions, with three reductions that make the parameter sweep tractable locally: full-batch Adam rather than minibatches of $2{,}048$ contexts (the loss is the exact distributional KL in both cases), $15{,}000$ steps, and the reduced learning-rate grid $\{3{\cdot}10^{-3},\, 10^{-2},\, 3{\cdot}10^{-2},\, 10^{-1},\, 3{\cdot}10^{-1}\}$, which brackets the optima for $k \in \{4, 16\}$. This protocol reproduces the sweep's capacities at $h = d = 128$ within $12\%$ ($n^\ast = 1{,}030$ against $1{,}028$ for $k = 4$, and $340$ against $385$ for $k = 16$).

\begin{table}[h]
    \centering
    \caption{Sweep configuration for the capacity validation experiment of Figure~\ref{fig:validation-capacity}. $\{ \cdots \}$ indicate the values that we sweep over.}
    \label{tab:capacity-validation}
    \setlength{\tabcolsep}{8pt}
    \begin{tabular}{ll}
        \toprule
        \textbf{Architecture} & linear head, $h = d = 128$ ($hd = 16{,}384$ parameters) \\
        \textbf{Loss} & exact KL divergence to the target distributions \\
        \textbf{Optimizer} & Adam, batch size $2{,}048$, constant learning rate, $150{,}000$ steps \\
        \textbf{Support size} $k$ & $\{1,\, 2,\, 4,\, 8,\, 16,\, 32\}$ \\
        \textbf{Number of contexts} $n$ & $21$ values, log-spaced between $32$ and $32{,}768$ \\
        \textbf{Learning rate} & $\{10^{-3},\, 3{\cdot}10^{-3},\, 10^{-2},\, 3{\cdot}10^{-2},\, 10^{-1},\, 3{\cdot}10^{-1},\, 1\}$ \\
        \textbf{Seeds} & $5$ \\
        \bottomrule
    \end{tabular}
\end{table}

\clearpage
\section{Theoretical results regarding embedding dimension}
\label{app:embedding}

In this appendix, we study the minimum embedding dimension $h$ needed for the linear-softmax model to exactly represent all target distributions, for a fixed support size $k$.
In contrast with the rest of our theoretical analysis, we first assume that the neural network backbone is powerful enough to produce any desired embedding (Appendices~\ref{subsec:polytope-primer} and~\ref{subsec:embedding-proof}).
This means that we can treat both the weights $W$ and the embeddings $\{e_i\}$ as free parameters that we can pick as works best.
Since the ability of the backbone to produce any embedding is related to its capacity, we are implicitly assuming that the backbone has infinite capacity.
Appendix~\ref{subsec:embedding-robust} will consider the robustness of that result when we relax this assumption.

This analysis studies how much $h$ affects what the linear-softmax model can represent.
Formally, we say that it can \emph{represent} a family of target distributions $\{p_i\}$ if there exist weights $W$ and embeddings $\{e_i\}$, both independent of the inverse temperature $\beta > 0$, such that the predictions $\hat p_i = \mathrm{softmax}(\beta W e_i)$ satisfy $\hat p_i \to p_i$ as $\beta \to \infty$ for every $i$.
This limit is needed because the softmax only puts exactly zero mass on off-support tokens when the logits diverge; the temperature $\beta$ is the only quantity we let vary along it, so that representing a family means exhibiting a single configuration of weights and embeddings that becomes exact when sharpened.\footnote{Letting the whole configuration vary with~$\beta$ would not change the threshold of Theorem~\ref{thm:embedding}. Indeed, $\hat p_i \to p_i$ forces the logit-level quantities of Lemma~\ref{lem:KL-control} to satisfy $\nu(S_i, \ell_i) \to 0$ and $\mu(S_i, \ell_i) \to \infty$, since logit differences are log-ratios of predicted probabilities. Fixing any point of such a sequence at which $\mu > 2\nu$, every $A \subseteq [d]$ with $|A| \leq k$ is then strictly separated from its complement by some vector: directly when $|A| = k$, and otherwise by $e_S + e_{S'}$ for two size-$k$ supports with $S \cap S' = A$, whose on-$A$ values exceed all its other values by at least $\mu - 2\nu > 0$. Radon's theorem applied to $h+2$ of the rows of that fixed~$W$, whose smaller side has at most $\lfloor (h+2)/2 \rfloor \leq k$ elements when $h \leq 2k-1$, then contradicts this separation exactly as in the proof of Proposition~\ref{thm:ubt}.}
We use throughout this appendix the same target distributions as in the rest of our theoretical analysis, namely uniform distributions on random size-$k$ supports $S_i \subseteq [d]$ (Eq.~\ref{eqn:target-dist}).

\paragraph{Lower bound on the embedding dimension.} Our first main result characterizes the minimum embedding dimension needed for such a representation to exist.
\begin{theorem}[Minimum embedding dimension]
    \label{thm:embedding}
    Let $d \geq 2k+2$ and $d \geq h+1$.
    The linear-softmax model can represent every family of uniform distributions on size-$k$ supports of $[d]$ if and only if $h \geq 2k$.
\end{theorem}

From the matrix factorization argument of \citet{yang_breaking_2018}, we know that, in general, the embedding dimension needed to represent all possible probability distributions should be at least the vocabulary size $d$; hence the $d \geq h +1$ assumption of the Theorem is not restrictive.
This is what the authors called the softmax bottleneck.
However, this is a general result that is agnostic to the type of distributions being modeled.
Theorem~\ref{thm:embedding} refines this answer for the specific type of distributions we consider, parametrized by support size~$k$.
Closer to our setting, \citet{chang_softmax_2022} show through a geometric argument that a single hidden state cannot place high probability on several words whose embeddings surround interfering word embeddings, identifying such multi-modal distributions as a failure case of the softmax layer; our results sharpen this observation into a precise linear scaling with ambiguity.
We prove Theorem~\ref{thm:embedding} by turning this statistical question into a geometric one: in Section~\ref{subsec:polytope-primer} we review the classical results from the theory of convex polytopes that we later invoke, and in Section~\ref{subsec:embedding-proof} we relate the problem of the model representing these distributions to the problem of neighborliness in convex polytopes to prove Theorem~\ref{thm:embedding}.

\textbf{Extension to more realistic noisy embeddings.} An important assumption behind Theorem~\ref{thm:embedding} is that each context embedding $e_i$ can be tuned arbitrarily finely to the support it encodes.
In practice, the embeddings produced by the backbone are constrained, and small deviations should not destroy the outputs.
Our second main result, Theorem~\ref{thm:robust-embedding} below, shows that such robustness is recoverable at a modest cost: with random weights and an embedding dimension larger than~$2k$ by a factor of order $\log(d/k)$, the representation survives bounded adversarial perturbations of the embedding at the price of a small KL slack.
Appendix~\ref{subsec:embedding-robust} formalizes and proves this result; the informal statement below gives the high-level takeaway.

\begin{theorem}[Robust representation via random weights, informal version of Theorem~\ref{thm:robust-embedding}]
    \label{thm:robust-embedding-informal}
    With an embedding dimension $h = \Omega\bigl(k\log(d/k)\bigr)$ and randomly-drawn output weights, the linear-softmax model admits shared-per-support embeddings such that, for every size-$k$ support $S$ and every bounded adversarial context-dependent perturbation of the corresponding embedding, the predicted distribution is within~$\varepsilon$ of the target uniform distribution on~$S$ in terms of KL divergence, for a small enough perturbation radius.
\end{theorem}
\setcounter{robustinformalnum}{\value{theorem}}

\subsection{A primer on convex polytopes}
\label{subsec:polytope-primer}

Our analysis of the embedding-dimension root of the curse of ambiguity in the ideal case relies on classical results from the theory of convex polytopes, and in particular on the notion of neighborliness.
This subsection provides a self-contained primer on the relevant concepts and states the seminal results that we later invoke, namely the classical upper bound on neighborliness and the cyclic polytope construction of \citet{gale_neighborly_1963}.

\paragraph{Convex polytopes and their faces.} A \emph{convex polytope} $P \subset \mathbb{R}^h$ is the convex hull of a finite set of vertices $v_1, \dots, v_d \in \mathbb{R}^h$.
A subset $F \subseteq \{v_1, \dots, v_d\}$ is a \emph{face} of $P$ if there exist a vector $e \in \mathbb{R}^h$ and a scalar $a \in \mathbb{R}$ such that
\begin{equation}
    \begin{split}
        \label{eq:face-def}
        e^\top v_i &= a \ \text{ for } v_i \in F,\\
        e^\top v_j &< a \ \text{ for } v_j \notin F.
\end{split}
\end{equation}
Geometrically, $F$ is the set of vertices that are maximally and equally aligned with $e$, and the affine hyperplane $\{x : e^\top x = a\}$ is a supporting hyperplane of $P$ touching $P$ exactly at $F$.
The vector $e$ is typically called the witness vector of the face.

\paragraph{Neighborliness.} A polytope $P$ is called \emph{$m$-neighborly} if every subset of at most $m$ of its vertices is a face.
Being $2$-neighborly means that every pair of vertices is connected by an edge, which is the case for the tetrahedron but not for the cube, see Figure~\ref{fig:neighborliness}.
Being $3$-neighborly additionally means that every triple of vertices bounds a triangular face, and so on.
Neighborliness is thus a measure of how many vertices belong to the faces of the polytope.

\begin{figure}[h]
\centering
    \includegraphics{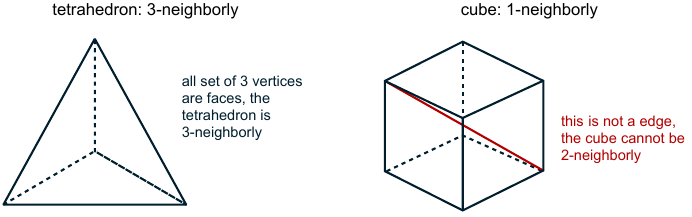}
\caption{The tetrahedron is $3$-neighborly: every subset of its four
vertices is a face.
The cube is only $1$-neighborly: antipodal vertices
are not connected by an edge, so not every pair of vertices forms a
face.}
\label{fig:neighborliness}
\end{figure}

\paragraph{How neighborly can a polytope be?} A natural question to ask is how neighborly we can expect a polytope, and in particular how this depends on the number of vertices.
The answer is given by the following classical threshold, a consequence of Radon's theorem.

\begin{proposition}[Neighborliness threshold, see e.g.\ \citet{grunbaum_convex_2003}]
\label{thm:ubt}
Any convex polytope in $\mathbb{R}^h$ with strictly more than $h+1$ vertices is at most $\lfloor h/2\rfloor$-neighborly.
\end{proposition}

In $\mathbb{R}^3$, this means that any polytope with at least $5$ vertices is at most $1$-neighborly, that is, no polytope with five or more vertices can have every pair of vertices joined by an edge.
The result does not apply to the tetrahedron example of Figure~\ref{fig:neighborliness}, which is $3$-neighborly, as it has only $4$ vertices.
The $\lfloor h/2\rfloor$ threshold comes from Radon's theorem: any set of $h+2$ points in $\mathbb{R}^h$ can be partitioned into two disjoint subsets whose convex hulls intersect.
Once there are enough points, some of them are thus forced to lie inside the convex combinations of the others.
The smaller side~$A$ of such a partition has at most $\lfloor (h+2)/2 \rfloor$ elements and cannot be a face: a supporting hyperplane exact on~$A$ would have to contain a point of $\mathrm{conv}(A)$ that is simultaneously a convex combination of vertices lying strictly below it.

\paragraph{Cyclic polytopes and the polynomial construction.} This threshold is tight: for every $h$ and $d \ge h+1$, there exist polytopes achieving $\lfloor h/2\rfloor$-neighborliness.
The canonical construction is the \emph{cyclic polytope} $C(d, h)$, introduced by \citet{gale_neighborly_1963}.
It is obtained by placing $d$ vertices along the \emph{moment curve}
\begin{equation}
    \phi(t) = (t, t^2, \dots, t^h),
\end{equation}
namely $v_i=\phi(t_i)$ for $t_1 < \cdots < t_d$.

\begin{proposition}[Neighborliness of the cyclic polytope, \citet{gale_neighborly_1963}]
\label{prop:cyclic}
For every $h$ and $d \geq h+1$, the cyclic polytope $C(d, h)$ is $\lfloor h/2\rfloor$-neighborly.
\end{proposition}

The power of this construction is that it reduces the problem of neighborliness to a polynomial interpretation.
The dot product between a vector $e \in \mathbb{R}^h$ and the vertex $v_i = \phi(t_i)$ gives
\begin{equation}
    e^\top v_i = e_1 t_i + e_2 t_i^2 + \dots + e_h t_i^h =: q(t_i),
\end{equation}
where $q$ is a polynomial of degree at most $h$.
The face condition of Eq.~\ref{eq:face-def} then translates to $q$ taking specific values on the different vertices.
Indeed, a subset $F$ is a face of $C(d, h)$ if and only if there exists a polynomial $q$ of degree at most $h$ that equals a common value $a$ on $\{t_i : i \in F\}$ and is strictly less than $a$ on $\{t_i : i \notin F\}$.
Equivalently, denoting $r(t) := a - q(t)$, $F$ is a face if and only if there is a polynomial $r$ of degree at most $h$ that vanishes on $\{t_y : y \in F\}$ and is strictly positive elsewhere.
For any subset $F$ of size $k \le h/2$, such a polynomial can easily be constructed explicitly:
\begin{equation}
    \label{eq:r-construction}
    r(t) \;=\; \prod_{y \in F} (t - t_y)^2.
\end{equation}
This polynomial has degree $2k \le h$, vanishes on the points of $F$, and is strictly positive everywhere else, which establishes Proposition~\ref{prop:cyclic} in the regime $k \leq h/2$; we refer to~\citet{gale_neighborly_1963} for the general statement.

\subsection{Reduction to neighborliness and proof of Theorem~\ref{thm:embedding}}
\label{subsec:embedding-proof}

The key insight behind our analysis is to relate the problem of the linear-softmax model representing the target distributions to the problem of neighborliness in convex polytopes.
We first characterize the structure that $W$ and $\{e_i\}$ must have to achieve this, then show that this structure is exactly a $k$-neighborliness condition on the polytope defined by the rows of $W$, and finally combine this reduction with the results of the previous subsection to prove Theorem~\ref{thm:embedding}.

\paragraph{Required structure of the logits.} Denote by $w_j \in \mathbb{R}^h$ the $j$-th row of the weight matrix, so that the $j$-th logit for context $i$ reads $\ell_{ij} = \beta\, w_j^\top e_i$ with $\beta$ the inverse temperature that goes to infinity.
Since the configuration itself does not depend on~$\beta$, sharpening the softmax simply concentrates the prediction on the largest logits: for any fixed vector $v \in \mathbb{R}^d$,
\begin{equation}
    \label{eq:softmax-argmax}
    \mathrm{softmax}(\beta v)_j \;=\; \frac{e^{\beta v_j}}{\sum_{l \in [d]} e^{\beta v_l}} \;\xrightarrow[\beta \to \infty]{}\;
    \begin{cases}
        1/|\arg\max v| & \text{if } j \in \arg\max v,\\
        0 & \text{otherwise,}
    \end{cases}
\end{equation}
that is, $\mathrm{softmax}(\beta v)$ converges to the uniform distribution on the set of maximizers of~$v$.
Applied to $v = W e_i$, whose target is the uniform distribution on~$S_i$, this shows that the model represents the family if and only if, for every context~$i$,
\begin{equation}
    \label{eq:argmax-condition}
    S_i \;=\; \arg\max_{j \in [d]}\; w_j^\top e_i .
\end{equation}
Writing $a_i := \max_j w_j^\top e_i$ for the common maximal value, Eq.~\ref{eq:argmax-condition} amounts to two conditions: $w_j^\top e_i = a_i$ for every $j \in S_i$, which makes all on-support outputs receive the same probability mass in the limit, and $w_j^\top e_i < a_i$ for every $j \notin S_i$, which makes the off-support mass vanish.

\paragraph{Reduction to a $k$-neighborliness problem.} These two conditions are precisely the definition (Eq.~\ref{eq:face-def}) of the vectors $\{w_j\}_{j \in S_i}$ being a face of the polytope with vertices $\{w_j\}_{j \in [d]}$, with $e_i$ acting as the witness vector and $a_i$ as its value.
That is, the output weights define the polytope, and each embedding selects one of its faces.
Since we want these conditions satisfied for all supports $S_i$ of size $k$, we want all subsets of size $k$ to be faces of the polytope, that is, the polytope to be $k$-neighborly.
Each row is then automatically a vertex of that polytope, which therefore does have $d$ of them, as required to later invoke Proposition~\ref{thm:ubt}.
Indeed, a row that is not a vertex is a convex combination of the others, $w_j = \sum_{l \neq j}\alpha_l w_l$ with $\alpha$ in the simplex.
Pick $l_0 \neq j$ with $\alpha_{l_0} > 0$ and a support $S_i$ that contains~$j$ but not~$l_0$, which exists as $d \geq k+1$.
By Eq.~\ref{eq:argmax-condition}, $w_l^\top e_i \leq a_i$ for every~$l$, with strict inequality at $l_0$, so that $w_j^\top e_i = \sum_{l \neq j}\alpha_l\, w_l^\top e_i < a_i$, contradicting $w_j^\top e_i = a_i$.
In particular the rows are pairwise distinct.
We formalize this reduction in the following lemma.

\begin{lemma}[Reduction to neighborliness]
\label{lem:reduction}
The linear-softmax model can represent every family of uniform distributions on size-$k$ supports of $[d]$ if and only if there exists a convex polytope in $\mathbb{R}^h$ with $d$ vertices that is $k$-neighborly.
\end{lemma}

\begin{proof}
($\Rightarrow$) Suppose the model can represent every such family.
Instantiate this hypothesis at the maximal family, the one whose contexts enumerate all $\binom{d}{k}$ size-$k$ supports of $[d]$; this is the step that forces a single polytope to serve every support, as different families may otherwise be represented by different weights.
Let $W$ and $\{e_S\}$ be parameters realizing that representation.
By Eq.~\ref{eq:argmax-condition}, for every size-$k$ subset $S \subseteq [d]$ the vectors $\{w_j\}_{j \in S}$ form a face of $\mathrm{conv}(\{w_j\}_{j \in [d]})$ with witness $e_S$ and value $a_S$, and each of the $d$ rows is a vertex as observed above.
The polytope $\mathrm{conv}(\{w_j\}_{j \in [d]}) \subset \mathbb{R}^h$ therefore has $d$ vertices and every size-$k$ subset of them as a face, that is, it is $k$-neighborly.

($\Leftarrow$) Conversely, let $P \subset \mathbb{R}^h$ be a $k$-neighborly polytope with vertices $\{w_j\}_{j \in [d]}$.
For every size-$k$ subset $S$, pick $(e_S, a_S) \in \mathbb{R}^h \times \mathbb{R}$ witnessing that $\{w_j\}_{j \in S}$ is a face, that is, satisfying Eq.~\ref{eq:face-def}.
Take as weight matrix the $W$ whose rows are the $\{w_j\}$, associate to each context of the family the embedding $e_{S}$ of its support, and let $\beta \to \infty$.
The logits $\beta\, w_j^\top e_S$ take the common value $\beta a_S$ on $S$ and strictly smaller values elsewhere, so the softmax output converges to the uniform distribution on $S$.
The model therefore represents the family.
\end{proof}

We can now prove our main theorem by combining Lemma~\ref{lem:reduction} with the two polytope results of Section~\ref{subsec:polytope-primer}.

\begin{proof}[Proof of Theorem~\ref{thm:embedding}]
($\Leftarrow$) Suppose $h \geq 2k$.
Since $d \geq h+1$, the cyclic polytope $C(d, h)$ is well defined, and Proposition~\ref{prop:cyclic} gives that it is $\lfloor h/2\rfloor$-neighborly, and therefore $k$-neighborly since $k \leq h/2$.
Lemma~\ref{lem:reduction} then implies that the model represents every family of uniform distributions on size-$k$ supports of $[d]$.

($\Rightarrow$) Suppose $h < 2k$, that is, $k > \lfloor h/2 \rfloor$.
Since $d \geq 2k+2 > h+1$, Proposition~\ref{thm:ubt} tells us that no polytope in $\mathbb{R}^h$ with $d$ vertices is $k$-neighborly.
By Lemma~\ref{lem:reduction}, the model cannot represent every family of uniform distributions on size-$k$ supports of $[d]$.
\end{proof}

\subsection{A robust version of the representation result}
\label{subsec:embedding-robust}

One of the main assumptions behind Theorem~\ref{thm:embedding} is that there is no constraint on the choice of context embeddings: we implicitly assume that the backbone can produce any desired embedding, so that each context can be tuned arbitrarily finely to the support it encodes.
In practice this does not hold, and small deviations of the backbone output from the ideal embedding should not destroy the ability of the model to reconstruct the target output distribution.

In this subsection, we complement Theorem~\ref{thm:embedding} by considering adversarial perturbations, of bounded norm~$\tau$, of a shared embedding per support.
Since the representation can no longer be exact as the logits on support will generally differ under any nonzero perturbation, we also relax the quality criterion and allow for a small KL slack~$\varepsilon$.
The embeddings are taken of unit norm and the rows of the output weight matrix of norm at most one, which makes~$\tau$ a scale-free measure of the precision required of the backbone.

The construction behind Theorem~\ref{thm:embedding} does not survive this relaxation: it turns every support into a face of a single polytope only by making those faces nearly degenerate, as the witnessing polynomial of Eq.~\ref{eq:r-construction} has a double root at each on-support vertex and therefore separates a support from its complement by a margin that shrinks very quickly as the support size grows.
In the exact setting this costs nothing, since an arbitrarily large inverse temperature compensates for any positive margin, but under perturbations the inverse temperature can no longer be taken arbitrarily large: past a certain point it amplifies the perturbation enough to spread the on-support logits apart.
Random output weights instead give margins of order $1/\sqrt{k}$ uniformly over all supports; we use them to build a robust representation and quantify the embedding dimension required to do so.

\paragraph{Proof strategy.} The argument proceeds in four steps.
\begin{enumerate}
    \item We introduce the notion of $(\varepsilon, \tau)$-robust representation (Definition~\ref{def:robust-repr}), which asks a single embedding per support to stay within~$\varepsilon$ of the target distribution in KL divergence, uniformly over perturbations of norm at most~$\tau$.
    \item We relate the KL divergence, the quantity we ultimately care about, to two geometric properties of the polytope formed by the rows of~$W$: the margin separating the on-support rows from the off-support ones, and the variation of the logits within the support (Eq.~\ref{eqn:kl-margin}). Those are the same quantities we used in the capacity argument of Section~\ref{app:capacity} and this will enable us to reuse some of the results from that section.
    If the margin is large enough and the on-support variation small enough, then, after adjusting the temperature~$\beta$, the predicted distribution looks like the target one.
    \item We exhibit weights and embeddings meeting these geometric conditions: we draw the output weights at random and define the embedding of each support as the solution of a minimax program tailored to that draw (Eq.~\ref{eqn:minimax-witness}).
    Lemma~\ref{lem:nsp-margin} shows, through duality, that the null-space property of compressed sensing implies zero on-support variation and a margin of order $1/\sqrt{k}$, uniformly over all $\binom{d}{k}$ supports, and Proposition~\ref{prop:random-nsp} recalls that random weights satisfy this property with high probability as soon as $h \gtrsim k \log(d/k)$.
    \item We tie everything back together in the proof of Theorem~\ref{thm:robust-embedding}, where the temperature~$\beta$ is traded off against the perturbation radius~$\tau$.
\end{enumerate}

\textbf{Robust representation.} Let us first formally introduce the notion of robust representation, which is one generalization of the notion of representation that we have been dealing with so far.
\begin{definition}[$(\varepsilon, \tau)$-robust representation]
    \label{def:robust-repr}
    A family of target distributions $\{p_S\}$, with $S$ ranging over size-$k$ supports of $[d]$, admits an \emph{$(\varepsilon, \tau)$-robust representation} in $\mathbb{R}^h$ if there exist weights $W \in \mathbb{R}^{d \times h}$ with $\lVert w_j\rVert \leq 1$ for every $j$, an inverse temperature $\beta > 0$, and unit-norm centers $\{e_S \in \mathbb{R}^h\}$, $\lVert e_S \rVert = 1$, such that, for every size-$k$ subset $S$ and every $\xi \in \mathbb{R}^h$ with $\lVert \xi\rVert \leq \tau$,
    \begin{equation}
        D_{\mathrm{KL}}\!\left(p_S \,\middle\|\, \mathrm{softmax}\bigl(\beta\, W(e_S + \xi)\bigr)\right) \leq \varepsilon.
    \end{equation}
\end{definition}

In contrast with the limiting setup of Theorem~\ref{thm:embedding}, the inverse temperature~$\beta$ needs to be finite: allowing $\beta \to \infty$ together with a nonzero perturbation of the embedding would blow up the KL.
Both normalizations, $\lVert w_j\rVert \leq 1$ and $\lVert e_S\rVert = 1$, are needed for the definition to have content.
Indeed, the reparametrization $(\beta, e_S, \xi) \mapsto (\beta/\lambda,\, \lambda e_S,\, \lambda \xi)$ leaves every logit unchanged, so without a constraint on the centers one could inflate them and make any fixed radius~$\tau$ negligible in comparison; every exact representation in the sense of Theorem~\ref{thm:embedding} would then be $(\varepsilon, \tau)$-robust for every~$\tau$, and the notion would be vacuous.
Fixing $\lVert e_S\rVert = 1$ quotients out this scale ambiguity and makes $\tau$ a relative precision: it is the fraction of the embedding norm that the backbone is allowed to get wrong.
The normalization $\lVert w_j\rVert \leq 1$ plays the analogous role on the weight side, and gives that a perturbation of magnitude~$\tau$ in embedding space translates to a logit perturbation of magnitude at most~$\beta \tau$ by Cauchy--Schwarz.

\paragraph{From margin and variation to KL.} Our first step is to reduce the robust-representation condition to two geometric quantities on the weight polytope.
For a support $S$ of size $k$, a weight matrix~$W$, and an embedding~$e$, define the \emph{geometric margin} and \emph{geometric on-support variation}
\begin{equation}
    \label{eqn:geom-mu-nu}
    \mu(S, e) := \min_{j \in S,\, j' \notin S} \bigl(w_j^\top e - w_{j'}^\top e\bigr),
    \qquad
    \nu(S, e) := \max_{j, j' \in S} \bigl|w_j^\top e - w_{j'}^\top e\bigr|.
\end{equation}
An exact witness to $S$ being a face corresponds to $\mu(S, e) > 0$ and $\nu(S, e) = 0$.
Note that we overload notation with the logit-level quantities of Lemma~\ref{lem:KL-control}: for $\ell = \beta W e$, one has $\nu(S, \ell) = \beta\,\nu(S, e)$ (exactly) and $\mu(S, \ell) \geq \beta\,\mu(S, e)$, the inequality coming from the fact that the geometric margin of Eq.~\ref{eqn:geom-mu-nu} is defined with a minimum over the support whereas the logit-level margin of Eq.~\ref{eqn:mu-def} is defined with the on-support mean, which is at least that minimum.
Substituting into Eq.~\ref{eqn:kl-quadratic-bound} gives, for any $e$ such that $\beta\,\nu(S, e) \leq 1$,
\begin{equation}
    \label{eqn:kl-margin}
    D_{\mathrm{KL}}\!\left(p_S \,\middle\|\, \mathrm{softmax}(\beta W e)\right)
    \;\leq\; \mathrm{e}\,\bigl(\beta\,\nu(S, e)\bigr)^2 \;+\; \frac{d - k}{k}\, e^{-\beta\, \mu(S, e)}.
\end{equation}
Eq.~\ref{eqn:kl-margin} thus reduces the robust-representation problem to a quantitative-neighborliness question: we need a polytope in which every size-$k$ subset of vertices is a face with a margin that is large compared to the perturbation radius.

\paragraph{A random-weight construction with uniformly positive margin.} We draw $g_1, \dots, g_d$ iid from $\mathcal{N}(0, I_h/h)$, collect them as the rows of $G \in \mathbb{R}^{d \times h}$, and set $W := G/\max_j \lVert g_j\rVert$ so that $\lVert w_j \rVert \leq 1$ as Definition~\ref{def:robust-repr} requires.
For each size-$k$ subset $S \subseteq [d]$, we use as center the normalized \emph{regularized minimax witness} $e_S / \lVert e_S\rVert$, where
\begin{equation}
    \label{eqn:minimax-witness}
    e_S \in \arg\min_{e \in \mathbb{R}^h}\;\Bigl\{\max_{j' \notin S}\, \bigl|w_{j'}^\top e\bigr| \;+\; \tfrac{\kappa}{\sqrt{k}}\lVert e\rVert \Bigr\}\;\; \text{s.t.}\;\; w_j^\top e = \tfrac{1}{\sqrt{k}} \text{ for every } j \in S,
\end{equation}
where $\kappa > 0$ is a parameter of our construction, which we will set in Lemma~\ref{lem:nsp-margin} below as a function of the geometry of the draw of~$W$.
Without the penalty, this program simply maximizes the separation between the support and its complement; the penalty is what will additionally give us control on $\lVert e_S\rVert$, which we need because normalizing a center to unit norm divides its margin by $\lVert e_S\rVert$.
The on-support target $\tfrac{1}{\sqrt{k}}$ is the natural level for a unit-norm center, since $e = \sum_{j \in S} w_j / \lVert \sum_{j \in S} w_j\rVert$ already achieves $w_j^\top e \approx \tfrac{1}{\sqrt{k}}$ on the support for near-orthogonal weights.
By construction, all on-support logits are equal to $\tfrac{1}{\sqrt{k}}$, so $\nu(S, e_S) = 0$, and $\mu(S, e_S) \geq \tfrac{1}{\sqrt{k}} - \eta_S$ where
\begin{equation}
    \label{eqn:eta-S}
    \eta_S := \min_{e \in \mathbb{R}^h}\; \Bigl\{\max_{j' \notin S}\,|w_{j'}^\top e| + \tfrac{\kappa}{\sqrt{k}}\lVert e\rVert\Bigr\} \quad \text{s.t.} \quad w_j^\top e = \tfrac{1}{\sqrt{k}} \ \text{ for every } j \in S
\end{equation}
is the optimal value of the program, whose dependence on the fixed~$\kappa$ we leave implicit.
Since the two terms of the objective are nonnegative, a bound on~$\eta_S$ controls the off-support alignment and the norm of the witness simultaneously,
\begin{equation}
    \label{eqn:two-for-one}
    \max_{j' \notin S}\,|w_{j'}^\top e_S| \;\leq\; \eta_S, \qquad \lVert e_S\rVert \;\leq\; \frac{\sqrt{k}}{\kappa}\,\eta_S,
\end{equation}
which is precisely why the two requirements can be handled by a single argument.
A naive pointwise union bound on $\eta_S$ over the $\binom{d}{k}$ supports is wasteful: the $e_S$ corresponding to overlapping supports are tightly coupled through~$W$.
The null-space property exploits this coupling by treating the kernel of~$W^\top$ as a single object.

\begin{definition}[Robust null-space property]
    \label{def:nsp}
    The matrix $W$ satisfies the \emph{$\ell_2$-robust null-space property} at sparsity~$k$ with constants $\gamma \in (0, 1)$ and $\chi > 0$ if, for every $v \in \mathbb{R}^d$ and every size-$k$ or smaller subset $S \subseteq [d]$,
    \begin{equation}
        \label{eqn:nsp}
        \lVert v_S\rVert_2 \;\leq\; \frac{\gamma}{\sqrt{k}}\, \lVert v_{S^c}\rVert_1 \;+\; \chi\,\lVert W^\top v\rVert_2 .
    \end{equation}
\end{definition}

The following deterministic lemma is the bridge from the null-space property, a statement about the kernel of~$W^\top$, to our geometric quantities on the row polytope of~$W$.
It goes through the dual of the minimax-witness program of Eq.~\ref{eqn:eta-S}.
Recall that this program, the \emph{primal}, is a minimization over embeddings $e \in \mathbb{R}^h$ subject to the $k$ on-support constraints, and that its \emph{dual} is a maximization over one multiplier per such constraint, $\lambda \in \mathbb{R}^k$, whose value never exceeds the primal one; \emph{strong duality}, which holds here since the constraints are affine, means that the two optimal values coincide, so that $\eta_S$ can be computed on whichever side is more convenient.
The dual side is the convenient one here: solving the primal directly would require handling the $\binom{d}{k}$ supports separately, whereas the dual expresses~$\eta_S$ in terms of vectors that almost lie in $\ker(W^\top)$, and a single property of that kernel then covers all supports at once.

\begin{lemma}[Null-space property implies uniform margin and bounded witness]
    \label{lem:nsp-margin}
    Assume $W$ satisfies the $\ell_2$-robust null-space property at sparsity~$k$ with constants $\gamma < 1$ and $\chi > 0$, and set $\kappa := \tfrac{1 - \gamma}{2\chi}$ in Eq.~\ref{eqn:minimax-witness}.
    Then, for every size-$k$ subset $S \subseteq [d]$, the regularized minimax witness $e_S$ exists, satisfies $\lVert e_S\rVert \leq \tfrac{(1+\gamma)\chi}{1-\gamma}$, and its normalization $\hat e_S := e_S/\lVert e_S\rVert$ satisfies
    \begin{equation}
        \label{eqn:lemma-margin}
        \nu(S, \hat e_S) = 0, \qquad \mu(S, \hat e_S) \;\geq\; \frac{c_{\gamma, \chi}}{\sqrt{k}} \quad \text{with} \quad c_{\gamma, \chi} := \frac{(1-\gamma)^2}{2\,(1+\gamma)\,\chi} .
    \end{equation}
\end{lemma}

\begin{proof}
    Fix~$S$. The proof proceeds in three steps: we first check that the witness is well defined, then reduce the claim to a bound on the optimal value~$\eta_S$, and finally obtain that bound by dualizing the program.

    \medskip
    \noindent\textbf{Step 1: the witness exists.}
    Let $v \in \mathbb{R}^d$ be supported on~$S$ and such that $W^\top v = \sum_{j \in S} v_j w_j = 0$.
    Applying Eq.~\ref{eqn:nsp} to this~$v$, whose two right-hand-side terms both vanish, gives $\lVert v_S\rVert_2 \leq 0$, hence $v = 0$: no nontrivial linear combination of the on-support rows vanishes, that is, $\{w_j\}_{j \in S}$ are linearly independent.
    The $k$ constraints of Eq.~\ref{eqn:minimax-witness} are therefore feasible.
    The objective is moreover continuous and diverges as $\lVert e \rVert \to \infty$, because of the penalty term, so the infimum is attained on the (closed) feasible set and $e_S$ is well defined.

    \medskip
    \noindent\textbf{Step 2: reduction to a bound on $\eta_S$.}
    The constraints force all on-support inner products to equal $\tfrac{1}{\sqrt{k}}$, so $\nu(S, e_S) = 0$ exactly, and $\mu(S, e_S) \geq \tfrac{1}{\sqrt{k}} - \eta_S$ as noted above.
    Both $\mu$ and $\nu$ are $1$-homogeneous in~$e$, so normalizing does not change the first equality, $\nu(S, \hat e_S) = 0$, and turns the second bound into $\mu(S, \hat e_S) = \mu(S, e_S)/\lVert e_S\rVert$.
    Establishing Eq.~\ref{eqn:lemma-margin} thus only requires a lower bound on $\mu(S, e_S)$ and an upper bound on $\lVert e_S \rVert$, and by Eq.~\ref{eqn:two-for-one} both follow from the single bound $\eta_S \leq \tfrac{1 + \gamma}{2\sqrt{k}}$ that we prove next: it gives $\mu(S, e_S) \geq \tfrac{1-\gamma}{2\sqrt{k}}$ and $\lVert e_S\rVert \leq \tfrac{1+\gamma}{2\kappa} = \tfrac{(1+\gamma)\chi}{1-\gamma}$, whose ratio is $c_{\gamma, \chi}/\sqrt{k}$.

    \medskip
    \noindent\textbf{Step 3: bounding $\eta_S$ by duality.}
    Write the objective of Eq.~\ref{eqn:minimax-witness} as $f(e) := \max_{j' \notin S}|w_{j'}^\top e| + \tfrac{\kappa}{\sqrt{k}}\lVert e\rVert$ and its constraints as $W_S e = \tfrac{1}{\sqrt{k}}\mathbf{1}$, where $W_S$ collects the rows indexed by~$S$.
    Dualizing the equality constraints with multipliers $\lambda \in \mathbb{R}^k$ gives 
    \begin{equation}
        \eta_S = \max_\lambda \left [\frac{1}{\sqrt{k}}\mathbf{1}^\top\lambda + \min_e \left (f(e) - (W_S^\top\lambda)^\top e \right ) \right ],
    \end{equation}
    where strong duality holds because the constraints are affine and the primal is feasible with finite value.
    As $f$ is convex and positively homogeneous, meaning that $f(te) = t f(e)$ for every $t \geq 0$, the inner minimum is~$0$ if $W_S^\top\lambda$ belongs to the subdifferential of~$f$ at the origin, $\partial f(0) := \{u \in \mathbb{R}^h : u^\top e \leq f(e) \text{ for every } e\}$, and $-\infty$ otherwise: the dual is thus a maximization of the linear objective $\tfrac{1}{\sqrt{k}}\mathbf{1}^\top \lambda$ over the multipliers whose image $W_S^\top \lambda$ falls in that set.
    Here $\partial f(0) = \{\sum_{j' \notin S}a_{j'}w_{j'} + z : \lVert a\rVert_1 \leq 1,\, \lVert z\rVert_2 \leq \tfrac{\kappa}{\sqrt{k}}\}$ is the sum of the convex hull of $\{\pm w_{j'}\}_{j' \notin S}$, which comes from the maximum in~$f$, and of the ball of radius $\tfrac{\kappa}{\sqrt{k}}$, which comes from the penalty.
    An optimal~$\lambda$ therefore comes with such an $(a, z)$, and gluing the two together into the single vector $v \in \mathbb{R}^d$ defined by $v_S := \lambda$ and $v_{S^c} := -a$ makes $W^\top v = z$ small,
    \begin{equation}
        \label{eqn:dual-certificate}
        \eta_S \;=\; \tfrac{1}{\sqrt{k}}\,\mathbf{1}^\top v_S, \qquad \lVert v_{S^c}\rVert_1 \leq 1, \qquad \lVert W^\top v\rVert_2 \leq \tfrac{\kappa}{\sqrt{k}} .
    \end{equation}
    In other words, the optimal value is attained by a vector that is nearly, but not exactly, in $\ker(W^\top)$; penalizing $\lVert e\rVert$ in the primal is exactly what relaxes $v \in \ker(W^\top)$ into this near-kernel condition, and this is the reason we need the robust version of the null-space property rather than its classical exact-kernel form.
    Eq.~\ref{eqn:dual-certificate} is now precisely of the shape that Eq.~\ref{eqn:nsp} controls: bounding $\mathbf{1}^\top v_S \leq \sqrt{k}\lVert v_S\rVert_2$ by Cauchy--Schwarz and applying the null-space property to~$v$ yields
    \begin{equation}
        \eta_S \;\leq\; \lVert v_S\rVert_2 \;\leq\; \frac{\gamma}{\sqrt k} + \frac{\chi \kappa}{\sqrt k} \;=\; \frac{1+\gamma}{2\sqrt{k}}
    \end{equation}
    by definition of~$\kappa$, which is the bound used in Step~2 and concludes the proof.
\end{proof}

It remains to check that our random weights satisfy the robust null-space property with a constant $\gamma$ bounded away from one and a constant $\chi$ bounded above.
This is a standard result in compressed sensing, and we obtain it by chaining two textbook results applied to $G^\top \in \mathbb{R}^{h \times d}$, viewed as a measurement matrix taking $h$ linear measurements of vectors in~$\mathbb{R}^d$: subgaussian matrices satisfy the restricted isometry property as soon as the number of measurements exceeds $k\log(\mathrm{e}d/k)$ up to a constant, and a small restricted isometry constant implies the robust null-space property.

\begin{proposition}[Random weights satisfy the null-space property]
    \label{prop:random-nsp}
    There exist universal constants $C > 0$, $\gamma_0 \in (0,1)$ and $\chi_0 > 0$ such that the following holds.
    For any $\delta \in (0, 1)$, if
    \begin{equation}
        \label{eqn:h-nsp-condition}
        h \;\geq\; C\,k\,\log(\mathrm{e}d/k) + C\log(1/\delta),
    \end{equation}
    then the matrix $W$ constructed above satisfies the $\ell_2$-robust null-space property of Definition~\ref{def:nsp} at sparsity~$k$ with constants $\gamma = \gamma_0$ and $\chi = \chi_0$, with probability at least $1 - \delta$.
\end{proposition}

\begin{proof}
    The matrix $G^\top$ is a subgaussian random matrix with $h$ rows and $d$ columns, normalized as in \citet[Definition 9.1]{foucart_mathematical_2013}.
    Under Eq.~\ref{eqn:h-nsp-condition} with $C$ large enough, \citet[Theorem 9.2]{foucart_mathematical_2013} applied at sparsity $2k$ gives that its restricted isometry constant satisfies $\delta_{2k} \leq 3/8 < 4/\sqrt{41}$ with probability at least $1 - \delta/2$.
    On that event, \citet[Theorem 6.13]{foucart_mathematical_2013} yields the $\ell_2$-robust null-space property of order~$k$ for~$G$, with constants $\gamma_0 < 1$ and $\chi_1$ that depend only on the numerical value $3/8$ bounding~$\delta_{2k}$ and are therefore universal.
    Rescaling $G$ into $W = G/\zeta$ with $\zeta := \max_j\lVert g_j\rVert$ leaves $\gamma_0$ and $\ker(W^\top) = \ker(G^\top)$ unchanged and only replaces $\chi_1$ by $\zeta\,\chi_1$ in Eq.~\ref{eqn:nsp}.
    The factor $\zeta$ is random, but Eq.~\ref{eqn:nsp} only becomes weaker as its constant grows, so it suffices to bound~$\zeta$ on a high-probability event.
    Each $\lVert g_j\rVert$ concentrates around~$1$, and, since $k\log(\mathrm{e}d/k) \geq \log(\mathrm{e}d)$ makes Eq.~\ref{eqn:h-nsp-condition} imply $h \geq C\log(d/\delta)$, Gaussian norm concentration and a union bound over the $d$ rows give $\zeta \leq 2$ with probability at least $1 - \delta/2$ for $C$ large enough.
    On the intersection of the two events, of probability at least $1 - \delta$, the property therefore holds with the universal constants $\gamma_0$ and $\chi_0 := 2\chi_1$.
\end{proof}

\paragraph{Putting everything together.} We now have all the ingredients to prove the main result of this subsection: Eq.~\ref{eqn:kl-margin} turns the KL divergence we want to bound into a requirement on the margin and on-support variation of the embeddings, Proposition~\ref{prop:random-nsp} gives that random output weights satisfy the null-space property once $h \gtrsim k\log(d/k)$, and Lemma~\ref{lem:nsp-margin} converts that property into the uniform margin our embeddings need.
The proof below combines the three, the only remaining work being the choice of the temperature~$\beta$, which must be large enough for the off-support mass to be negligible yet small enough for the perturbation not to spread the on-support logits.

\setcounter{tempthmsave}{\value{theorem}}
\setcounter{theorem}{\value{robustinformalnum}}
\addtocounter{theorem}{-1}

\begin{theorem}[Robust representation via random weights]
    \label{thm:robust-embedding}
    There exist universal constants $C, c > 0$ such that the following holds.
    For every $k, d$ with $d \geq 2k+2$, every $\varepsilon, \delta \in (0, 1)$, and every
    \begin{equation}
        \label{eqn:tau-condition}
        \tau \;\leq\; \frac{c\,\sqrt{\varepsilon}}{\sqrt{k}\,\log\!\bigl(d/(k\varepsilon)\bigr)},
    \end{equation}
    the family of uniform distributions on size-$k$ supports of $[d]$ admits an $(\varepsilon, \tau)$-robust representation in $\mathbb{R}^h$ with probability at least $1 - \delta$ over the draw of the random output weights, provided
    \begin{equation}
        \label{eqn:h-robust-condition}
        h \;\geq\; C\,k\,\log(ed/k) + C\log(1/\delta).
    \end{equation}
\end{theorem}

\setcounter{theorem}{\value{tempthmsave}}

\begin{proof}
    Take $W$ and the normalized witnesses $\hat e_S$ of the construction above as the weights and centers of Definition~\ref{def:robust-repr}.
    Under Eq.~\ref{eqn:h-robust-condition} with a large enough constant, Proposition~\ref{prop:random-nsp} guarantees that~$W$ satisfies the $\ell_2$-robust null-space property at sparsity~$k$ with constants $\gamma = \gamma_0$ and $\chi = \chi_0$ with probability at least $1 - \delta$.
    On this event, Eq.~\ref{eqn:lemma-margin} of Lemma~\ref{lem:nsp-margin} yields $\nu(S, \hat e_S) = 0$ and $\mu(S, \hat e_S) \geq c_0/\sqrt{k}$ uniformly over all size-$k$ subsets~$S$, where $c_0 := c_{\gamma_0, \chi_0}$ is universal.

    Fix a size-$k$ subset $S$ and a perturbation $\xi$ with $\lVert \xi\rVert \leq \tau$, and set $e := \hat e_S + \xi$.
    Since $|w_j^\top \xi| \leq \tau$ for every~$j$ by Cauchy--Schwarz, such a perturbation moves each of the two geometric quantities by at most~$2\tau$,
    \begin{equation}
        \mu(S, e) \geq \mu(S, \hat e_S) - 2\tau \geq \frac{c_0}{\sqrt{k}} - 2\tau, \qquad \nu(S, e) \leq \nu(S, \hat e_S) + 2\tau = 2\tau.
    \end{equation}
    Since $d \geq 2k+2$ and $\varepsilon < 1$, we have $\log\bigl(d/(k\varepsilon)\bigr) \geq \log 2$, so Eq.~\ref{eqn:tau-condition} together with $\sqrt\varepsilon \leq 1$ gives $\tau \leq c/(\sqrt{k}\log 2) \leq c_0/(4\sqrt{k})$ as soon as $c \leq c_0\log(2)/4$; we may therefore assume $\tau \leq c_0/(4\sqrt{k})$, so that $\mu(S, e) \geq c_0/(2\sqrt{k})$.
    Choose $\beta = \tfrac{2\sqrt{k}}{c_0}\,\log\!\bigl(2(d-k)/(k\varepsilon)\bigr)$, so that
    \begin{equation}
        \frac{d-k}{k}\,e^{-\beta\mu(S,e)} \;\leq\; \frac{d-k}{k}\,e^{-\beta c_0/(2\sqrt{k})} \;\leq\; \frac{\varepsilon}{2}.
    \end{equation}
    The condition $\beta\,\nu(S, e) \leq 1$ required by Eq.~\ref{eqn:kl-margin} is implied by $\tfrac{4}{c_0}\sqrt{k}\,\tau\log\!\bigl(2d/(k\varepsilon)\bigr) \leq 1$, which holds for $\tau$ satisfying Eq.~\ref{eqn:tau-condition} with $c$ small enough, using $\log\bigl(2d/(k\varepsilon)\bigr) \leq 2\log\bigl(d/(k\varepsilon)\bigr)$ as $d/(k\varepsilon) \geq 2$.
    Applying Eq.~\ref{eqn:kl-margin},
    \begin{equation}
        D_{\mathrm{KL}} \;\leq\; \mathrm{e}\,(2\beta \tau)^2 + \frac{\varepsilon}{2}.
    \end{equation}
    Finally, $\mathrm{e}\,(2\beta \tau)^2 \leq \tfrac{16\mathrm{e}}{c_0^2}\,k\,\tau^2\log^2\!\bigl(2d/(k\varepsilon)\bigr) \leq \varepsilon/2$ by Eq.~\ref{eqn:tau-condition} for $c$ small enough, which yields $D_{\mathrm{KL}} \leq \varepsilon$.
\end{proof}

We make three observations about this result.
\begin{itemize}
    \item[--] \textbf{Exact versus robust regime.} Theorem~\ref{thm:embedding} guarantees exact representation as soon as $h \geq 2k$.
    At the cost of an additional $\log(d/k)$ factor on the embedding dimension, Theorem~\ref{thm:robust-embedding} shows that the representation survives bounded adversarial perturbations of the context embedding and a $D_{\mathrm{KL}}$ slack~$\varepsilon$, whereas the witness behind Theorem~\ref{thm:embedding} tolerates only perturbations as small as the margin it leaves, which shrinks very quickly with the support size.
    Random weights break this pathology, in line with the compressed-sensing intuition that random projections yield quantitatively neighborly polytopes~\citep{candes_robust_2006,donoho_compressed_2006}.
    \item[--] \textbf{Curse of ambiguity in the tolerable perturbation.} The tolerable radius in Eq.~\ref{eqn:tau-condition} shrinks as $1/\sqrt{k}$ and, the centers being of unit norm, it is directly a relative precision.
    This reflects the fact that coordinating $k$ on-support vertices on a face of the polytope becomes geometrically more demanding as~$k$ grows: larger ambiguity costs not only more embedding dimension, but also more precision on the backbone embedding.
    \item[--] \textbf{Connection to sparse recovery.} The null-space property of Definition~\ref{def:nsp} is, in its unrobust form, the standard sufficient and necessary condition for $\ell_1$-minimization to recover every $k$-sparse vector from linear measurements~\citep{candes_robust_2006,donoho_compressed_2006,foucart_mathematical_2013}.
    Lemma~\ref{lem:nsp-margin} repurposes it to control the row polytope of~$W$ instead: a small null-space overlap with any $k$-subset translates into a strict separation between the $k$ on-support rows and the remaining vertices, and the robust term additionally bounds the norm of the witness that realizes this separation.
\end{itemize}

\newpage
\section{Theoretical results regarding optimization}
\label{app:optimization}

\subsection{One-dimensional differential equation for softmax gradient flow}

In the main text, we have mentioned that the gradient flow dynamics of the linear model can be reduced in some regimes (precisely far below the capacity threshold) to a one-dimensional differential equation. We here justify how we can transform gradient flow on a single logistic regression problem over $d$ parameters to a one-dimensional differential equation. The reduction itself is exact; the asymptotic expansions that follow are all taken in the regime $k \ll d$, in which the number of admissible next tokens is small compared to the vocabulary size, as is the case in practice. In that regime we systematically use $d - k \approx d$, and keep only the leading order in $k/d$.

Let $\ell \in \mathbb{R}^d$ be the model's logits and initialize them to $0$ and $\hat p = \mathrm{softmax}(\ell)$ its corresponding probability distribution. For a single context with unit-norm embedding $e$, gradient flow on the weights $W$ induces exactly gradient flow on the logits $\ell = We$, so we work with the logits directly. We take the target distribution $p$ to be the one from Equation~\ref{eqn:target-dist}, that is uniform over $k$ out of the $d$ classes:
\begin{equation}
    p_{y} = \begin{cases} \frac{1}{k} & \text{if } y \in S \\ 0 & \text{otherwise.} \end{cases}
\end{equation}
It is worthwhile to note that in this analysis, we do not consider any context (or equivalently we fix it) and therefore have dropped $x$ subscripts compared to the main text.

The gradient flow dynamics on the logits of a single context on the expected cross entropy loss, that is, on the KL divergence between $p$ and $\hat p$, are
\begin{align}
    \dot{\ell} &= \nabla_\ell \mathbb{E}_y [ \log(\hat p_y) ] \\
    &= -\nabla_\ell D_{\mathrm{KL}}(p \| \hat p)\\
    &= p - \hat p.
\end{align}

Since all on-support logits ($y \in S$) start equal and receive the same gradient, they remain equal for all time. Likewise for the off-support logits. Additionally, the sum of the logits does not affect the probability distribution $\hat p$ so we can reduce the model to one variable, $\Delta \ell$, that is the difference between on- and off-support logits. It follows that
\begin{align}
    \dot{\Delta \ell} &= \left ( p_\mathrm{on} - \hat p_{\mathrm{on}} \right ) - \left ( p_\mathrm{off} - \hat p_{\mathrm{off}} \right )\\
    &= \frac{1}{k} - \frac{1 - (d-k) \hat p_\mathrm{off}}{k} + \hat p_\mathrm{off} \\
    &= \frac{d}{k} \frac{1}{d-k+k\exp(\Delta \ell)}
\end{align}
with $\Delta \ell(0) =0$ and $\hat p_\mathrm{on}$ (resp. $\hat p_\mathrm{off}$) the probability of on- (resp. off-) support outputs.

The differential equation is separable: $\frac{k}{d}(k\exp(\Delta \ell) + d - k)\,\mathrm{d}(\Delta \ell) = \mathrm{d}t$. Integrating both sides yields
\begin{equation}
    \label{eqn:implicit_Delta_ell}
    \frac{k^2}{d}\left(\exp(\Delta \ell) - 1\right) + \frac{k(d-k)}{d}\Delta \ell = t.
\end{equation}

\paragraph{Asymptotics.} Let us now look at the early and late evolution of the KL divergence. First, it can be written as a function of $\Delta \ell$:
\begin{align}
    D_{\mathrm{KL}}(p \| \hat p) &= \sum_{y \in S} \frac{1}{k} \log \left ( \frac{1}{k \hat p_\mathrm{on}} \right )\\
    &= \log \left ( \frac{1}{1 - (d-k)\hat p_\mathrm{off}}\right ) \\
    &= \log \left ( \frac{k\exp(\Delta \ell) + d-k}{k \exp(\Delta \ell) }\right )\\
    &= \log\!\left(1 + \frac{d-k}{k\exp(\Delta \ell)}\right).
\end{align}
Around $t=0$, $\Delta \ell$ remains close to $0$, which means that up to a first order approximation
\begin{align}
    D_\mathrm{KL}(p \| \hat p) &\approx \left . D_\mathrm{KL}(p \| \hat p) \right |_{t=0} + \left . \frac{\mathrm{d} D_\mathrm{KL}(p \| \hat p)}{\mathrm{d}\Delta \ell} \right |_{t=0} \Delta \ell\\
    &= \log\left (\frac{d}{k}\right) - \frac{d-k}{d} \Delta \ell.
\end{align}
In this regime we can linearize the exponential in the implicit equation given by Equation~\ref{eqn:implicit_Delta_ell}, $\exp(\Delta \ell) - 1 \approx \Delta \ell$, whose left-hand side then collapses to $\frac{k^2}{d}\Delta \ell + \frac{k(d-k)}{d}\Delta \ell = k \Delta \ell$, so that
\begin{equation}
    \label{eqn:early_Delta_ell}
    \Delta \ell \approx \frac{d}{k(d-k)}t \approx \frac{t}{k},
\end{equation}
where the first expression is what one obtains by dropping the exponential term altogether instead of linearizing it; the two agree at leading order in $k/d$, which is the accuracy we work at here. Combining Equation~\ref{eqn:early_Delta_ell} with $\frac{d-k}{d} \approx 1$, we get that $D_\mathrm{KL}$ is approximately
\begin{equation}
    D_\mathrm{KL}(p \| \hat p) \approx \log \left ( \frac{d}{k} \right ) - \frac{t}{k}
\end{equation}
at the beginning of the dynamics, showing that learning a distribution with $k$ different possible targets takes $k$ times longer than the deterministic case early on in the dynamics. Now moving to late-stage dynamics, the exponential term $\exp(\Delta \ell)$ dominates as $\Delta \ell \rightarrow +\infty$, so Equation~\ref{eqn:implicit_Delta_ell} gives
\begin{equation}
    \exp(\Delta \ell) \approx \frac{td}{k^2}.
\end{equation}
Plugging this into the KL divergence expression, we get
\begin{align}
    D_\mathrm{KL}(p \| \hat p) &\approx \log \left ( 1 + \frac{k(d-k)}{td} \right )\\
    &\approx \frac{k(d-k)}{td} \approx \frac{k}{t},
\end{align}
where the last step again uses $d - k \approx d$. This highlights that, asymptotically, the convergence rate $1/t$ is the same for all $k$ values, but the loss increases linearly with $k$ when $k \ll d$.

\subsection{Validation of theory}
\label{app:optimization-validation}

\begin{figure}[t]
    \centering
    \includegraphics{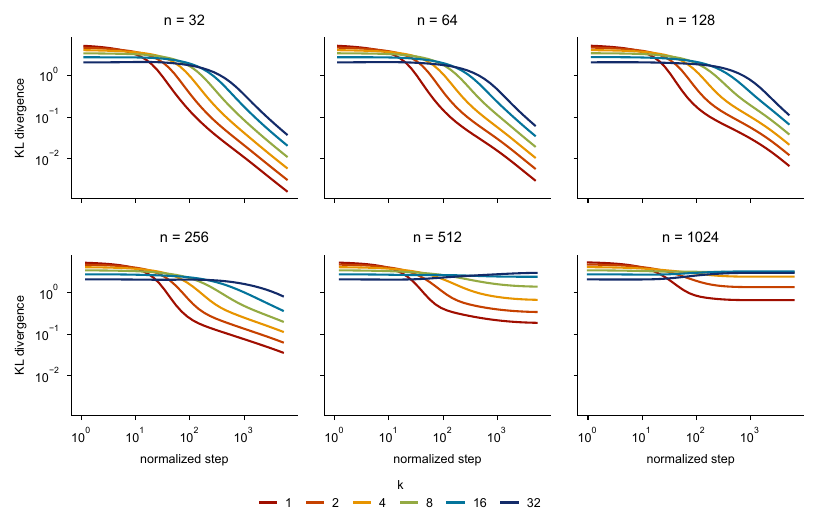}
    \caption{\textbf{Gradient-flow theory qualitatively captures learning below the capacity threshold, even with mixed ambiguity within the data.} Each panel shows the mean KL divergence on the size-$k$ subset of contexts, against normalized step $\mathrm{step} / n$, for one value of the number of contexts $n$. Below the capacity threshold (top row, $n \in \{32, 64, 128\}$), the curves stack in increasing order of $k$ and decay at similar rate, mirroring the theoretical prediction that more ambiguous distributions are harder to learn but share the same asymptotic slope; see Figure~\ref{fig:learning}. As we approach and cross the capacity threshold (bottom row, $n \in \{256, 512, 1024\}$), the dynamics flatten and the simplified theory no longer applies: ambiguous distributions plateau at high loss while less ambiguous ones are still learned, consistent with the capacity-based picture of Appendix~\ref{app:capacity}.}
    \label{fig:validation-learning-gf}
\end{figure}

\begin{figure}[t]
    \centering
    \includegraphics{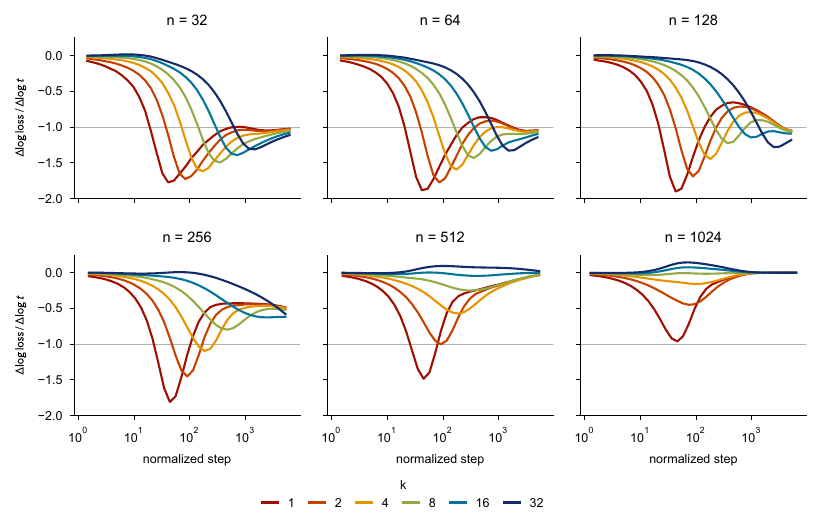}
    \caption{\textbf{Local log-log slope $\Delta\log\mathrm{loss} / \Delta\log t$ for each panel of Figure~\ref{fig:validation-learning-gf}.} Below the capacity threshold, all $k$ converge to slopes close to the dashed reference at $-1$, matching the gradient-flow prediction $D_\mathrm{KL} \sim k(d-k)/(td) \approx k/t$ for small $k$ values.}
    \label{fig:validation-learning-slope}
\end{figure}

The gradient-flow analysis above is derived for a single fixed target distribution and a $d$-dimensional logits vector that is updated directly. The actual toy setting we are interested in operates on a linear head with shared parameters across many contexts of varying ambiguity. We empirically check that the simplified theory still qualitatively captures the learning dynamics in this more realistic regime, and where it breaks down. We find that, below the capacity threshold, the per-$k$ KL curves match the theoretical picture: ambiguity slows learning early (Eq.~\ref{eqn:implicit_Delta_ell}) and the late-time decay tends to a slope of $-1$ on a log-log scale, as predicted by the asymptotics of Eq.~\ref{eqn:asymptotics_late_training}. As we add more contexts, the overall pace of learning slows, even after normalizing time by the number of contexts, presumably because gradient updates that are useful for one context interfere with others. Close to and beyond the capacity threshold, gradient flow is no longer the right lens: the network's finite capacity becomes the dominant constraint, and the less ambiguous distributions converge first because they are cheaper to store, as predicted by the capacity analysis of Appendix~\ref{app:capacity}.

\begin{table}[h]
    \centering
    \caption{Sweep configuration for the gradient-flow validation experiment of Figures~\ref{fig:validation-learning-gf} and~\ref{fig:validation-learning-slope}. Each ambiguity level forms a group of $n$ contexts, for a total of $6n$ context-distribution pairs per run.}
    \label{tab:learning-validation}
    \setlength{\tabcolsep}{8pt}
    \begin{tabular}{ll}
        \toprule
        \textbf{Architecture} & linear head, $h = d = 256$\\&($hd = 65{,}536$ parameters) \\
        \textbf{Optimizer} & stochastic gradient descent, batch size $1{,}024$, \\ & constant learning rate $1$, $6.4{\cdot}10^6$ steps \\
        \textbf{Support size} $k$ & $\{1,\, 2,\, 4,\, 8,\, 16,\, 32\}$, mixed in equal proportion \\
        \textbf{Contexts per ambiguity level} $n$ & $\{32,\, 64,\, 128,\, 256,\, 512,\, 1024\}$ \\
        \textbf{Seeds} & $5$ \\
        \bottomrule
    \end{tabular}
\end{table}

\paragraph{Experimental setup.} The setup mirrors the capacity validation experiment of Appendix~\ref{app:capacity-validation}, with two differences. First, we train with plain stochastic gradient descent instead of Adam, so as to stay close to the gradient-flow regime that the theory describes. Second, each context is associated with one of six target distributions of varying support size $k \in \{1, 2, 4, 8, 16, 32\}$, with the six ambiguity levels mixed in equal proportion within every batch: a single trained model thus sees the full range of ambiguities at the same time. This lets us probe whether the per-$k$ gradient-flow predictions still hold when distributions of different difficulty compete for the same parameters, as is expected to be the case in practice. The full sweep configuration is summarized in Table~\ref{tab:learning-validation}, where $n$ is the number of contexts per ambiguity level so that the total dataset contains $6n$ context-distribution pairs; we report the mean of the per-$k$ KL across $5$ seeds.

\newpage
\section{Theoretical results regarding sampling noise}
\label{app:sampling}

In this appendix we analyze the effect of label sampling noise on the learning dynamics of the logistic regression model of Appendix~\ref{app:optimization}. The setup is identical, except that we now train with sampled labels: at each step we draw $y\sim p$ and apply the stochastic gradient descent update on the cross-entropy loss, that is
\begin{equation}
    \label{eqn:sgd-update}
    \ell_j \leftarrow \ell_j + \eta\left(\mathbf 1[y=j] - \hat p_j\right),
\end{equation}
with $\eta > 0$ the learning rate. We measure time in units of the learning rate, $t = \eta \times (\text{number of steps})$, so that the drift of Eq.~\ref{eqn:sgd-update} matches the gradient flow $\dot \ell = p - \hat p$ of Appendix~\ref{app:optimization} and the two trajectories can be compared at equal~$t$. The main text states the results below at $\eta = 1$ for readability; we keep~$\eta$ explicit throughout this appendix because the continuous-time approximation we rely on is a small-$\eta$ statement, and the noise floor we derive is proportional to~$\eta$.

The update has mean $\eta\,(p_j - \hat p_j)$, that is, the gradient-flow drift, and a zero-mean per-step fluctuation whose consequences we now characterize. Recall that $\ell_t^\star$ denotes the gradient-flow trajectory of Appendix~\ref{app:optimization}, $\hat p^\star(t) = \mathrm{softmax}(\ell_t^\star)$ its softmax, and $\hat p^\star_S(t) := \sum_{j \in S} \hat p_j^\star(t)$ the on-support probability mass along that trajectory.

The central object of the analysis is the within-on-support logit variance
\begin{equation}
    \label{eqn:V-def-app}
    V := \frac{1}{k}\sum_{j \in S}(\ell_j - \bar\ell_S)^2, \, \text{with } \bar\ell_S := \frac{1}{k}\sum_{j \in S}\ell_j.
\end{equation}
Our analysis proceeds in four steps:
\begin{enumerate}
    \item Show that, at leading order in the fluctuations, the excess KL loss is entirely controlled by $\mathbb E V$, so that the whole problem reduces to understanding how this scalar evolves.
    \item Take the small-learning-rate continuous-time limit of the update in Eq.~\ref{eqn:sgd-update}, which is tracked by a stochastic differential equation driven by the per-step gradient covariance.
    \item Identify the subspace on which the fluctuations live.
    \item Apply Itô's lemma to $V$ to derive its ordinary differential equation, and combine everything to obtain the noise floor.
\end{enumerate}

\subsection{The within-on-support variance controls the excess loss}

We first show why $V$ is the natural quantity to track. Writing $\ell = \ell^\star + \delta$ and expanding the KL divergence to second order in $\delta$,
\begin{equation}
    \label{eqn:KL-expansion}
    D_\mathrm{KL}(p\|\hat p) = D_\mathrm{KL}(p\|\hat p^\star) + (\hat p^\star - p)^\top \delta + \tfrac{1}{2}\,\delta^\top H(\ell^\star)\,\delta + O(\|\delta\|^3),
\end{equation}
where $H(\ell) = \mathrm{diag}(\hat p) - \hat p \hat p^\top$ is the Hessian of the softmax cross-entropy loss. We will establish in the next subsections that, along the gradient-flow trajectory, $\delta$ has zero mean (so that the linear term vanishes in expectation) and lives in the $(k-1)$-dimensional within-on-support contrast subspace
\begin{equation}
    \label{eqn:contrast-subspace}
    \mathcal E := \Big\{u \in \mathbb R^d \,:\, u_j = 0 \text{ for } j \notin S, \ \sum_{j \in S} u_j = 0 \Big\},
\end{equation}
on which $H$ acts as $H u = (\hat p^\star_S(t)/k)\,u$. Under these facts, the quadratic form in Eq.~\ref{eqn:KL-expansion} reduces to $(\hat p^\star_S(t)/k)\,\|\delta\|^2$, and using $\|\delta\|^2 = k\,V$,
\begin{equation}
    \label{eqn:KL-excess-V}
    \mathbb E[D_\mathrm{KL}(p\|\hat p)] - D_\mathrm{KL}(p\|\hat p^\star) \approx \frac{\hat p^\star_S(t)}{2}\,\mathbb E V.
\end{equation}
The main effect of label sampling noise on the loss is therefore captured by the scalar $\mathbb E V$, and it remains to compute its dynamic, which we will do in the next three subsections.

\subsection{Continuous-time limit of the stochastic update}

We briefly recall the standard continuous-time approximation of stochastic gradient descent, following~\citet{mandt_stochastic_2017, li_stochastic_2017}. For a general loss $L(\theta)$ minimized with stochastic gradient descent with learning rate $\eta$ on parameter $\theta$, and per-step gradient noise covariance $\Sigma(\theta)$, the iterates are tracked in the small-learning-rate limit, on the time scale $t = \eta \times (\text{number of steps})$, by the stochastic differential equation
\begin{equation}
    \label{eqn:sde-generic}
    \mathrm{d}\theta_t = -\nabla L(\theta_t)\,\mathrm{d}t + \sqrt{\eta\,\Sigma(\theta_t)}\,\mathrm{d}W_t,
\end{equation}
with $W_t$ a standard Brownian motion. The drift is $\eta$-independent in this time parametrization, whereas the diffusion carries a $\sqrt\eta$: this is the usual statement that the learning rate controls how far the iterates wander from gradient flow, not how fast they travel along it. Writing $\theta_t = \theta_t^\star + \delta_t$, the deviation $\delta_t$ obeys, to leading order, the Ornstein--Uhlenbeck equation
\begin{equation}
    \label{eqn:sde-delta}
    \mathrm{d}\delta_t = -H(\theta_t^\star)\,\delta_t\,\mathrm{d}t + \sqrt{\eta\,\Sigma(\theta_t^\star)}\,\mathrm{d}W_t.
\end{equation}
For the softmax cross-entropy loss and the update in Eq.~\ref{eqn:sgd-update},
\begin{equation}
    \label{eqn:Sigma-H}
    \Sigma(\ell) = \mathrm{Cov}_{y\sim p}(\mathbf 1[y=\cdot]) = \mathrm{diag}(p) - pp^\top,
    \qquad
    H(\ell) = \mathrm{diag}(\hat p) - \hat p \hat p^\top.
\end{equation}
Note that $\Sigma$ depends only on the target $p$, not on the current logits, and that $\mathrm{Tr}(\Sigma) = 1 - \|p\|_2^2 = (k-1)/k$.

\subsection{Fluctuations live on the within-on-support contrast subspace}

Along the symmetric gradient-flow trajectory of Appendix~\ref{app:optimization}, all on-support logits are equal and likewise for off-support logits. For any $u \in \mathcal E$, direct computation from Eq.~\ref{eqn:Sigma-H} gives
\begin{equation}
    \label{eqn:joint-eigenvectors}
    \Sigma\,u = \frac{1}{k}\,u, \qquad H\,u = \frac{\hat p^\star_S(t)}{k}\,u,
\end{equation}
because $p^\top u = 0$ (the entries of $p$ are constant on $S$ and $u$ sums to zero on $S$) and likewise $\hat p^{\star\top} u = 0$. Outside $\mathcal E$, the covariance $\Sigma$ acts as zero: the symmetric on-support direction $\mathbf 1_S$ is annihilated since $\Sigma\,\mathbf 1_S = p - p\,(p^\top \mathbf 1_S) = p - p = 0$, and off-support coordinates are untouched by sampling since $p_j = 0$ for $j\notin S$. Thus all of the gradient noise feeds into the within-on-support contrast subspace $\mathcal E$, and $\delta_t$ has zero expectation outside it. Restricting to the projection $u_t := P_{\mathcal E}\delta_t$, we obtain the scalar Ornstein--Uhlenbeck equation
\begin{equation}
    \label{eqn:u-sde}
    \mathrm{d}u_t = -\frac{\hat p^\star_S(t)}{k}\,u_t\,\mathrm{d}t + \sqrt{\frac{\eta}{k}}\,\mathrm{d}\tilde W_t,
\end{equation}
where $\tilde W_t$ is a $(k-1)$-dimensional standard Brownian motion on $\mathcal E$. Note that the damping rate is set by the loss curvature and does not involve~$\eta$, while the noise injected per unit time is proportional to~$\eta$.

\subsection{Dynamics of the within-support variance and noise floor}

Since $\bar\ell_S$ is fixed under $P_{\mathcal E}$ (the symmetric direction is not driven by noise), we have $V = \|u_t\|^2 / k$ along the gradient-flow trajectory. Itô's lemma applied to $\|u_t\|^2$ with Eq.~\ref{eqn:u-sde} gives
\begin{align}
    \mathrm{d}\|u_t\|^2 &= 2\,u_t^\top \mathrm{d}u_t + \eta\,\mathrm{Tr}\!\left(\Sigma\big|_{\mathcal E}\right)\mathrm{d}t \\
    &= -\frac{2\hat p^\star_S(t)}{k}\,\|u_t\|^2\,\mathrm{d}t + 2\sqrt{\frac{\eta}{k}}\,u_t^\top \mathrm{d}\tilde W_t + \eta\,\frac{k-1}{k}\,\mathrm{d}t,
\end{align}
where we used
\begin{equation}
    \mathrm{Tr}(\Sigma|_{\mathcal E}) = \mathrm{Tr}(\Sigma) = \frac{k-1}{k}.
\end{equation}
Taking expectations kills the martingale term $\mathrm{d}\tilde W_t$, and dividing by $k$ yields the ordinary differential equation for $\mathbb E V$:
\begin{equation}
    \label{eqn:V-dynamics-app}
    \frac{\mathrm{d}\,\mathbb E V}{\mathrm{d}t} = -\frac{2\,\hat p^\star_S(t)}{k}\,\mathbb E V + \eta\,\frac{k-1}{k^2},
\end{equation}
which is the equation announced in the main text (Eq.~\ref{eqn:V-dynamics-main}) at $\eta = 1$.
As a result, $\mathbb E V$ converges to
\begin{equation}
    \label{eqn:V-stationary}
    \mathbb E V_\infty = \eta\,\frac{k-1}{2k}
\end{equation}
as $\hat p^\star_S(t)$ converges to $1$.
Plugging Eq.~\ref{eqn:V-stationary} into Eq.~\ref{eqn:KL-excess-V} gives the asymptotic noise floor
\begin{equation}
    \label{eqn:noise-floor-app}
    \lim_{t \to \infty}\big(\mathbb E[D_\mathrm{KL}^{\,\mathrm{SGD}}(t)] - D_\mathrm{KL}^{\,\mathrm{GF}}(t)\big) = \frac{1}{2} \cdot \eta\,\frac{k-1}{2k} = \eta\,\frac{k-1}{4k},
\end{equation}
which is the formula reported in the main text as Eq.~\ref{eqn:sgd-gap-main}, there stated at $\eta = 1$. It is exactly zero for $k=1$ (no ambiguity, no noise), monotonically increasing in $k$, and asymptotes to $\eta/4$ as $k\to\infty$. The learning rate only sets the overall scale of the floor: the ambiguity dependence, which is what we are after, is the $\eta$-independent factor $(k-1)/(4k)$.

\subsection{Remark: convergence rate is untouched at leading order}

The derivation above is a leading-order statement in the learning rate: the mean trajectory $\mathbb E \ell_t$ coincides with gradient flow up to higher-order corrections in~$\eta$, so the convergence rates of Appendix~\ref{app:optimization} carry over unchanged. Sampling only adds zero-mean fluctuations whose variance saturates at the floor in Eq.~\ref{eqn:noise-floor-app}. This contrasts with discrete stochastic gradient descent at finite step size, where Jensen-style corrections to the drift do modify the effective rate; classical analyses on logistic regression~\citep{bach_non-asymptotic_2011, bach_adaptivity_2014} and quadratic losses~\citep{flammarion_averaging_2015, mandt_stochastic_2017} give step-size slowdowns and floor values consistent with Eq.~\ref{eqn:noise-floor-app} as the leading term.

\subsection{Validation of theory}
\label{app:sampling-validation}

\begin{figure}[!t]
    \centering
    \includegraphics{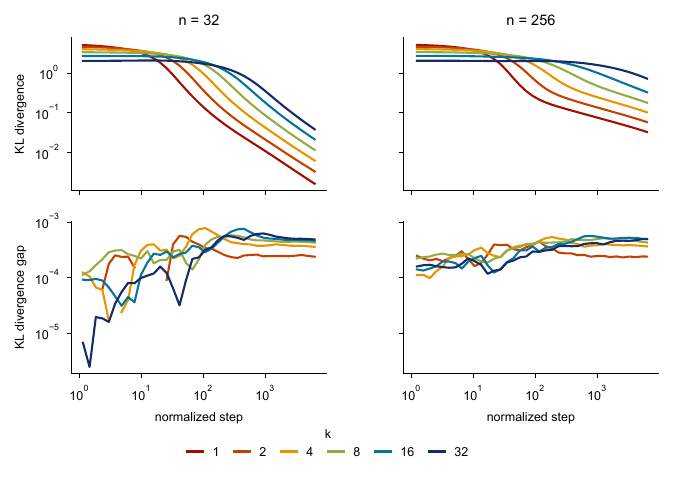}
    \caption{\textbf{The empirical sampling gap grows with $k$ asymptotically, in line with the theoretical prediction of Section~\ref{subsec:sampling}.} Top row: evolution of the mean per-$k$ KL divergence under stochastic gradient descent on the cross-entropy loss (with sampled targets), against the training step normalized by the number of contexts per ambiguity level $n$, for $n \in \{32, 256\}$. Bottom row: gap to the matched distributional-target reference (full-batch gradient descent on the same task); $k=1$ is omitted as its gap is below the seed-noise floor. The late-time gap increases with $k$, matching the asymptote of Equation~\ref{eqn:sgd-gap-main}.}
    \label{fig:validation-learning-sampling-gap}
\end{figure}

The motivation here mirrors that of Appendix~\ref{app:optimization-validation}: the analysis of Appendix~\ref{app:sampling} is derived for a single fixed target distribution and a logits vector updated directly, and we want to check that its qualitative predictions carry over to a linear head with shared parameters across many contexts of varying ambiguity. The key difference is that we now study the impact of output sampling noise rather than the underlying gradient-flow dynamics. The numerics are demanding because the quantity of interest is the small gap between the sampled and the population trajectory; as a reference, the theoretical curves of Figure~\ref{fig:learning} required averaging over $1{,}024$ Monte-Carlo trajectories to be smooth, which is already costly for the simple linear models we consider here. As a result, we cannot resolve the precise early-time onset of the sampling gap, but Figure~\ref{fig:validation-learning-sampling-gap} shows that the late-time excess loss increases with $k$, consistent with the prediction of Equation~\ref{eqn:sgd-gap-main}.

\paragraph{Experimental setup.} The setup is identical to the gradient-flow validation of Appendix~\ref{app:optimization-validation} (Table~\ref{tab:learning-validation}), except that targets are now sampled at every step, so the model is trained on the cross-entropy loss instead of the full KL divergence (which is, up to the entropy of the target, the population objective used in the gradient-flow validation). To resolve the small sampling gap we restrict to $n \in \{32, 256\}$, shorten training to $1.6{\cdot}10^6$ steps, and average over $32$ seeds for both the sampling sweep and the matched distributional reference used to compute the gap. These runs use $\eta = 1$ and mini-batches spread over all contexts, so the effective per-context learning rate differs from the single-context setting of Eq.~\ref{eqn:sgd-update}; this rescales the absolute height of the plateau, and we therefore only compare its $k$~dependence to the theory.

\newpage
\section{Connection to classical sampling results}
\label{app:classical}

The analyses of Appendices~\ref{app:optimization}--\ref{app:sampling} look at the curse of ambiguity through the lens of learning; that is, how a parametric model trained by gradient descent on a finite set of contexts fits a target distribution.
We complement that view here with two classical, non-parametric results on sampling from a single discrete distribution: the coupon collector problem and the maximum likelihood estimator of a categorical distribution. 
The curse of ambiguity already appears in these simple statistical problems: the sample complexity grows with $k$ even in this idealized setting where no representation or parameters have to be learned.

\subsection{The coupon collector problem}
\label{app:coupon}

For a neural network to model $p$ accurately, it must, at the very least, have observed every token in the support $S$ at least once during training: a token that never appears in the data is a token that the network cannot meaningfully tell apart from any other off-support token.
We can therefore lower bound the sample complexity of learning $p$ by the time it takes for all $k$ tokens of $S$ to appear in the sample.

This problem is known as the coupon collector problem~\citep{erdos_classical_1961, feller_introduction_1968}: how many independent draws from $p$ are needed to observe all of its $k$ values?
For the uniform distribution we consider, the expected collection time $T_k$ satisfies
\begin{equation}
    \label{eqn:coupon-mean}
    \mathbb{E}[T_k] = k \sum_{i=1}^k \frac{1}{i} = k \log k + \gamma k + \frac{1}{2} + o(1),
\end{equation}
with $\gamma \approx 0.577$ the Euler-Mascheroni constant.
Generalizations of Equation~\ref{eqn:coupon-mean} to non-uniform $p$ are known but less clean, with the expected collection time depending on the full distribution rather than on a single summary statistic; see~\citet{flajolet_birthday_1992} for an extensive treatment.
The qualitative picture is unchanged: $\mathbb{E}[T] \geq 1/p_\mathrm{min}$ at minimum, and the logarithmic overhead in the effective number of tokens persists.

Two consequences of these results are worth highlighting. First, the cost scales as $k \log k$ so the ``naive'' budget of one sample per support element is not enough, and even seeing each token a constant number of times requires an additional logarithmic factor in $k$.
Second, this $k \log k$ floor holds regardless of the model, the optimizer, or the loss.
As a result, any learner will inherit from this sampling-based curse of ambiguity.

\subsection{Maximum likelihood estimation of a categorical distribution}
\label{app:mle}

Once all $k$ tokens of the support have been observed, a learner can start estimating the actual probabilities on top of identifying the support. The standard non-parametric estimator for this problem is the maximum likelihood estimator
\begin{equation}
    \label{eqn:mle}
    \hat p_j = \frac{N_j}{m}, \quad \text{where } N_j = \sum_{t=1}^m \mathbf{1}[y_t = j],
\end{equation}
with $m$ the number of drawn samples.

\paragraph{Total variation.} The total-variation rate of the frequency estimator has been studied extensively~\citep{devroye_combinatorial_2001, berend_sharp_2013, kamath_learning_2015, han_minimax_2015}. For a uniform target on a support of size $k$ and at least as many samples as support elements, $m \geq k$, we have
\begin{equation}
    \label{eqn:mle-tv}
    \mathbb{E}\,\|\hat p - p\|_1 \;\asymp\; \sqrt{\frac{k}{m}},
\end{equation}
which matches the minimax lower bound over all distributions on $k$ symbols~\citep{han_minimax_2015}.
The restriction to $m \geq k$ is only there to exclude the degenerate regime $m \ll k$, in which the $\ell_1$ error saturates at its maximal value instead of following Eq.~\ref{eqn:mle-tv}.
Reaching a fixed total-variation error $\varepsilon$ thus requires $m \asymp k / \varepsilon^2$ samples, which is a manifestation of the curse of ambiguity, this time at the level of pure statistics.

\paragraph{KL divergence.}
Until all $k$ support tokens have been observed, $\hat p$ has some $\hat p_j = 0$ on the support of $p$ and $D_\mathrm{KL}(p \,\|\, \hat p) = +\infty$. Since that happens with positive probability at every finite sample size, the unconditional expected KL divergence of the frequency estimator is infinite for every~$m$: collecting all coupons is not just a lower bound for the support-identification problem, it is a strict prerequisite for the KL divergence of the maximum likelihood estimator to be finite at all. We therefore state the rate conditionally on the coupon-collection event of Appendix~\ref{app:coupon}, which by Eq.~\ref{eqn:coupon-mean} has high probability once $m \gg k \log k$. In that regime the frequency estimator attains the minimax rate for this problem \citep{braess_bernstein_2004, han_minimax_2015},
\begin{equation}
    \label{eqn:mle-kl}
    \mathbb{E}\left[D_\mathrm{KL}(p \,\|\, \hat p) \;\middle|\; \text{all } k \text{ tokens observed}\right] \;\asymp\; \frac{k}{m},
\end{equation}
a rate that the same references show is achieved unconditionally, at every $m \geq k$, by the add-constant smoothed variants of $\hat p$ that are standard precisely because they avoid the infinite-KL issue above.
The numerator now scales linearly in $k$. This matches qualitatively the late-time gradient flow loss $D_\mathrm{KL} \approx k(d-k)/(td) \approx k/t$ of Section~\ref{subsec:optimization}, confirming that the $k$ dependence we observe in our learning dynamics analysis reflects the intrinsic statistical difficulty of the problem.

\paragraph{Estimation of functionals.}
Scalar functionals of $p$, most notably its entropy or support size, admit faster estimators than $p$ itself: plug-in estimators incur an $O(k/m)$ bias~\citep{miller_note_1955, paninski_estimation_2003}, while polynomial-approximation estimators~\citep{valiant_estimating_2011, jiao_minimax_2015, wu_minimax_2016, valiant_instance_2017} reach the minimax rate $m \asymp k / \log k$, strictly faster than the $\Theta(k)$ rate of full-distribution estimation.
Language models, however, are trained to model the full next-token distribution rather than scalar summaries of it, so our analysis inherits the harder $\Theta(k)$ scaling of Equations~\ref{eqn:mle-tv} and~\ref{eqn:mle-kl} rather than this milder $k / \log k$ rate.

\newpage
\section{Synthetic $n$-gram task}
\label{app:exp_details}
\label{app:exp_details_ngram}

\paragraph{Dataset construction.} We start from the TinyStories corpus \citep{eldan_tinystories_2023}, lowercase each story, strip non-alphabetic characters, and split on whitespace. We then take as vocabulary the $V = 2{,}000$ most frequent words, all of which occur more than $10{,}000$ times in the corpus, and count $3$-gram occurrences over the corpus, dropping $3$-grams that contain any out-of-vocabulary word. Normalizing these counts yields the ground-truth conditional distributions $p_x$ for approximately $1.1$ million unique contexts $x$ (pairs of in-vocabulary words). For each context, we compute the effective support size $\exp(H(p_x))$ directly from $p_x$. We also report results for the raw support size, defined as the number of next tokens with $p_{x,j} \geq 1/(10V)$; this threshold removes tokens whose empirical probability sits at numerical-noise level relative to the uniform distribution.

\paragraph{Ambiguity measured by raw support size.} The main text quantifies ambiguity via the effective support size. To check that our conclusions do not hinge on this choice, Figure~\ref{fig:ngram-confirmation-curse-support} reproduces Figure~\ref{fig:ngram-confirmation-curse} with the raw support size on the y-axis, and shows the same qualitative trends.

\begin{figure}[h]
    \centering
    \includegraphics{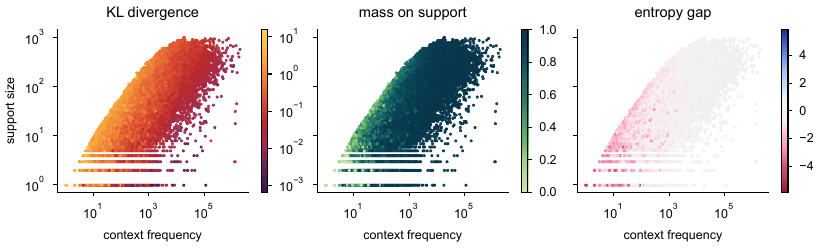}
    \caption{Same as Figure~\ref{fig:ngram-confirmation-curse}, but using the raw support size of the ground-truth distribution as the measure of ambiguity instead of the effective support size.}
    \label{fig:ngram-confirmation-curse-support}
\end{figure}

\paragraph{Sequence sampling.} Training sequences are sampled autoregressively from the empirical $3$-gram distribution. Each sequence is initialized with two beginning-of-sequence tokens, and at each subsequent position $t$ we draw
\begin{equation}
    y_t \sim p(\cdot \mid y_{t-2},\, y_{t-1}),
\end{equation}
where $p(\cdot \mid y_{t-2}, y_{t-1})$ is the empirical conditional. Two edge cases arise. First, contexts containing one or two beginning-of-sequence tokens use the corresponding empirical unigram and bigram marginals. Second, the autoregressive procedure can produce a pair of in-vocabulary words that never co-occurred in the corpus, in which case the empirical conditional is undefined; we fall back to the uniform distribution over the vocabulary for these unseen contexts. Both edge cases account for a small fraction of the sampled tokens and are excluded from the per-context evaluation, which is restricted to contexts $x$ for which $p_x$ is well defined.

\paragraph{Architecture and training.} The models are decoder-only transformers with $4$ layers, $4$ attention heads, rotary position embeddings, and an MLP expansion factor of $4$; at the default width $d_\mathrm{model} = 256$ this amounts to $4$M parameters. We train on sequences of $256$ tokens with batches of $128$ sequences, so the default run of $3{,}000$ steps sees $98$M tokens, using AdamW with weight decay $0.01$ and a cosine schedule. We use one data seed per configuration and tune the learning rate per configuration, as summarized in Table~\ref{tab:ablations}.

We separate two width parameters. The residual stream, and with it the whole backbone, has width $d_\mathrm{model}$; the linear head that produces the logits acts on a $d_\mathrm{embed}$-dimensional vector, obtained from the last-layer residual stream by an additional linear projection when $d_\mathrm{embed} \neq d_\mathrm{model}$ (no projection is inserted when they are equal). It is $d_\mathrm{embed}$ that plays the role of the embedding dimension~$h$ of our theory, which is why the embedding-dimension ablation varies it at fixed backbone. One caveat is worth stating: since the projection is applied to a $d_\mathrm{model}$-dimensional vector, the embeddings entering the head span a subspace of dimension at most $\min(d_\mathrm{embed}, d_\mathrm{model})$, so the two largest values of the ablation, $d_\mathrm{embed} \in \{512, 1024\}$ at $d_\mathrm{model} = 256$, cannot buy representational room beyond $256$ dimensions. The informative part of that sweep is therefore its lower half, where the head is genuinely bottlenecked.

\paragraph{Per-context evaluation.} At the end of training we sample $262{,}144$ fresh sequences from the same $3$-gram process, run the model on them, and average each per-context metric over the positions at which the context occurs. Each dot of Figure~\ref{fig:ngram-confirmation-curse} is one context, and its context frequency is its count in the training corpus rather than in this evaluation sample.

\paragraph{Iso-loss pipeline.} We exploit the empirical observation that $\log D_\mathrm{KL}$ is well approximated by a linear function of the log effective support size and the log context frequency $\log f$. For each configuration, we fit
\begin{equation}
    \label{eqn:iso-fit}
    \log D_\mathrm{KL}(x) \approx a \log k(x) + b \log f(x) + c
\end{equation}
to the per-context losses by weighted least squares, where $k(x)$ denotes the effective support size of context $x$. Since the bulk of contexts sits in the dense low-$k$, low-frequency region, we bin the $(\log k, \log f)$ plane into a $50 \times 50$ grid and weight each context by the inverse of the number of contexts in its bin; this prevents the dense region from dominating the fit and effectively targets the support of the $(k, f)$-distribution rather than its mode. We then report the level set on which the fitted $D_\mathrm{KL}$ equals a representative mid-range value of $0.5$ nats, chosen because it sits in the middle of the per-context loss range observed in Figure~\ref{fig:ngram-confirmation-curse}. Figure~\ref{fig:iso_loss_pipeline} summarizes this pipeline when applied on Figure~\ref{fig:ngram-confirmation-curse} (left).

\begin{figure}[ht]
    \centering
    \includegraphics[]{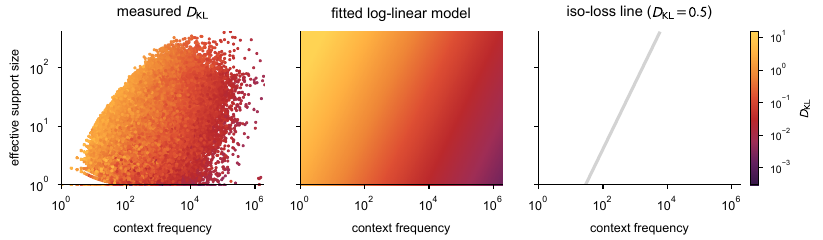}
    \caption{\textbf{Visual representation of how we summarize the per-context loss landscape of Figure~\ref{fig:ngram-confirmation-curse} (left) into the iso-loss line we use in Figure~\ref{fig:ablation-roots}.}}
    \label{fig:iso_loss_pipeline}
\end{figure}

\paragraph{Extensive ablation.} Figure~\ref{fig:ablation-roots} summarizes each ablation configuration by a single iso-loss contour fitted to its per-context losses. For completeness, we show in Figures~\ref{fig:ablation-clouds-capacity}--\ref{fig:ablation-clouds-sampling} the underlying per-context data from which those contours are derived: each row is one configuration of the corresponding ablation, plotted as a scatter of $(\text{context frequency},\, \text{effective support size})$ points colored by the same three metrics as Figure~\ref{fig:ngram-confirmation-curse} (KL divergence, mass on support, entropy gap). Within each figure, color scales are shared across rows so that configurations can be compared directly. These cloud-point views show that the qualitative trends summarized by the iso-loss contours are visible at the per-context level and not artifacts of the log-linear fit of Equation~\ref{eqn:iso-fit}.

\begin{table}[h]
    \centering
    \setlength{\tabcolsep}{4pt}
    \renewcommand{\arraystretch}{1.2}
    \begin{tabular}{lccccc}
          \toprule
          & Default & Capacity & Embedding
  & Optimization & Sampling \\
          \midrule
          Architecture &
  \multicolumn{5}{c}{$4$-layer transformer,
   $4$ attention heads} \\
          Vocabulary   &
  \multicolumn{5}{c}{$2{,}000$ most frequent
  words of TinyStories} \\
          Data          &
  \multicolumn{5}{c}{$256$-token sequences,
  batch size $128$, one seed} \\
          Optimizer    &
  \multicolumn{5}{c}{AdamW, cosine
  schedule, weight decay $0.01$} \\
          Learning rate & \makecell{$\{10^{-4},\, 3
  \times 10^{-4},$\\$10^{-3}, \, 3 \times
  10^{-3}\}$} & \multicolumn{4}{c}{$\{3
  \times 10^{-4},\, 10^{-3}, \, 3 \times
  10^{-3}\}$} \\
          \midrule
          $d_\mathrm{model}$ & $256$ &
  \makecell{$\mathbf{\{32,\, 64,}$\\$\mathbf{128,\, 256,}$\\$\mathbf{512,\, 1024\}}$} & $256$ &
  $256$ & $256$ \\
          $d_\mathrm{embed}$ & $256$ &
  $256$ & \makecell{$\mathbf{\{32,\, 64,}$\\$\mathbf{128,\, 256,}$\\$\mathbf{512,\, 1024\}}$} & $256$ &
  $256$ \\
          Training steps & $3000$ & $3000$
  & $3000$ & \makecell{$\mathbf{\{356,\,
  725,}$\\$\mathbf{1500,\,
  3000,}$\\$\mathbf{6000,\, 12000\}}$} &
  $3000$ \\
          Distributional & no & no & no &
  no & \makecell{\{\textbf{yes},\,
  \textbf{no}\}} \\
          \midrule
          Runs & $4$ & $18$ & $18$ & $18$ &
   $6$ \\
          \bottomrule
      \end{tabular}
    \vspace{1em}
    \caption{\textbf{Configuration of ablation experiments on the synthetic $3$-gram task.} Each column corresponds to a sweep that varies a single parameter (highlighted in bold) while holding all others at the default value. Learning rate is tuned for every configuration in every sweep: the number of runs is the number of swept values times the size of the learning-rate grid, and the reported curves use the best learning rate ($3 \times 10^{-3}$ for the default configuration).}
    \label{tab:ablations}
\end{table}

\begin{figure}[!h]
    \centering
    \includegraphics{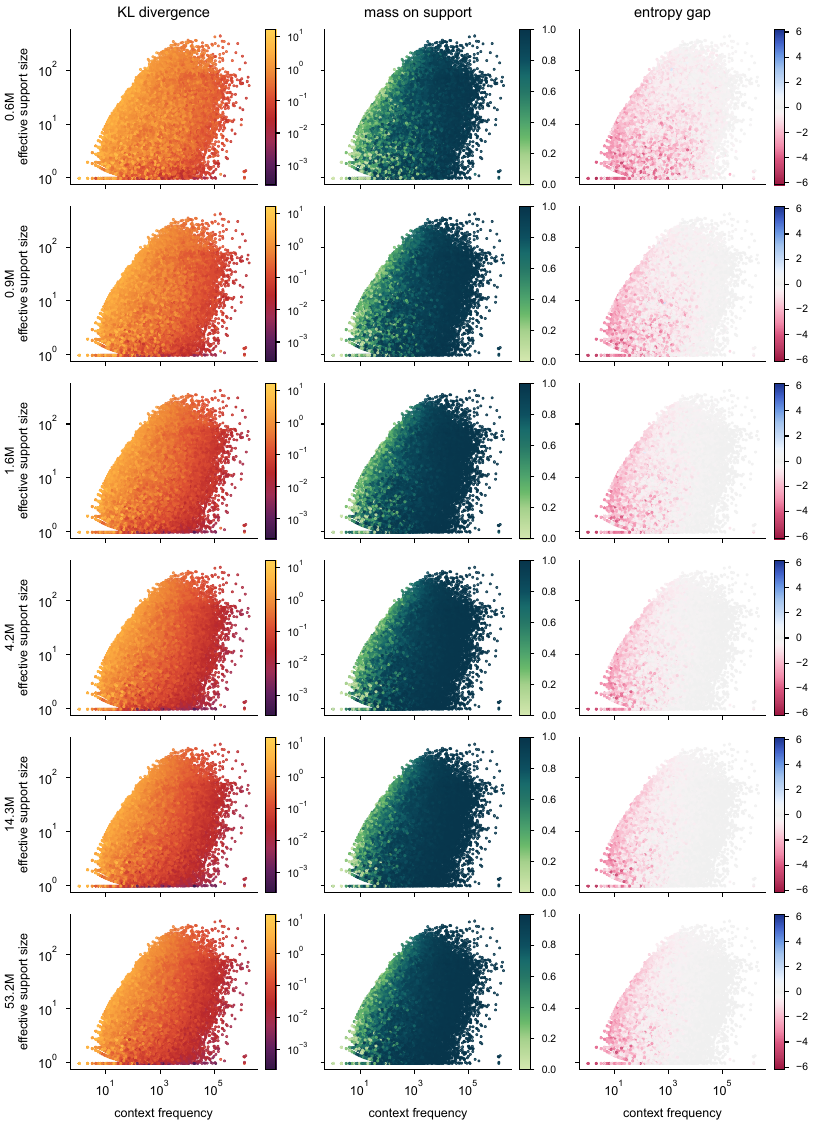}
    \caption{\textbf{Per-context losses across the capacity ablation.} Each row corresponds to one value of $d_\mathrm{model} \in \{32,\, 64,\, 128,\, 256,\, 512,\, 1024\}$ (annotated as parameter count in millions). All other parameters are held at the defaults of Table~\ref{tab:ablations}: embedding dimension $256$, $3000$ training steps, sampled cross-entropy loss, with learning rate selected per configuration as the best of $\{3 \times 10^{-4},\, 10^{-3},\, 3 \times 10^{-3}\}$.}
    \label{fig:ablation-clouds-capacity}
\end{figure}

\begin{figure}[!h]
    \centering
    \includegraphics{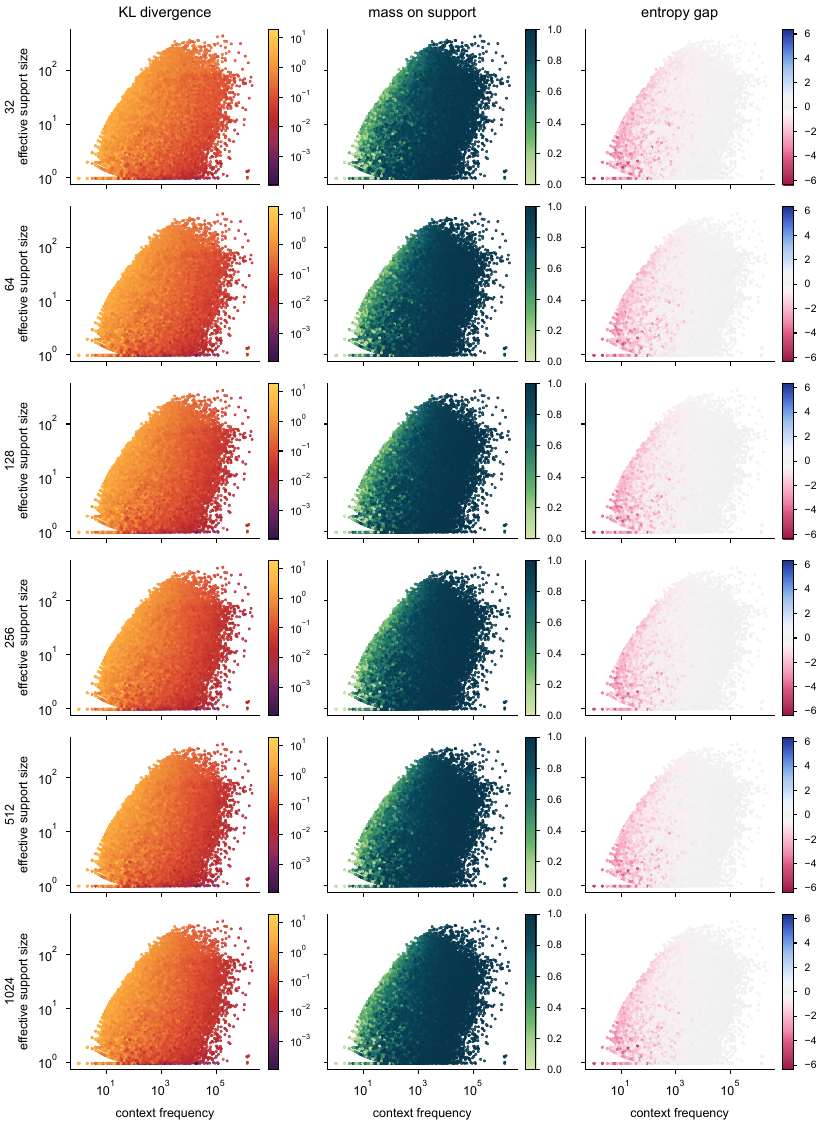}
    \caption{\textbf{Per-context losses across the embedding-dimension ablation.} Each row corresponds to one value of $d_\mathrm{embed} \in \{32,\, 64,\, 128,\, 256,\, 512,\, 1024\}$. All other parameters are held at the defaults of Table~\ref{tab:ablations}: $d_\mathrm{model} = 256$, $3000$ training steps, sampled cross-entropy loss, with learning rate selected per configuration as the best of $\{3 \times 10^{-4},\, 10^{-3},\, 3 \times 10^{-3}\}$.}
    \label{fig:ablation-clouds-embedding}
\end{figure}

\begin{figure}[!h]
    \centering
    \includegraphics{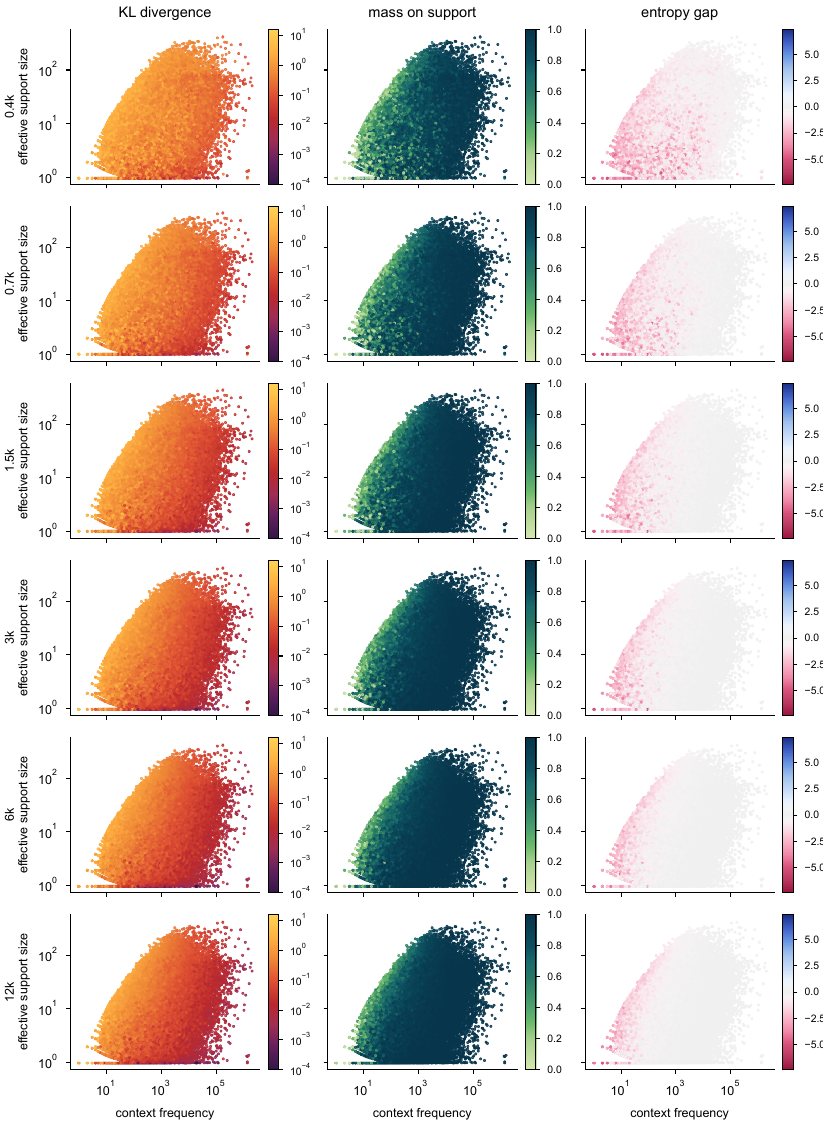}
    \caption{\textbf{Per-context losses across the training-duration ablation.} Each row corresponds to one number of training steps in $\{356,\, 725,\, 1500,\, 3000,\, 6000,\, 12000\}$. All other parameters are held at the defaults of Table~\ref{tab:ablations}: $d_\mathrm{model} = d_\mathrm{embed} = 256$, sampled cross-entropy loss, with learning rate selected per configuration as the best of $\{3 \times 10^{-4},\, 10^{-3},\, 3 \times 10^{-3}\}$.}
    \label{fig:ablation-clouds-optimization}
\end{figure}

\begin{figure}[!h]
    \centering
    \includegraphics{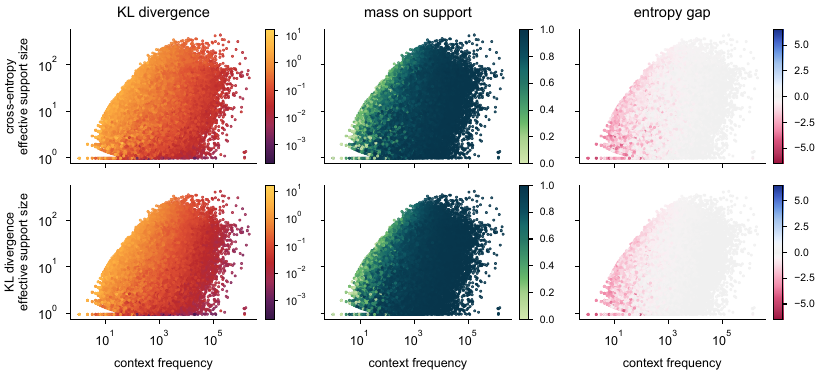}
    \caption{\textbf{Per-context losses across the loss ablation.} The two rows compare optimizing the full KL divergence $D_\mathrm{KL}(p_x \,\|\, \hat p_x)$ (top) against the sampled cross-entropy with one label drawn per step (bottom). All other parameters are held at the defaults of Table~\ref{tab:ablations}: $d_\mathrm{model} = d_\mathrm{embed} = 256$, $3000$ training steps, with learning rate selected per configuration as the best of $\{3 \times 10^{-4},\, 10^{-3},\, 3 \times 10^{-3}\}$.}
    \label{fig:ablation-clouds-sampling}
\end{figure}

\newpage
\clearpage
\section{Large language models experiments}

\subsection{Experiment details}
\label{app:exp_details_realworld}

\paragraph{Models and data.} We use OLMo-2 13B \citep{olmo_2_2025} as a proxy for the ground-truth next-token distribution and OLMo-2 1B as the main model evaluated against this proxy. Both checkpoints are taken before any post-training. We sample $100{,}000$ sequences of length $512$ from the Dolma v1.7 corpus \citep{soldaini_dolma_2024} on which OLMo-2 was pretrained, tokenize them with the OLMo-2 tokenizer, and run a single forward pass of each model on these sequences in bfloat16. This yields $100{,}000 \times 512 = 5.12 \times 10^7$ context-target pairs, one per (sequence, position).

\paragraph{Per-context metrics.} For each context, we compute the effective support size $\exp(H(p_\mathrm{13B}))$ of the proxy ground-truth distribution; we also report results for the raw support size, defined as the number of tokens with predicted probability $\geq 1/(10V)$ under OLMo-2 13B, where $V$ is the OLMo-2 vocabulary size. Both procedures match those used for the $3$-gram task in Appendix~\ref{app:exp_details_ngram}. We then evaluate OLMo-2 1B against the same proxy ground truth via the per-context KL divergence $D_\mathrm{KL}(p_\mathrm{13B} \,\|\, p_\mathrm{1B})$, the mass that OLMo-2 1B assigns to the support of OLMo-2 13B, and the entropy gap between the two predicted distributions, exactly as in Figure~\ref{fig:ngram-confirmation-curse}.

\paragraph{Semantic context frequency.} Since natural sequences rarely repeat past a few tokens, the context frequency that drives our $3$-gram analysis cannot be read off from the corpus directly. We instead estimate it by performing density estimation on the cloud of OLMo-2 13B embeddings of all evaluated contexts, leveraging the fact that semantically similar contexts produce similar embeddings. Concretely, we
\begin{itemize}
    \item[i)] extract the last-layer hidden state of OLMo-2 13B for every context, which provides a $5120$-dimensional embedding,
    \item[ii)] PCA-reduce the resulting cloud to dimension $512$, fitting the projection on a uniformly subsampled set of $5 \times 10^5$ embeddings,
    \item[iii)] build a FAISS index \citep{douze_faiss_2026} on the projected cloud,
    \item[iv)] for each context, query its $2{,}000$ nearest neighbors using the index, excluding the query itself.
\end{itemize}
We then compute $n_\mathrm{est} = \sum_i \mathrm{e}^{-d_i / \lambda}$ over these nearest neighbors, where $d_i$ is the distance to the $i$-th neighbor. This kernel density estimate counts how many contexts in the data are similar to the query, and is what we call the semantic context frequency. It depends on two choices, the bandwidth~$\lambda$ and the number of neighbors retrieved. We use $\lambda = 10$ and $2{,}000$ neighbors; the latter caps $n_\mathrm{est}$ at $2{,}000$, and the largest value we observe at $\lambda = 10$ is around $10^3$, so the cap only compresses the extreme right tail of the cloud. Varying $\lambda$ over $\{5, 10, 20\}$ yields qualitatively similar plots, as shown in Figure~\ref{fig:realworld-ablation-lambda}. We visualize the joint distribution of $n_\mathrm{est}$ and the estimated effective support size, together with the average position in the sequence on which each context occurs, in Figure~\ref{fig:realworld-position-density}.

\begin{figure}[!h]
    \centering
    \includegraphics{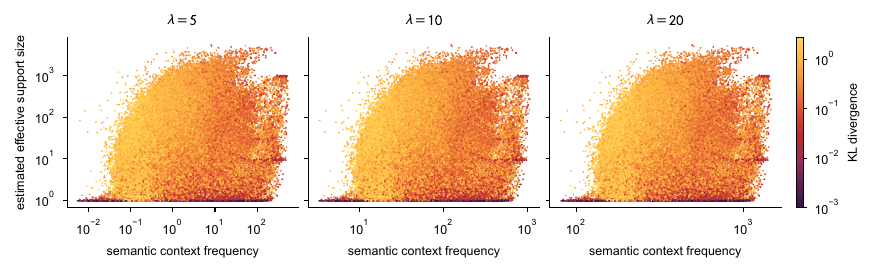}
    \caption{\textbf{Effect of changing the kernel bandwidth $\lambda$ on the results of Figure~\ref{fig:realworld-curse}.}
    Same setup as Figure~\ref{fig:realworld-curse}, varying the bandwidth $\lambda$ of the exponential kernel $\exp(-d_i/\lambda)$    used to compute $n_\mathrm{est}$ from kNN distances.
    Larger $\lambda$ shifts the x-axis to the right (more neighbors are weighted in), but the cloud shape and the per-context KL coloring are essentially invariant under $\lambda$.}
    \label{fig:realworld-ablation-lambda}
\end{figure}

\begin{figure}[!h]
    \centering
    \includegraphics{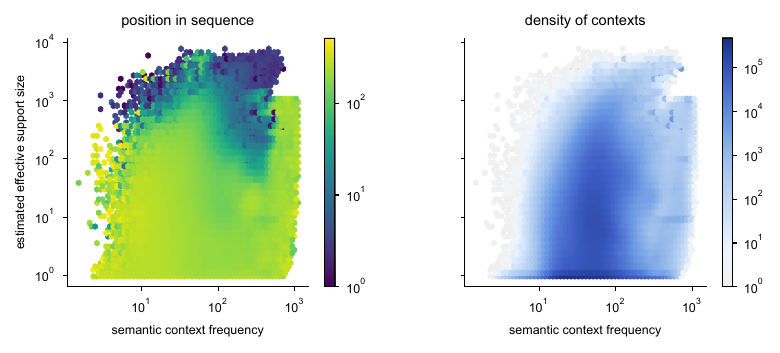}
    \caption{\textbf{Joint distribution of semantic context frequency and estimated effective support size, alongside the average position in the sequence at which each context occurs.} Both panels share axes with Figure~\ref{fig:realworld-curse}. \emph{Left:} each hex is colored by the geometric-mean position (in tokens) of the contexts that fall in it, with opacity modulated by the log of the per-hex count so sparse cells fade out. Low-frequency, low-effective-support cells are dominated by early-sequence positions, where local context is short and the kNN density estimate is least reliable. \emph{Right:} log-density of contexts in the same plane. Most of the mass concentrates in a narrow band where $n_\mathrm{est}$ and effective support size are positively correlated.}
    \label{fig:realworld-position-density}
\end{figure}

\paragraph{Validation against true context frequency.} On natural data, we have no direct access to the true context frequency, so the proxy above cannot be checked end-to-end. To validate that it carries the right information, we apply the same density pipeline on the synthetic $3$-gram task of Section~\ref{sec:experiments}, where the true context frequency is known. We extract the last-layer hidden state of the $4$-layer transformer of Appendix~\ref{app:exp_details_ngram} at every position of a held-out evaluation set, and run the same FAISS-based density estimation, up to two immaterial differences: the kernel weights are averaged over the retrieved neighbors rather than summed, which only rescales the proxy by a constant, and the bandwidth is set to $\lambda = 2$, the scale of the distances in this smaller embedding space. This yields $2.5 \times 10^7$ evaluated positions, spanning $5.2 \times 10^5$ distinct contexts. Figure~\ref{fig:realworld-validation} compares the resulting proxy to the true context frequency: the two are strongly correlated, with a Pearson correlation of $0.86$ between the log frequencies over all evaluated positions and a Spearman correlation of $0.92$ over $2 \times 10^6$ positions subsampled uniformly among them (both correlations vary by less than $0.03$ across the bandwidths $\lambda \in \{1, 2, 5\}$ we ran on this task).
Weighting positions equally is the relevant convention here, since this is also how contexts enter the per-context evaluation of Figure~\ref{fig:realworld-curse}, each context contributing in proportion to how often it occurs.
It does mean that the agreement is dominated by frequent contexts: computed over distinct contexts instead, so that a context seen once weighs as much as one seen a million times, the same correlations drop to $0.30$ and $0.24$.
The proxy therefore ranks frequent contexts reliably but is noisy on the rare tail, consistent with the caveat we make in Figure~\ref{fig:realworld-position-density} about the low-$n_\mathrm{est}$ region.

\begin{figure}[!h]
    \centering
    \includegraphics{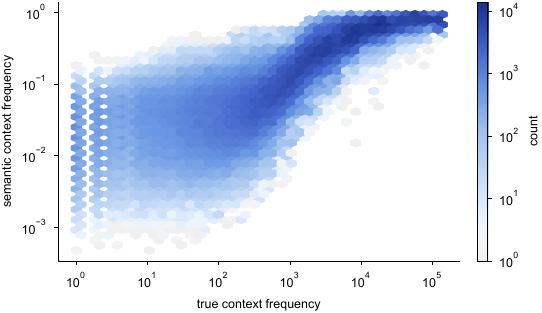}
    \caption{\textbf{The kNN density proxy correlates with the true context frequency on the $3$-gram task.} Log-density (hexbin, $10^6$ positions subsampled uniformly from the held-out evaluation set) of the true frequency of each context in the training corpus (x-axis) against its kernel density $\mathrm{e}^{-d_i/\lambda}$ averaged over its retrieved nearest neighbors in the last-layer embedding space of the trained $4$-layer transformer (y-axis), at $\lambda = 2$. On this task we plot the average rather than the sum of Appendix~\ref{app:exp_details_realworld}, which is why the y-axis lies below one; the two differ by the constant number of retrieved neighbors, so both correlations below are identical for either choice. The two quantities are strongly rank-correlated across positions (Spearman correlation of $0.92$), supporting the use of $n_\mathrm{est}$ as a context-frequency proxy on natural data, where the true count is not observable.}
    \label{fig:realworld-validation}
\end{figure}

\paragraph{Additional ablations.} For completeness, we report two further ablations. First, we replace effective support size by the raw support size of the OLMo-2 13B distribution as the ambiguity measure (Figure~\ref{fig:realworld-curse-support}). The qualitative picture is unchanged. Second, we vary the size of the probe model from 1B to 7B parameters (Figure~\ref{fig:realworld-ablation-probe}). The 7B probe tracks the reference more closely overall, with the largest gains concentrated at low effective support size and high context frequency, that is, in the bottom-right region of the cloud. We have two hypotheses for where this improvement takes place:
\begin{itemize}
    \item[(i)] The 7B probe has more capacity than the 1B probe, allowing it to fit medium-ambiguity contexts that exceeded the 1B model's capacity. This could explain the phase-transition-like improvement at the bottom of the cloud.
    \item[(ii)] On the right side of the cloud, where contexts are common enough to escape the capacity bottleneck, the gain could be explained by the gradual late-time decay of Equation~\ref{eqn:asymptotics_late_training} from the optimization analysis of Section~\ref{subsec:optimization}, in which the convergence rate no longer depends on the support size and all contexts improve in parallel.
\end{itemize}
These remain speculations and verifying them empirically is left to future work.

\begin{figure}[!h]
      \centering
      \includegraphics{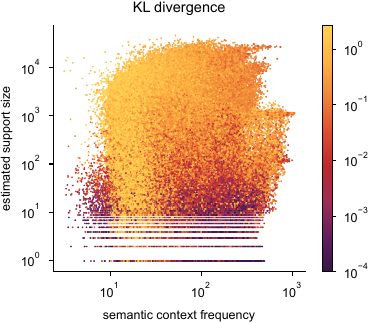}
      \caption{\textbf{The curse of ambiguity persists when ambiguity is measured by raw support size instead of effective support size.} Same setup as Figure~\ref{fig:realworld-curse}, except that the y-axis now shows the raw support size of the OLMo-2 13B reference distribution. The qualitative picture is unchanged.}
      \label{fig:realworld-curse-support}
  \end{figure}

  \begin{figure}[!h]
      \centering
      \includegraphics{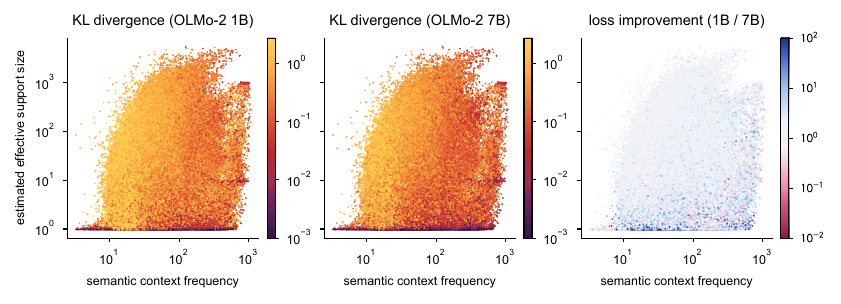}
      \caption{\textbf{Effect of varying the size of the evaluated model on the results of Figure~\ref{fig:realworld-curse}.} Same setup as Figure~\ref{fig:realworld-curse}, with OLMo-2 1B (left) and OLMo-2 7B (middle) as the model evaluated against the OLMo-2 13B reference. The 7B model tracks the reference more closely (KL is lower across the cloud), mirroring the improvement observed when increasing model size in the synthetic $3$-gram task; cf.~Figures~\ref{fig:ablation-roots} and~\ref{fig:ablation-clouds-capacity}.}
      \label{fig:realworld-ablation-probe}
  \end{figure}

\subsection{In-context predictability explains the low-loss low-context-frequency band}
\label{app:realworld_in_context_band}

Figure~\ref{fig:realworld-curse} reveals a band of tokens with low semantic context frequency ($n_\mathrm{est} < 10$) on which the small probe model and the large target model nevertheless agree very tightly. This is surprising at first sight: how can a model accurately predict the next token given a context that has rarely been seen during training?

We investigate this band by inspecting the underlying contexts (Table~\ref{tab:realworld_band_examples}). In each reported example, the last few tokens preceding the prediction already appear earlier in the same sequence, followed there by the same continuation. A generic in-context copying mechanism, such as an induction head, would suffice to make these predictions. To verify that this pattern extends beyond the subsampled examples of Table~\ref{tab:realworld_band_examples}, we measure, for each token in the band, whether the one- or two-token prefix preceding it appears at any earlier position of the same sequence and whether the actual next token coincides with one of the in-context continuations. We find that $61\%$ of in-band predictions are explained by an in-context bigram completion (one-token prefix), versus $11.2\%$ on the rest of the corpus; for trigram completions (two-token prefix), the corresponding rates are $33\%$ in the band and $3.4\%$ elsewhere. The prefixes themselves also recur in the same sequence far more often in the band ($74\%$ for one-token prefixes, $36\%$ for two-token ones) than in the rest of the corpus ($47\%$ and $11\%$, respectively).

The low-loss band of Figure~\ref{fig:realworld-curse} therefore primarily corresponds to context-dependent predictions. This is a concrete example of how the structure of the data can affect our analysis: the simple fact that text repeats itself at short range within a document enables fast prediction and generalization on contexts that have been barely seen, or not seen at all, during training. Precisely measuring the curse of ambiguity in practice would require controlling for this kind of structure, which we leave to future work.

 \definecolor{darkred}{rgb}{0.55, 0.0, 0.0}

  \begin{table}[!h]
  \centering
  \scriptsize
  \begin{tabular}{p{0.97\linewidth}}
  \toprule
  \texttt{ population (1994 est.) Net migration rate:   -12.05
  migrant(s)/}\textcolor{darkred}{\texttt{1,}}\texttt{000 population (1994 est.) Infant
  mortality rate:   18.5 deaths/1}\textit{\texttt{,}} \\
  \midrule
  \texttt{RESSIDA. I will not, uncle. I have forgot my father; I know no touch of
  consanguinity, No kin, no love, no blood, no soul so near me As the
  sweet}\textcolor{darkred}{\texttt{ Troil}}\texttt{us. O you gods divine, Make Cressid's name
   the very crown of falsehood, If ever she leave Tro}\textit{\texttt{il}} \\
  \midrule
  \texttt{ from the men of Shechem, and from the house of Millo, and
  devour}\textcolor{darkred}{\texttt{ Abimele}}\texttt{ch.  9:21 And Jotham ran away, and
  fled, and went to Beer, and dwelt there, for fear of Abimelech his brother.  9:22 When
  Abime}\textit{\texttt{le}} \\
  \midrule
  \texttt{:014:017 And I, behold, I will harden the hearts of the Egyptians, and they shall
  follow them: and I will get me}\textcolor{darkred}{\texttt{ honour upon Ph}}\texttt{araoh,
  and upon all his host, upon his chariots, and upon his horsemen.  02:014:018 And the
  Egyptians shall know that I am the LORD, when I have gotten me honour upon}\textit{\texttt{
  Ph}} \\
  \midrule
  \textcolor{darkred}{\texttt{/1,}}\texttt{000 population (1994 est.) Net migration rate:
  -1.78 migrant(s)/1,000 population (1994 est.) Infant mortality rate:   8.9
  deaths/1}\textit{\texttt{,}} \\
  \midrule
  \texttt{God, that I may be at once avenged of the Philistines for my two eyes.
  16:29}\textcolor{darkred}{\texttt{ And Sam}}\texttt{son took hold of the two middle pillars
  upon which the house stood, and on which it was borne up, of the one with his right hand,
  and of the other with his left.  16:30 And}\textit{\texttt{ Sam}} \\
  \midrule
  \texttt{ the tone of fairy tales, and it is certainly not lawlessness or even liberty,
  though men under a mean modern tyranny may think it liberty by comparison.  People out of
  Portland Gaol might think Fleet Street free; but closer study will prove that both fairies
  and journalists are the slaves of duty. Fairy}\textcolor{darkred}{\texttt{
  godmo}}\texttt{thers seem at least as strict as other god}\textit{\texttt{mo}} \\
  \midrule
  \texttt{ creditors by   purchasing a sizable amount of the delinquent commercial debt in the
     secondary market at a substantial discount. The government had paid   100\% of remaining
  official}\textcolor{darkred}{\texttt{ debt arre}}\texttt{ars to the US, Germany, France,
  and Spain. All commercial debt ar}\textit{\texttt{re}} \\
  \midrule
  \texttt{: and Ab}\textcolor{darkred}{\texttt{imele}}\texttt{ch king of Gerar sent, and took
  Sarah.  20:3 But God came to Abime}\textit{\texttt{le}} \\
  \midrule
  \texttt{ prepare to meet and attack him.  Wright, with his
  division}\textcolor{darkred}{\texttt{ of Sedg}}\texttt{wick's corps, was ordered, by any
  road he could find, to join on to Warren's right, and Getty with his division, also of
  Sed}\textit{\texttt{g}} \\
  \bottomrule
  \end{tabular}
  \vspace{1em}
  \caption{\textbf{Examples of low-loss tokens from the band highlighted in Figure~\ref{fig:realworld-curse}} ($n_\mathrm{est} < 10$). The token to predict is shown in \textit{italics} and the matching $n$-gram earlier in the same sequence is shown in \textcolor{darkred}{dark red} (a trigram when a trigram match exists, a bigram otherwise).}
  \label{tab:realworld_band_examples}
  \end{table}

\subsection{Quantifying the curse of ambiguity at fixed semantic context frequency}
\label{app:quantification_correlation}

Figure~\ref{fig:realworld-curse} provides a qualitative picture confirming that large language models suffer from the curse of ambiguity. We take a step towards quantifying it by computing the correlation between the per-context loss and the estimated ambiguity, at fixed semantic context frequency. We do this for both ambiguity proxies, effective support size (Figure~\ref{fig:realworld-binned-correlation-perplexity}) and raw support size (Figure~\ref{fig:realworld-binned-correlation}). In both cases, the per-context loss is positively correlated with ambiguity in log-log space, quantitatively confirming the curse of ambiguity. We report the squared Pearson correlation $R^2$, which reaches $0.37$ for effective support size and $0.68$ for raw support size, that is, correlations of $r \approx 0.61$ and $r \approx 0.83$. The values are higher for raw support size partially because $R^2$ measures linear correlation while the binned plots for effective support size look more nonlinear. The correlation is also highest in regimes where the overall loss is not too large, in particular it drops in the middle band of semantic context frequency.

\begin{figure}[!h]
    \centering
    \includegraphics{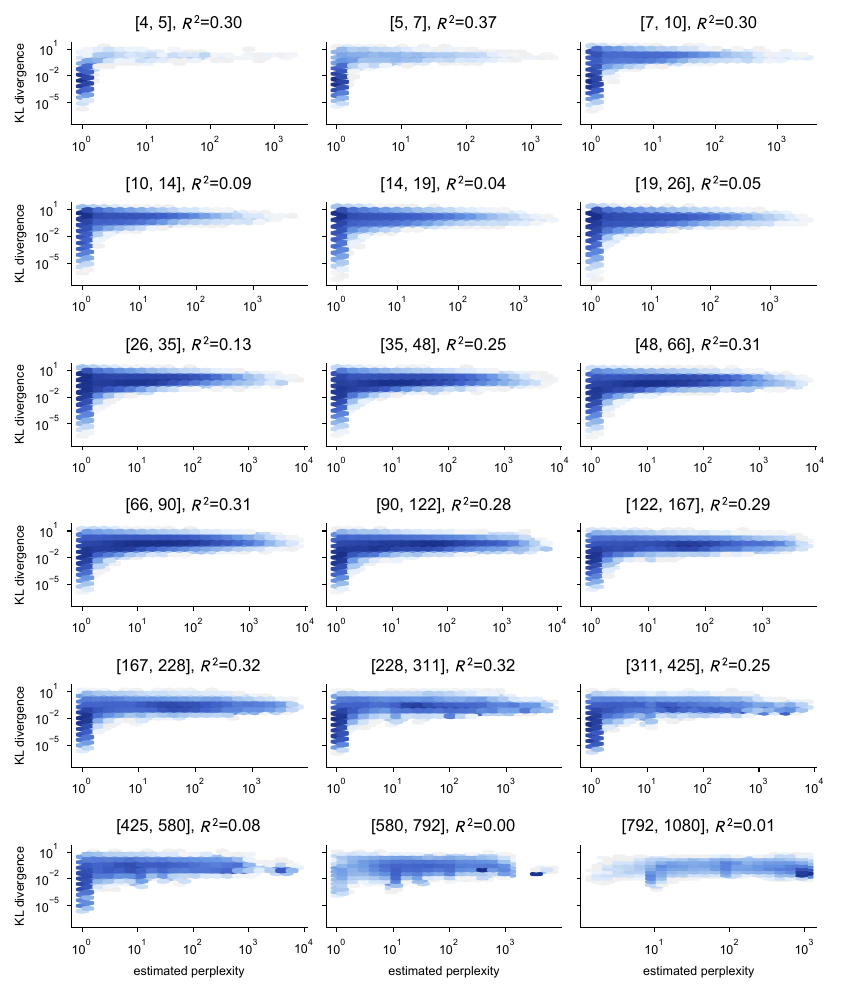}
    \caption{\textbf{Per-context loss and effective support size are correlated at fixed semantic context frequency.} Each panel slices the data of Figure~\ref{fig:realworld-curse} by a log-spaced bin of semantic context frequency, indicated in brackets in the title, and shows the 2D log-density of the KL divergence between the predictions of the 13B reference model and the 1B probe against the estimated effective support size; the value next to each title is the squared Pearson correlation $R^2$ between $\log D_\mathrm{KL}$ and $\log$ effective support size within the bin. The correlation is strongest for both small (rare contexts, $n_\mathrm{est} \lesssim 10$) and large (common contexts, $n_\mathrm{est} \gtrsim 100$) semantic context frequencies. It drops for intermediate values of $n_\mathrm{est}$, which is partially explained by the high overall loss in that regime: when most contexts already sit near the maximum KL, effective support size has little remaining variance to predict.}
    \label{fig:realworld-binned-correlation-perplexity}
\end{figure}

\begin{figure}[!h]
    \centering
    \includegraphics{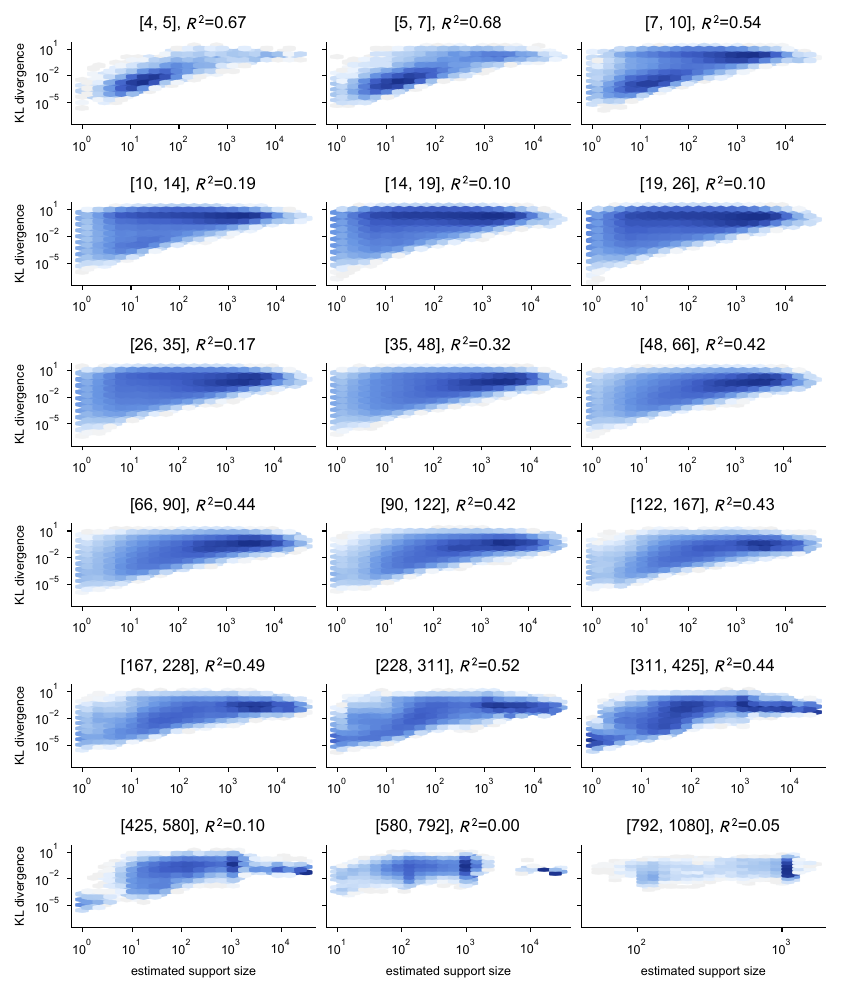}
    \caption{\textbf{Per-context loss vs. raw support size at fixed semantic context frequency.} Same setup as Figure~\ref{fig:realworld-binned-correlation-perplexity}, with raw support size in place of effective support size on the x-axis. $R^2$ is higher than in the effective-support-size version, partly because the relationship is more linear in this plane.}
    \label{fig:realworld-binned-correlation}
\end{figure}


\end{document}